\documentclass[journal]{IEEEtran} 
\usepackage{graphicx}
\usepackage{epsfig} 
\usepackage{booktabs}
\usepackage{array}
\usepackage{times} 
\usepackage{amsmath} 
\usepackage{amssymb}  
\usepackage{enumerate}
\usepackage{mathrsfs}
\usepackage{multirow}
\usepackage{subfigure}
\usepackage{latexsym}
\usepackage{bm}
\usepackage{dsfont}
\usepackage{multicol}
\usepackage{multirow}
\usepackage{color}
\usepackage{url}
\usepackage{cite}
\usepackage{algorithmic}
\usepackage{soul}
\usepackage[ruled]{algorithm2e}
\usepackage{comment}
\usepackage{picinpar}
\usepackage{flushend}
\usepackage{colortbl}
\usepackage{multirow}
\usepackage{pifont}
\usepackage{alltt}
\usepackage[hidelinks]{hyperref}
\usepackage{siunitx}
\usepackage{breakurl}
\usepackage{pbox}
\usepackage{footmisc} 
\usepackage{threeparttable}
\def\BibTeX{{\rm B\kern-.05em{\sc i\kern-.025em b}\kern-.08em
		T\kern-.1667em\lower.7ex\hbox{E}\kern-.125emX}}
\newtheorem{remark}{Remark}
\newtheorem{assumption}{Assumption}

\newtheorem{theorem}{Theorem}

{
	\begin{bmatrix}}%
	{\end{bmatrix}
	}

\def\build#1_#2^#3{\mathrel{\mathop{\kern0pt#1}\limits_{#2}^{#3}}}%

\def\build#1_#2^#3{\mathrel{\mathop{\kern0pt#1}\limits_{#2}^{#3}}}%

\makeatletter
\renewcommand*\env@matrix[1][*\c@MaxMatrixCols c]{%
	\hskip -\arraycolsep
	\let\@ifnextchar\new@ifnextchar
	\array{#1}}

	\newcommand{\mc}{\textcolor{black}}

\makeatother

\begin{document}
	
 \title{
 Distributed Secure Learning Control for Large-scale Multirobots under Stealthy {\color{black}Actuator} Attacks}
 \author{
		Xinglong Zhang,   Qingwen Ma, Cong Li, Hui Yin, Changxin Zhang, Yueying Wang, Wei Pan, Xin Xu
		\thanks{
 Xinglong Zhang, Qingwen Ma, Cong Li, Changxin Zhang, and Xin Xu are with the College of Intelligence Science and Technology, National University of Defense Technology, Changsha, 410073, China. Hui Yin is with the
College of Mechanical and Vehicle Engineering, Hunan University,
 Changsha, 410082, China. Yueying Wang is with the School of Mechatronic Engineeringand Automation, Shanghai University, Shanghai 200444, China. Wei Pan is with the Department of Computer Science, The University of
Manchester, 2628 CD Manchester, U.K. (e-mail:  \{zhangxinglong18, lc, zhangchangxin19, xinxu\}@nudt.edu.cn, 
 huiyin@hnu.edu.cn, yueyingwang@shu.edu.cn, wei.pan@manchester.ac.uk). 
			%
			
			
		}
	}
	
	\maketitle
	
\begin{abstract}

 Distributed learning control for multirobot systems (MRS) offers significant flexibility in presence of uncertainties but lacks provable performance guarantees. A promising direction involves integrating reinforcement learning (RL) into distributed model predictive control (DMPC), leveraging the strengths of RL in nonlinear policy design and the receding-horizon replanning capabilities of DMPC. However, ensuring secure control within such a learning framework under malicious cyber attacks, particularly stealthy ones, remains a critical challenge, because the distributed policies generation depends on information exchange among neighbors, where compromised agents can rapidly influence the behavior of others through the communication network.
This article proposes a distributed secure learning control (DSLC) framework for large-scale MRS under malicious, stealthy {\color{black}actuator} attacks. Our framework offers two key features: (i) a unified approach that enables secure learning control across various coordination scenarios, including general formation, affine formation, and containment; and (ii) a game-theoretic distributed learning-based predictive control strategy that learns how to balance the attacker and defender through a differential-game based DMPC framework. Specifically, DSLC employs a distributed attacker–actor–critic architecture to learn the optimal defense and attack policies online within each prediction interval. Unlike numerical optimization-based controllers that calculate open-loop control sequences, our method simultaneously generates adversarial attack policies and corresponding defense policies in analytical closed-loop form. The defense policies could be directly generalized to MRS with varying scales and diverse {\color{black}actuator} attack probabilities, enabling rapid and secure control deployment in large-scale MRS. The effectiveness and scalability of DSLC are validated through comprehensive simulations and real-world experiments in multiple wheeled robots via various control tasks.

	\end{abstract}
	{}

	\IEEEdisplaynontitleabstractindextext
	\IEEEpeerreviewmaketitle
 
 \section{Introduction}
 Multirobot systems (MRS) consist of a number of robots working cooperatively to complete complex tasks that are beyond the capabilities of a single robot, and have received increasing attention in both civil and industrial fields \cite{santilli2021dynamic,wu2023secure,ma2023robust,zhou2023robust}.  In recent decades, distributed control for MRS has received notable attention for enabling a wide range of cooperative tasks, such as formation and containment control. In multirobot control, cooperative performance optimization is crucial, which can address task-level coordination and energy consumption optimization in addition to stability-oriented control.
 


As a foundational optimization-based control methodology, distributed model predictive control (Distributed MPC, DMPC) has been extensively studied in MRS \cite{lee2015cooperative,zhou2022event,conte2016distributed}, due to its ability to handle explicit constraints and its well-established theoretical framework. Traditional DMPC for MRS typically depends on prior knowledge of nonlinear dynamics, and each local controller is required to solve complex numerical optimization problems online at each time instant. This reliance on real-time computation poses significant challenges for the practical deployment of DMPC in large-scale MRS, particularly when operating on resource-constrained robotic platforms. 
Recent advances have used policy learning techniques, such as actor-critic reinforcement learning (RL), to replace online numerical solvers for the generation of closed-loop MPC policies \cite{zhang2025toward}. These approaches are promising in improving the scalability and practical applicability of DMPC for large-scale systems. However, each local controller in a distributed control setting calculates its policy using state information received from neighboring agents via communication networks.
This reliance on inter-agent communication introduces significant vulnerabilities: cyber attacks targeting local controllers can compromise the global performance of MRS.   Despite growing interest,  distributed learning control with security guarantees for MRS in the presence of malicious cyber attacks remains an open and unexplored problem. This problem is particularly critical in the context of large-scale multirobot control under attacks, where communication constraints and limited computational resources further complicate the design of secure controllers. 


Malicious cyber attacks represent one of the most critical security threats to MRS, as they can compromise or manipulate the information exchanged among robots, thereby undermining cooperative behavior and impeding task execution \cite{wu2018secure,li2023intelligent}. 
 %
%
Unlike widely investigated exogenous disturbances on dynamical systems, which are not manipulable and typically lead to performance degradation, malicious attacks are often intentional and pose serious threats to system security, such as unauthorized attempts to take control of robots. Malicious attacks can generally be categorized into two main types: denial-of-service (DoS) attacks and deception attacks~\cite{wu2023secure,feng2020secure}. DoS attacks disrupt communication by overwhelming servers or network resources, thereby blocking access to critical services. Due to their high traffic demands, these attacks are relatively easier to detect. 
 In contrast, deception attacks, including false data injection (FDI) attacks, replay attacks, and actuator attacks, are significantly more stealthy and insidious. {\color{black}For example, in actuator attacks, an adversarial policy covertly manipulates the actuator signals such that the corrupted inputs remain indistinguishable from legitimate control commands to the controller and do not trigger anomaly- or residual-based detection mechanisms, while simultaneously driving the system toward undesired or potentially unsafe behaviors.} Owing to their covert nature, such attacks are considerably more difficult to detect and mitigate, thereby posing a substantial threat to the integrity and functionality of MRS.
To mitigate such threats, various distributed secure control strategies have been proposed, including event-triggered control~\cite{tan2023event,geng2023hybrid,zhan2023event}, 
 observer-based robust control~\cite{zhao2025observer}, and detection-monitor-based control~\cite{bonczek2022detection,liu2023resilient,li2025securing}. 
Specifically, secure control for MRS under DoS attacks was considered in \cite{tan2023event,geng2023hybrid,zhan2023event},
and deception attacks for MRS were dealt with in \cite{bonczek2022detection,liu2023resilient}.
Although these approaches have made notable progress in control under certain types of cyber attacks, they primarily focus on security and robustness. The critical challenge of achieving real-time cooperative performance optimization under malicious, stealthy attacks remains unresolved.

To address above challenges, a distributed secure learning control (DSLC) framework is presented for large-scale MRS under malicious, stealthy actuator attacks. The effectiveness and potentiality of DSLC are validated both theoretically and empirically through extensive simulations and real-world experiments on multiple wheeled robots.
The contributions of this article are summarized below.
\begin{enumerate}[(i)]
\item We propose a unified, game-theoretic, distributed learning-based predictive control algorithm to address the secure control challenges of large-scale MRS under malicious, stealthy actuator attacks. Our approach represents a generalizable framework to optimize cooperative performance across various coordination scenarios, including general formation, affine formation, and containment.
\item In each prediction interval of our approach, attackers (who degrade control performance) and defenders (who secure the closed-loop system) are modeled as adversarial players in a finite-horizon differential game. Solving such a problem in real time using nonlinear numerical optimization solvers is highly challenging. 
To address this, we propose an efficient attacker–actor–critic learning algorithm that generates attack and defense policies online. Unlike numerical optimization-based methods that calculate open-loop control sequences over the entire prediction horizon, our approach decomposes the finite-horizon optimization problem into a series of smaller subproblems. Then, it simultaneously generates the adversarial attack policies and corresponding defense policies in analytic closed-loop form, rather than open-loop control sequences, through solving the subproblems in a heuristic and forward-in-time iteration way. This approach enables efficient deployment across diverse robot team sizes and attack distribution probabilities.
\item The proposed approach has been extensively validated through both simulated and real-world experiments. The results demonstrate its scalability in distributed control of 1000 robots under malicious actuator attacks. Furthermore, comprehensive real-world experiments on multiple wheeled robots, covering general formation, affine formation, and containment scenarios under actuator attacks, further highlight the superiority and generalizability of our methodology in diverse applications.
\end{enumerate}

The remainder of the article is organized as follows.  Section II reviews the related work. Section III presents the MRS models under actuator attacks and formulates the differential game-based DMPC problem. The proposed DSLC framework as well as its efficient implementation via a distributed attacker-actor-critic structure are presented in Section IV. Section V demonstrates the simulation results, while the real-world experimental tests are performed in Section VI. Conclusions are drawn in Section VII. Some auxiliary materials and the main theoretical results are given in Appendices~\ref{appendix} {\color{black}and~\ref{Auiliary-MATERIAL}}.

 \textbf{Notation:}  Throughout this article, we use $\mathbb{R}$ and $\mathbb{R}^{+}$ to denote sets of real numbers and positive real numbers, respectively;
		$\mathbb{R}^{n}$ to denote the Euclidean space of the $n$-dimensional real vector;
		$\mathbb{R}^{n \times m}$ to denote the Euclidean space of $n \times m$ real matrices. Denote $\mathbb{N}$ as the set of integers and $\mathbb{N}_{l_1}^{l_2}$ as the set of integers $l_1,l_1+1,\cdots,l_2$. Given a general function $f(x)$ on the argument $x$, we denote $\nabla f(x)$ as the gradient with respect to $x$. For a group of vectors $z_i\in\mathbb{R}^{n_i}$, $i\in\mathbb{N}_{1}^{M}$, we use $\text{col}_{i\in\mathbb{N}_1^M}(z_i)$ or $(z_1,\cdots,z_M)$ to denote $[z_1^{\top},\cdots,z_M^{\top}]^{\top}$, where $M\in\mathbb{N}$.  For a matrix $P\in\mathbb{R}^{n\times n}$, $P\succ 0$ means that it is positive definite.  For a vector $x\in\mathbb{R}^{n}$, we denote $\|x\|_Q^2$ as $x^{\top}Qx$ and $\|x\|$ as the Euclidean norm.

\section{Related Work}\label{sec:related work}
This section first introduces the recent advances on RL and DMPC for multirobot control, and then reviews the secure control of MRS under attacks.

 \subsection{RL for Multirobot Control}
RL is a powerful data-driven paradigm for tackling nonlinear optimal control problems in continuous, high-dimensional state spaces. 
Recently, the integration of RL with distributed control architectures has emerged as a promising direction.
 Among them,  a zero-order distributed policy optimization approach was proposed in~\cite{9616447} for linear multi-agent systems, which blends consensus-based cost estimation with derivative-free local policy gradients, yielding stabilizing controllers with provable non-asymptotic guarantees.
 An asynchronous advantage actor-critic scheme was employed in \cite{8755279} to learn decentralized locomotion controllers, where each agent corresponds to an independently actuated segment of a snake or hexapod robot, thus achieving sample-efficient, hardware-friendly closed-loop policies without centralized coordination. 
 In resilient control, an RL-based fault-tolerant containment control protocol was developed in \cite{9597477} for multi-agent discrete-time systems, which ensures uniform boundedness of containment error under unknown faults. 
 Similarly, a deep RL-based distributed longitudinal control strategy was introduced in \cite{SHI2023104019}, which treats blocks of human-driven vehicles as aggregated entities to reduce stochasticity and improve oscillation damping and eco-driving performance. 
 
 It is worth noting that the existing RL approaches for nonlinear multirobot control~\cite{8755279,9597477,SHI2023104019} generally lack online learning capability and formal stability guarantees. Moreover, none of these works addresses the control problem under malicious cyber attacks within a distributed learning-based framework~\cite{wu10144090}. These limitations have motivated the development of our DSLC framework, which offers performance optimization with theoretical guarantees for multirobot control under adversarial cyber threats.

 \subsection{Distributed MPC for MRS}

\emph{Distributed MPC for multirobot control}: DMPC has been widely used in MRS to enable cooperative performance optimization, with applications spanning path planning~\cite{chen2023multirobot}, coordinated area coverage~\cite{yu2025cognitive}, and formation control~\cite{lee2015cooperative}. In~\cite{lee2015cooperative}, a DMPC method based on a cooperative co-evolutionary algorithm was proposed for MRS with undirected communication graphs, ensuring asymptotic stability regardless of optimality. In~\cite{wang2021distributed}, a DMPC algorithm was applied to optimize quadratic performance index under asynchronous sampling intervals in MRS.
To handle directed communication topologies, a neural network-optimized DMPC framework was introduced in~\cite{xiao2023integrated} for formation control, where optimal solutions were calculated using a quadratic programming solver based on a primal-dual neural network. In multirobot control with unknown disturbances, a robust DMPC strategy incorporating extended state observers was presented in~\cite{liu2020formation} to enhance disturbance rejection and anti-jamming capability in formation tasks. Furthermore, a control barrier function-based DMPC scheme was developed in~\cite{chao2024incorporating} for cluster flocking, ensuring inter-robot collision avoidance while maintaining desired formation behaviors.
Note that in the aforementioned DMPC approaches~\cite{wang2021distributed,lee2015cooperative,xiao2023integrated,liu2020formation,chao2024incorporating}, each local controller relies on numerical solvers to compute control inputs. However, this process is computationally intensive and lacks scalability for deployment, particularly in resource-constrained scenarios. 

\emph{Policy learning in DMPC}: Recent studies have integrated RL with MPC to combine their respective strengths for robot control, with most efforts focused on single robot systems~\cite{Song2022,Zhang2022TIV}. In multirobot control, a fast policy learning method was proposed in~\cite{zhang2025toward}, which enables the generation of explicit closed-loop DMPC policies without relying on numerical optimization online. This approach significantly improves computational efficiency and facilitates rapid deployment in large-scale MRS.
However, it is limited to a specific formation control task and does not consider the control impacts of cyber attacks or external disturbances. 

In this work, we tackle the challenge of distributed secure control for large-scale MRS operating under malicious, stealthy actuator attacks, while ensuring real-time performance optimization and system stability. Furthermore, our approach offers a unified optimization framework that generalizes across various coordination scenarios, including general formation, affine formation, and containment, thereby enhancing its applicability and robustness in diverse multirobot applications.

 \subsection{Secure Control for MRS under Attacks}
 Numerous secure control strategies have been developed for MRS under DoS and deception attacks~\cite{survey9609637}, but with a special focus on the stability and robustness guarantees under attacks. 
 
 \emph{Secure control under DoS attacks}: Based on attack detection tools, secure control methods such as distributed event-triggered control strategies \cite{tasooji2024event, zheng2024periodic} and sliding mode-based robust controller \cite{wang2025predefined} were developed to ensure safe and reliable multirobot operation in the presence of DoS attacks. 
 In \cite{zhan2024resilient}, a data-driven Koopman operator approach was introduced to estimate missing signals caused by DoS attacks, enabling resilient formation control for networked non-holonomic mobile robots. 
 Other notable contributions to secure control under DoS attacks can be found in \cite{yin2023security, liu2020sliding, qin2020optimal}; however, these studies cannot be directly applied to secure control under stealthy deception attacks. 

\emph{Secure control under deception attacks}: In contrast to DoS attacks, deception attacks are stealthier and more difficult to detect and identify, thereby receiving significant attention in the field of secure control. In \cite{wang2021distributedB}, a distributed adaptive robust control strategy was proposed for MRS to address the secure control problem under deception attacks, with stability analysis conducted for the closed-loop system. In \cite{li2025securing}, a secure and ubiquitous formation control method for MRS was developed to defend against replacement attacks. A convex neighbor polygon technique was introduced to explicitly characterize physical proximity among robots, aiding in the identification of attackers. Similarly, in \cite{yang2024defense}, both active and passive distributed secure control approaches were developed for MRS formation under displacement attacks, using only local information.
To counter replay attacks, a secure control framework was proposed in \cite{liu2023secure} for leader–follower formation control, where an extended state observer was designed to estimate unknown nonlinearities, including system uncertainties and external disturbances. Furthermore, observer-based robust control strategies were developed in \cite{ma2020sparse,ye2024extended} to mitigate the impact of false data injection attacks, with theoretical guarantees on closed-loop system stability.

It is highlighted that the above methods \cite{tasooji2024event,zheng2024periodic,wang2025predefined,zhan2024resilient,yin2023security,liu2020sliding,wang2021distributedB,li2025securing,yang2024defense,liu2023secure,ma2020sparse,ye2024extended} focus mainly on the robustness and stability issue for MRS against different types of cyber attacks. However, beyond robustness and stability, performance optimization under cyber attacks remains a critical and unaddressed challenge.
To the best of our knowledge, no existing work has tackled the performance optimization problem for MRS under cyber attack conditions.
This gap has motivated the development of our distributed secure learning control framework, which aims to achieve cooperative optimal control in the presence of malicious stealthy attacks.

 \section{Problem Formulation}\label{problem_formulation}
	In this section, we first introduce the continuous-time kinematic models of MRS under malicious actuator attacks. Next, we present the formulation of the differential game-based DMPC problem.  
	\subsection{Kinematic Models of MRS under Malicious Attacks}
	Consider $M$ continuous-time, nonlinear, mobile robots. The kinematic model of the $i$-th robot is given as
\begin{equation}\label{kinematic-model}
    \begin{aligned}
        \dot q_i = [v_i \cos\theta_i, \;v_i \sin \theta_i,\;\omega_i,\;a_i]^{\top},
    \end{aligned}
\end{equation}
where $q_i =[p_{x,i},\;p_{y,i},\;\theta_i,\;v_i]^{\top}$, $i\in\mathbb{N}_1^M$, $(p_{x,i},\;p_{y,i})$ denotes the position of the $i$-th robot in the Cartesian frame, $\theta_i$ is the yaw angle of the $i$-th robot with respect to the Cartesian frame, and $ v_i$ depicts the linear velocity; $a_i$ and $\omega_i$ are the acceleration and yaw rate of the robot, respectively. 

In this article, our goal is to develop a unified, distributed secure learning control framework for MRS described in~\eqref{kinematic-model} under malicious actuator attacks. In the following, we first introduce a unified notation to represent the consensus error state of each local robot, applicable across various control tasks including general formation, affine formation, and containment as follows. \\
(a) General formation:\\
\begin{subequations}\label{eqn:uni_error}
    \begin{align}\label{error}
        x_i =& T_i\left(c_{\scriptscriptstyle i}(q_r-q_i+\delta_{ri}) +\sum_{j\in\mathcal{N}_i}  \Lambda_1(q_j - q_i + \delta_{ji})\right),
    \end{align}
(b)  Affine formation: 
    \begin{align}\label{error-AF}
         x_i =& T_i\left(\zeta_{ri}(q_r-q_i) +\sum_{j\in\mathcal{N}_i}  \Lambda_1\zeta_{ji}(q_j - q_i )\right),
    \end{align}
(c) Containment:
    \begin{align}\label{error-Con}
         x_i =& T_i\left(\sum_{j\in\mathcal{N}_i^l} (q_{r_j}-q_i) +\sum_{j\in\mathcal{N}_i}  \Lambda_1 (q_j - q_i )\right),
    \end{align}
\end{subequations}
where $\Lambda_1=\text{diag}\{1,1,0,0\}$, and
$$
T_i=\left[\begin{array}{cccc}
\cos\theta_i&\sin\theta_i&0\\
-\sin\theta_i&\cos\theta_i&0\\
0&0&I_2\\
\end{array}\right], $$
 $q_r$ denotes the state information of the leader; $\delta_{\star i}$, $\star=j,r$ denotes the deviation between the $j$-th robot or the leader and the $i$-th robot; $\mathcal{N}_i $ represents the set of neighbors of the $i$-the robot, which implies that the $i$-th robot can obtain the information from the $j$-th robot for $j\in\mathcal{N}_i$;  $c_{\scriptscriptstyle i}=1$ if the $i$-th robot can receive the leader's information, $c_{\scriptscriptstyle i}=0$ otherwise. 
 $\zeta_{ri}$ and $\zeta_{ji}$ are the element of the Laplacian matrix $L^*$ for the affine formation of the MRS.
The set $\mathcal{N}_i^l$ includes the leader robots, of which the information can be received by the $i$-th robot in containment control; $q_{rj}$ is the state information of the $j$-th leader, $j\in \mathcal{N}_i^l$.
 For convenience, we define the error state as $x_i=[x_{i,1},x_{i,2},x_{i,3},x_{i,4}]^{\top}$, where $x_{i,\star}$ denotes the $\star$-th element of $x_i$.
 
In each local robot~\eqref{kinematic-model}, a malicious actuator attack is imposed on the input channel to deteriorate the control performance (see Fig.~\ref{Attack_Scheme_Diagram}). Hence, with~\eqref{kinematic-model} 
and~\eqref{eqn:uni_error}, one can write
a unified dynamic evolution of the error state for each robot $i$ under attacks, i.e., 
\begin{equation}\label{Eqn:LL}
 \dot x_{i}=f_{i}(x_{\scriptscriptstyle \mathcal{N}_i})+g_{i}(x_i)(u_{s,i}+ \beta_{t,i} u_{a,i}),
\end{equation}
where $u_{s,i}=[w_r-w_i,a_r-a_i]^{\top}$ are the input variables associated with the $i$-th robot, $w_r$ and $a_r$ are the reference yaw rate and acceleration, $u_{a,i}$ is the malicious
attack signal designed by the adversary, and the combined control applied to the system is $u_{i}=u_{s,i}+ \beta_{t,i} u_{a,i}\in {\mathbb{R}}^{2}$; 
$\beta_{t,i}$ is a Bernoulli random variable; 
  $\beta_{t,i}=1$ if the actuator of the $i$-th robot is attacked, $\beta_{t,i}=0$ otherwise; $\text{Prob}\left(\beta_{t,i}=1\right) = \beta_{i},\;\beta_{i} \in [0,1]$; 
$x_{\scriptscriptstyle \mathcal{N}_i}\in\mathbb{R}^{n_{\mathcal{N}_i}}$ is the collection of all the neighboring states (including itself), that is, $x_{\scriptscriptstyle \mathcal{N}_i}={\rm col}_{j\in{\scriptscriptstyle\mathcal{N}_i}}x_j$;
the mappings $f_{i}: \mathbb{R}^{n_{\scriptscriptstyle \mathcal{N}_i}}\rightarrow\mathbb{R}^{4}$ and $g_{i}: \mathbb{R}^{4}\rightarrow \mathbb{R}^{4\times 2}$ are state transition and input mapping functions, respectively, and $f_{i}(0)=0$; the detailed descriptions of $f_{i}(x_{\scriptscriptstyle \mathcal{N}_i})$ for scenarios (a), (b), (c) are omitted here for clarity and can be found in Appendix~\ref{model-append}. 
\begin{remark}
Unlike attacks on sensors or communication networks which are often easier to identify and isolate~\cite{zhan2024resilient}, this work focuses specifically on a class of stealthy actuator attacks. {\color{black}Such attacks are not necessarily limited to low-magnitude noises; instead, they may be high-magnitude signals that preserve the dynamical consistency of nominal control signals and evolve with smooth temporal profiles to avoid triggering anomaly- or residual-based
detection mechanism.} To address the secure control problem under stealthy actuator attacks, we adopt a game-theoretic, distributed learning control approach that optimizes the defense policy against the worst-case attack strategy. 
\end{remark}
	
	Collecting all the local robot systems from~\eqref{Eqn:LL}, the overall centralized dynamical model is written as 
	\begin{equation}\label{Eqn:C full model} \dot x=F_c(x)+G_c(x)u,
	\end{equation}
	where $u=\text{col}_{i\in\mathbb{N}_1^M}(u_i)\in\mathbb{R}^{2M}$, $F_c=\text{col}_{i\in\mathbb{N}_1^M}(f_i)$, the diagonal blocks of $G_c$
are $g_{i}$, $i\in\mathbb{N}_1^M$. 
 We also give the following standard assumption on the stabilizing control policy for stability analysis.
	\begin{assumption}[Structured stabilizing control~\cite{conte2016distributed}]\label{assum:stabilizing_control}
		There exists a feedback control policy $u_i(x_{\scriptscriptstyle\mathcal{N}_i})$  for all $i\in\mathbb{N}_1^M$, such that $u={\rm col}_{i\in\mathbb{N}_1^M}(u_i)$ is a stabilizing feedback control policy of~\eqref{Eqn:C full model}. 
	\end{assumption}
	\begin{figure*}[!htpb]
					\centering
     \includegraphics[width=1.6\columnwidth]{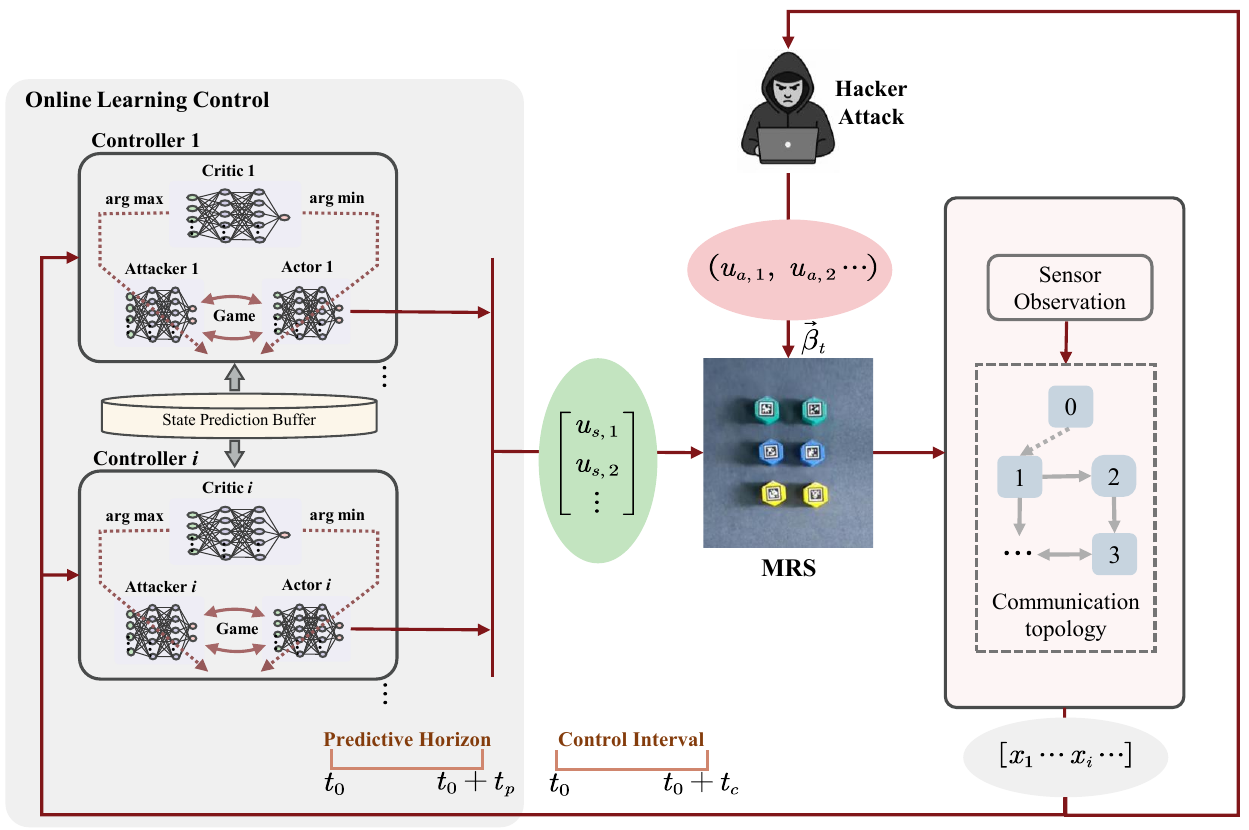}
				\caption{\mc{The scheme diagram of the distributed secure learning control (DSLC) framework. The attackers that deteriorate
the control performance and the defenders that secure
the closed-loop system are modeled as multi-players
in a differential-game based DMPC framework. The resulting problem is solved by a computationally efficient attacker-actor-critic learning algorithm, which generates the closed-loop control policies for the
attacker and defender. $\vec\beta_t=(\beta_{t,1},\cdots,\beta_{t,M})$}.\vspace{-2mm}}
				\label{Attack_Scheme_Diagram}
\end{figure*}
	\subsection{\mc{Differential-game based} DMPC Problem}
	 The secure control problem under malicious actuator attacks is formulated as a zero-sum differential game problem made by the defender and the attacker, within a continuous-time DMPC framework. At a generic time instant $t_0$, the following finite-horizon min-max optimization problem is solved: 
	\begin{equation}\label{Eqn:optimiz}
	\begin{array}{ccc}
	\min & \max & {J(x(t_0)))}\\
	{\bm u_{s,i}(t_0), \forall i\in\mathbb{N}_1^M}& {\bm u_{a,i}(t_0), \forall i\in\mathbb{N}_1^M}&\\
	\end{array},
	\end{equation}
	where $\bm u_{s,i}(t_0)=u_{s,i}(t_0:t_0+t_p)$, $\bm u_{a,i}(t)=u_{a,i}(t_0:t_0+t_p)$, $t_p>0$ is the prediction horizon, the global cost $J(x(t_0))=\sum_{i=1}^MJ_i(x_{\scriptscriptstyle \mathcal{N}_i}(t_0))$ with the local cost being defined as
	\begin{equation}\label{Eqn:HL_cost}
	\begin{array}{cl}
J_i(x_{\scriptscriptstyle \mathcal{N}_i}(t_0))=\mathbb{E}\left\{\int_{\tau=t_0}^{t_0+t_p}r_i(x_{\scriptscriptstyle \mathcal{N}_i}(\tau),u_{s,i}(\tau),u_{a,i}(\tau))d\tau\right.\vspace{1mm}\\
	 \hspace{30mm}+\left.\| x_i(t_0+t_p)\|_{P_i}^{2} \right\},
	\end{array}
	\end{equation}	
	the stage cost 
	$$r_i(x_{\scriptscriptstyle \mathcal{N}_i},u_i,u_{a,i})=\| x_{\scriptscriptstyle \mathcal{N}_i}\|_{Q_i}^{2}+\| u_{s,i}\|_{R_{s,i}}^{2}-\beta_{t,i} \|u_{a,i}\|^2_{R_{a,i}};$$
	 $Q_i=Q_i^{\top}\in\mathbb{R}^{n_{\scriptscriptstyle \mathcal{N}_i}\times n_{\scriptscriptstyle \mathcal{N}_i}}$, $Q_i\succ 0$, $R_{s,i}=R_{s,i}^{\top}\in\mathbb{R}^{2\times 2}$, $R_{a,i}=R_{a,i}^{\top}\in\mathbb{R}^{2\times 2}$, $R_{s,i},R_{a,i}\succ 0$, $i\in\mathbb{N}_1^M$; the matrix $P_i=P_i^{\top}\in\mathbb{R}^{4\times 4}$ is the terminal penalty matrix, $P_i\succ 0$. {\color{black}Note that, the maximization operator in problem~\eqref{Eqn:optimiz} models the adversary’s goal of maximizing cost~\eqref{Eqn:HL_cost} through an attack policy, while the minimization operator reflects the defender’s strategy of minimizing cost~\eqref{Eqn:HL_cost} against the worst-case attack.}

    In traditional DMPC, at a time instant $t_0$, the optimal control and attack sequence in the prediction interval can be calculated
using distributed optimization tools~\cite{conte2016distributed}. Only the defense and attack actions associated with the time interval $[t_0,t_0+t_c)$ are applied and the associated actions are calculated repeatedly at the next time $t_0+t_c$.
This paradigm requires each local robot to call nonlinear numerical solvers at each time interval, which is computationally intensive in the game-theoretic optimization framework for large-scale systems. In the sequel, we propose a computationally efficient distributed secure learning control algorithm that online learns the closed-loop policies for defenders and attackers. 
\section{Distributed Secure Learning Control Framework}   
In this section, the differential game-based DMPC problem is solved by our DSLC approach (see Fig.~\ref{Attack_Scheme_Diagram}). Unlike numerical optimization-based tools that online calculate open-loop defense and attack sequences, our DSLC approach simultaneously learns the analytic closed-loop defense policies, i.e., $u_{s}=\text{col}_{i\in\mathbb{N}_1^M}(u_{s,i}(x_{\scriptscriptstyle \mathcal{N}_i}))$ and the worst-case attack policies, i.e., $u_{a}=\text{col}_{i\in\mathbb{N}_1^M}(u_{a,i}(x_{\scriptscriptstyle \mathcal{N}_i}))$. In each prediction interval, the policy learning is executed in a forward-in-time way, which substantially reduces the computation load compared to the backward-in-time way in dynamic programming~\cite{zhang2025toward}.  
In what follows, a game-theoretic finite-horizon distributed policy learning algorithm is firstly introduced  with the closed-loop stability guarantee. Then, our approach is implemented by an efficient distributed attacker-actor-critic algorithm.
%
\subsection{Game-theoretic Finite-horizon Distributed Policy Learning} 
In what follows, our objective is to generate closed-loop policies for the differential game-based DMPC~\eqref{Eqn:optimiz} via a distributed policy learning approach. {\color{black}As noted in~\cite{kirk2004optimal}, a critical solution is to reformulate~\eqref{Eqn:optimiz} as the optimization of an overall Hamiltonian function in the prediction interval $t\in[t_0,t_0+t_p]$, that is} 
\begin{equation}\label{eqn:hailm}
\begin{array}{ll}
H(x(t),u_{s}(t),u_{a}(t),\nabla J\big(x(t)))= \\
\hspace{25mm}\mathbb{E}\left\{r(t)+\nabla J\big(x(t))^{\top}\dot x(t)\right\},
\end{array}
\end{equation}
where $r(t)=\sum_{i=1}^M r_i(t)$, $\nabla J\big(x(t))$ is the gradient of $J\big(x(t))$ with respect to $x$, and $J(x(t))=\sum_{i=1}^MJ_i(x_{\scriptscriptstyle \mathcal{N}_i}(t))$.
The optimal value function $J^{\ast}\big(x(t))$ satisfies the expected Hamilton–Jacobi–Isaacs (HJI) equation, i.e., 
\begin{equation}\label{Eqn:hjb}
    0=\min_{u_{s}(t)}\max_{u_{a}(t)}\mathbb{E}\left\{H(x(t),u_{s}(t),u_{a}(t),\nabla  J^{\ast}\big(x(t)))\right\}.
\end{equation}
Note that the local defense policy $u_{s,i}(x_{\scriptscriptstyle \mathcal{N}_i})$ and attack policy $u_{a,i}(x_{\scriptscriptstyle \mathcal{N}_i})$  of the $i$-th robot have direct impacts on the performance of the local costs $J_j^{\ast}\big(x_{\scriptscriptstyle {\mathcal{N}}_j})$ for all $j\in\bar{\mathcal{N}}_i$, where $\bar{\mathcal{N}}_i$ is the collection of local robots that share the $i$-th robot as one of their neighbors. Hence, for each robot $i$ in $t \in [t_0, t_0+t_p]$, the optimal policies of the defender and attacker can be calculated using the first-order optimality condition, that is,
\begin{subequations}\label{Eqn:optimal}
\begin{align}
 u_{a,i}^{\ast}(t)&=\frac{1}{2}R_{a,i}^{-1}\sum_{j\in\bar{\mathcal{N}}_i} g_i(x_i(t))^{\top}\nabla J_j^{\ast}\big(x_{{\scriptscriptstyle \mathcal{N}}_j}(t))\\
 u_{s,i}^{\ast}(t)&=-\frac{1}{2}R_{s,i}^{-1}\sum_{j\in\bar{\mathcal{N}}_i} g_i(x_i(t))^{\top}\nabla J_j^{\ast}\big(x_{\scriptscriptstyle {\mathcal{N}}_j}(t)),
\end{align}
\end{subequations}
$i\in\mathbb{N}_1^M$.
Substituting $u_{a,i}^{\ast}(t)$ and $ u_{s,i}^{\ast}(t)$ into the HJI equation leads to a partial differential equation on the optimal value function $J_j^{\ast}$, which is difficult to solve analytically. 
\begin{remark}
Note that the connection between the optimal defense and attack policies and the optimal local cost gradient $\nabla J_j^{\ast}\big(x_{\scriptscriptstyle {\mathcal{N}}_j}(t)\big)$ in equation~\eqref{Eqn:optimal} enables the design of a secure, distributed policy learning algorithm, where the defense and attack policies for the $i$-th robot can be computed using only local information from neighboring states.
\end{remark}

\emph{Game-theoretic distributed policy iteration}: In each prediction interval, a finite-horizon distributed policy iteration approach is designed to iteratively solve the game-theoretic HJI equation~\eqref{Eqn:hjb}. {\color{black}A typical policy iteration algorithm consists of two iterative procedures: policy evaluation and policy improvement~\cite{modares2014optimal}. Note that, during policy evaluation, the Hamiltonian function in~\eqref{eqn:hailm} contains the information of $\dot{x}$, which could be difficult to measure and noisy.} Hence, following the integral RL approach in \cite{vamvoudakis2014online}, we note for a time interval $T=t_p/N$ ($N\in\mathbb{N}$) and for any $t\in[t_0,\,t_0+t_p-T]$, the cost function satisfies
\begin{equation}\label{HJB-IRL}
J\big(x(t))=\mathbb{E}\left\{\int_{\tau=t}^{t+T}r(\tau)d\tau\right\}+J\big(x(t+T)).
\end{equation}
{\color{black}Equation~\eqref{HJB-IRL} allows us to perform policy evaluation in discrete time instants $t=t_0, t_0+T,\cdots, t_0+NT$, which enables efficient implementation while preserving prediction accuracy through the integral formulation over the interval $T$.}

Given an initial admissible control policy $u^0_{s,i}(x_{{\scriptscriptstyle\mathcal{N}}_i}(t))$ and $u^0_{a,i}(x_{{\scriptscriptstyle\mathcal{N}}_i}(t))$. The policy iteration at the iterative step $h$ for each robot $i$ is performed at each time step $t=t_0, t_0+T,\cdots, t_0+NT$.\\
(\romannumeral1) Policy evaluation:\\
\begin{equation}\label{policy-evaluation}
\begin{array}{ll}
\hspace{-2mm}\mathbb{E}\{\int_{\tau=t}^{t+T}r_i^h(\tau)d\tau\}+J_i^h\big(x_{{\scriptscriptstyle\mathcal{N}}_i}(t+T),u^h_{s,i}(t+T),u^h_{a,i}(t+T)\big)\vspace{1mm}\\
-J_i^h\big(x_{\scriptscriptstyle{\mathcal{N}}_i}(t),u^h_{s,i}(t),u^h_{a,i}(t)\big)=0.
\end{array}
\end{equation}
(\romannumeral2) Policy improvement:\\
\begin{subequations}\label{policy-improvement}
\begin{align}
 u_{a,i}^{h}(t)&=\frac{1}{2}R_{a,i}^{-1}\sum_{j\in\bar{\scriptscriptstyle\mathcal{N}}_i} g_i(x_i(t))^{\top}\nabla J_j^h\big(x_{{\scriptscriptstyle\mathcal{N}}_j}(t)),\vspace{1mm}\\
 u_{s,i}^{h}(t)&=-\frac{1}{2}R_{s,i}^{-1}\sum_{j\in\bar{\scriptscriptstyle\mathcal{N}}_i} g_i(x_i(t))^{\top}\nabla J_j^h\big(x_{{\scriptscriptstyle\mathcal{N}}_j}(t)).
\end{align}
\end{subequations}

Note that, through the formulation in~\eqref{HJB-IRL}, the min-max optimization problem~\eqref{Eqn:optimiz} over the prediction interval $[t_0,\,t_0+t_p]$ is now decomposed into $N$ subproblems. These subproblems are solved sequentially at each time step $t=t_0, t_0+T,\cdots, t_0+NT$ through the distributed policy iteration procedures~\eqref{policy-evaluation} and \eqref{policy-improvement}, which are executed in a heuristic, forward-in-time way.

\emph{Terminal penalty design:} {\color{black} A proper choice of $P_i$ in~\eqref{Eqn:optimiz} is crucial for the stability guarantee of the closed-loop system under the receding horizon optimization principle.} To this end, we first linearize model \eqref{Eqn:LL}  around the origin, i.e.,
$   \dot x_i = A_{\scriptscriptstyle\mathcal{N}_i} x_{\scriptscriptstyle\mathcal{N}_i} + B_i(u_{s,i} + \beta_{t,i} u_{a,i}) + \phi_i(x_{\scriptscriptstyle\mathcal{N}_i},u_{a,i},u_{s,i})$, where $\phi_i(x_{\scriptscriptstyle\mathcal{N}_i},u_{a,i},u_{s,i})\in\mathbb{R}^{n_{\scriptscriptstyle\mathcal{N}_i}}$ is the linearization error and $A_{\scriptscriptstyle\mathcal{N}_i}\in\mathbb{R}^{n_{\scriptscriptstyle\mathcal{N}_i}\times n_{\scriptscriptstyle\mathcal{N}_i}}$ and $B_{i}\in\mathbb{R}^{n_{\scriptscriptstyle\mathcal{N}_i}\times m_i}$ are the model parameters. 
According to \cite{hassan2002nonlinear}, $\lim_{(x_{\scriptscriptstyle\mathcal{N}_i},  u_{a,i},u_{s,i})\rightarrow 0}  
\phi_i(x_{\scriptscriptstyle\mathcal{N}_i},u_{a,i},u_{s,i})/(x_{\scriptscriptstyle\mathcal{N}_i},u_{a,i},u_{s,i})
 \rightarrow 0$, and there exists $L_\phi>0$ such that $\|\phi_i\| \le L_\phi \|x_{\scriptscriptstyle\mathcal{N}_i}\|$ for $x_{\scriptscriptstyle\mathcal{N}_i}\rightarrow 0$.
Let $ u_{s,i} = K_{\scriptscriptstyle \mathcal{N}_{s,i}} x_{\scriptscriptstyle\mathcal{N}_i}$, $ u_{a,i} = K_{\scriptscriptstyle \mathcal{N}_{a,i}} x_{\scriptscriptstyle\mathcal{N}_i}$ be candidate terminal control policies, where $K_{\scriptscriptstyle \mathcal{N}_{s,i}},\, K_{\scriptscriptstyle \mathcal{N}_{a,i}}\in\mathbb{R}^{m_i\times n_{\mathcal{N}_i}}$ are gain matrices. 
The terminal penalty matrix $P_i$ along with $K_{\scriptscriptstyle \mathcal{N}_{s,i}}$ and $ K_{\scriptscriptstyle \mathcal{N}_{a,i}}$ are designed to satisfy the following inequality, i.e., 
\begin{equation}\label{Lyapunov-equation}
\begin{aligned}
    &A_{\scriptscriptstyle \mathcal{N}_i}^{\top} P_{i} T_{\scriptscriptstyle \mathcal{N}_i} +   T_{\scriptscriptstyle \mathcal{N}_i}^{\top} P_i A_{\scriptscriptstyle \mathcal{N}_i} + 2 T_{\scriptscriptstyle \mathcal{N}_i}^{\top} P_i  B_{i} K_{\scriptscriptstyle \mathcal{N}_{s,i}} \\
      &\quad \quad \quad \quad \quad +  2\beta_{i}  T_{\scriptscriptstyle \mathcal{N}_i}^{\top} P_i B_i K_{\scriptscriptstyle \mathcal{N}_{a,i}} + Q_i^*  \le - \tau_{\scriptscriptstyle P} I + \Gamma_{\scriptscriptstyle \mathcal{N}_i},
\end{aligned}
\end{equation}
where $T_{\scriptscriptstyle \mathcal{N}_i}\in\mathbb{R}^{4\times4n_{\mathcal{N}_i}} $ is defined as the selective matrix with $ x_i = T_{\scriptscriptstyle \mathcal{N}_i} x_{\scriptscriptstyle \mathcal{N}_i}$, $\tau_{\scriptscriptstyle P} I \ge 2 T_{\scriptscriptstyle\mathcal{N}_i}^{\top} P_i L_{\phi_i}$, $\sum_{i=1}^M \Gamma_{\scriptscriptstyle \mathcal{N}_i}\leq 0$, $Q_i^* =  Q_i +  K_{\scriptscriptstyle \mathcal{N}_{s,i}}^{\top} R_{s,i} K_{\scriptscriptstyle \mathcal{N}_{s,i}} - \beta_{i} K_{\scriptscriptstyle \mathcal{N}_{a,i}}^{\top} R_{a,i} K_{\scriptscriptstyle \mathcal{N}_{a,i}}$. 
\begin{remark}
Note that the choices of $R_{s,i}$ and $R_{a,i}$ are critical to satisfy~\eqref{Lyapunov-equation}. As seen from~\eqref{Eqn:optimal}, $R_{s,i}$ should be set with a smaller magnitude relative to $R_{a,i}$ in order to yield a stronger defense input (i.e., higher input energy contribution from the defender). However, from~\eqref{Lyapunov-equation}, $R_{a,i}$ cannot be arbitrarily large; it must be selected such that $Q_i^* >0$ and~\eqref{Lyapunov-equation} are fulfilled.
The terminal penalty satisfying~\eqref{Lyapunov-equation} plays a critical role in establishing recursive feasibility and closed-loop stability within the distributed policy iteration framework. For detailed theoretical results and associated proofs, please refer to Appendix~\ref{appB}.\end{remark}
%
%
\subsection{Distributed Attacker-Actor-Critic Implementation}\label{sec:aclearning} 
Note that the distributed policy iteration procedure in~\eqref{policy-evaluation} and~\eqref{policy-improvement} is not ready for online rapid implementation due to the computationally intensive procedure~\eqref{policy-evaluation}.  
In what follows,  we present a game-theoretic, distributed attacker-actor-critic architecture to solve~\eqref{policy-evaluation} and~\eqref{policy-improvement} efficiently in each prediction interval. In the architecture, $M$ attacker-actor-critic network pairs are designed for MRS to generate near-optimal control policies with local information. Each pair consists of an attacker, an actor, and a critic for each robot. The critic network approximates the robot's value function, whereas the defense and attack policies are approximated by the actor and attacker networks, respectively. 
%

The critic neural network for the $i$-th robot is designed as
\begin{equation}\label{critic-network}
   \hat J_i(x_{\scriptscriptstyle{\mathcal{N}}_i}(t)) =  \hat W_{c,i}^{\top}  \varphi_{c,i}(x_{\scriptscriptstyle{\mathcal{N}}_i}(t)),
\end{equation}
where $\hat J_i $ is the estimation value of $ J_i^* $, $\hat W_{c,i}\in\mathbb{R}^{N_{c,i}}$ and $\varphi_{c,i} (x_{\scriptscriptstyle{\mathcal{N}}_i}(t))\in\mathbb{R}^{N_{c,i}}$ denote the  weight matrix and the activation function for the $i$-th robot, respectively. {\color{black}With the critic network~\eqref{critic-network} for value function approximation in~\eqref{HJB-IRL}}, one defines the approximation error as follows:
\begin{equation}\label{error-HJB-IRL-1}
\begin{aligned}
     \zeta_{i}(t)=\mathbb{E}\left\{ \int_{\tau=t}^{t+T}\hat r_i(\tau)d\tau\right\} + \hat W_{c,i}^{\top} & \varphi_{c,i}(x_{\scriptscriptstyle{\mathcal{N}}_i}(t+T))\\
     &-\hat W_{c,i}^{\top}\varphi_{c,i}(x_{\scriptscriptstyle{\mathcal{N}}_i}(t)),\\
\end{aligned}
\end{equation}
where $\hat r_i = \| x_{\scriptscriptstyle{\mathcal{N}}_i}\|_{Q_i}^{2}+\| \hat u_{s,i}\|_{R_{s,i}}^{2}-\beta_{t,i}\|\hat u_{a,i}\|^2_{R_{a,i}}$. 
Minimizing the loss function $E_{c,i} = \frac{1}{2} \|\zeta_{i}\|^2, \, \forall i\in\mathbb{N}$, is essential for obtaining an accurate approximation of the value function, which leads to the following update law for $\hat W_{c,i}$:
\begin{equation}\label{update-law-critic}
\begin{aligned}
	  \dot{\hat{W}}_{c,i}=&-\eta_{c,i} \frac{\Delta\varphi_{c,i}}{(1+\Delta\varphi_{c,i}^{\top}\Delta\varphi_{c,i})^2}\mathbb{E}
  \left\{\int_{t}^{t+T}\hat r_i(\tau)d\tau\right\}\\
  &-\eta_{c,i} \frac{\Delta\varphi_{c,i}}{(1+\Delta\varphi_{c,i}^{\top}\Delta\varphi_{c,i})^2} \Delta\varphi_{c,i}^{\top} \hat W_{c,i},\\  
\end{aligned}
\end{equation} 
where $\eta_{c,i}>0$ is the learning rate, and  $\Delta \varphi_{c,i} = \varphi_{c,i}(x_{\scriptscriptstyle{\mathcal{N}}_i}(t+T)) - \varphi_{c,i}(x_{\scriptscriptstyle{\mathcal{N}}_i}(t))$.

Likewise, the attacker and actor networks for constructing the explicit defense and attack policies of each robot $i$ are given as
\begin{subequations}\label{actor-networks}
\begin{align}
\hat u_{a,i}(t)&=\hat W_{a,i}^{\top} \varphi_{a,i}(x_{\scriptscriptstyle{\mathcal{N}}_i}(t)),\label{eqn:attack-hat}\\
\hat u_{s,i}(t)&=\hat W_{s,i}^{\top} \varphi_{s,i}(x_{\scriptscriptstyle{\mathcal{N}}_i}(t)),\label{eqn:actor-hat}
\end{align}
\end{subequations}
where $\hat u_{a,i}$ and $\hat u_{s,i}$ are the estimations of $u_{a,i}^*$ and $u_{s,i}^*$, $\hat W_{a,i}\in\mathbb{R}^{N_{a,i}\times m_i}$ and $\hat W_{s,i}\in\mathbb{R}^{N_{s,i}\times m_i}$ are the weight matrices, $\varphi_{a,i}(x_{\scriptscriptstyle{\mathcal{N}}_i}(t))$ and $\varphi_{s,i}(x_{\scriptscriptstyle{\mathcal{N}}_i}(t))$ denote the activation functions. 

{\color{black}To approximate the optimal attack and defense policies with~\eqref{actor-networks}}, the following loss functions are defined, respectively, as
\begin{subequations}\label{objective function of actor}
\begin{align}
E_{a,i}&=\frac{1}{2}\|\rho_{u_{a,i}^d} - \hat \rho_{\hat u_{a,i}}\|^2,\\
E_{s,i}&=\frac{1}{2}\|\rho_{u_{s,i}^d} - \hat \rho_{\hat u_{s,i}}\|^2,
\end{align}
\end{subequations}
with
\begin{equation*}
    \begin{array}{ll}
        \rho_{u_{a,i}^d}&=  \frac{1}{2} R_{a,i}^{-1}\sum_{j\in\bar{\mathcal{N}}_i} g_i^{\top}(x_i)\triangledown\varphi_{c,j}^{\top} \hat W_{c,j},\\
         \rho_{u_{s,i}^d}&=-\frac{1}{2}R_{s,i}^{-1}\sum_{j\in\bar{\mathcal{N}}_i} g_i^{\top}(x_i) \triangledown\varphi_{c,j}^{\top} \hat W_{c,j},\\
        \hat \rho_{\hat u_{a,i}}&=\hat W_{a,i}^{\top} \varphi_{a,i},\;\hat \rho_{\hat u_{s,i}}=\hat W_{s,i}^{\top} \varphi_{s,i},
    \end{array}
\end{equation*}
where $\rho_{u_{a,i}^d} $ and $ \rho_{u_{s,i}^d}$ are the desired control inputs generating by the critic networks of the $i$-th robot.
Subsequently, the update laws for the attacker and actor networks are given as
\begin{subequations}\label{update-law-actor}
    \begin{align}
    \dot{\hat W}_{a,i}&= \eta_{a,i} \varphi_{a,i}(\rho_{u_{a,i}^d} - \hat \rho_{\hat u_{a,i}})^{\top},\label{update-attack}\\
    \dot{\hat W}_{s,i}&= \eta_{s,i} \varphi_{s,i}(\rho_{u_{s,i}^d} - \hat \rho_{\hat u_{s,i}})^{\top},
    \end{align}
\end{subequations}
where $\eta_{a,i},\, \eta_{s,i}>0$ are the learning rates. 
\renewcommand{\algorithmicrequire}{\textbf{Notations:}}
\begin{algorithm}[h]
	\label{algorithm 1}
	{\caption{{DSLC for MRS under actuator attacks}}
		\begin{algorithmic}[1]
			\REQUIRE {$h_{m}$: the maximum iterations; $ t_p $, $ t_c $: the prediction horizon and control implementing interval; $t_{\scriptscriptstyle \text{end}}$: the simulation time; 
            $ T$: the simulation step size.}
			\STATE Initialize $ \hat W_{c,i} $, $ \hat W_{s,i} $, $\hat W_{a,i}$, $i\in\mathbb{N}_1^M$, $ x(0)$; 
			\WHILE{ $t_0 \le t_{\scriptscriptstyle \text{end}}$ }
			\WHILE{ $h \le h_m $ }
			\STATE $t=t_0$;
			\WHILE{$t\le t_0+t_p$}
			\STATE Calculate $ \hat u_{s,i},\;\hat u_{a,i}$, $i\in\mathbb{N}_1^M$ via~\eqref{actor-networks}; 
			\STATE Calculate ${\hat W}_{c,i}$, $i\in\mathbb{N}_1^M$, via~\eqref{update-law-critic}; // Updated for each $T$ step size. 
			\STATE Calculate ${\hat W}_{s,i},\;{\hat W}_{a,i}$, $i\in\mathbb{N}_1^M$, via~\eqref{update-law-actor}; // Updated for each $T$ step size.
			\ENDWHILE
			\STATE $h=h +1$;
			\ENDWHILE
			\STATE $ t=t_0$;
			\WHILE {$ t \le t_0 +t_c \; (t_c= Z T,\; Z \in \mathbb{N_+})$}
			\STATE Apply $\hat u_{s,i}$, $i\in\mathbb{N}_1^M$ with~\eqref{actor-networks};
			\STATE Update state $x$ with~\eqref{Eqn:C full model};
			\ENDWHILE     
			\STATE $ t_0=t_0+t_c $;
			\ENDWHILE			
		\end{algorithmic}}
	\end{algorithm}
    
In summary,  the pseudocode of the distributed attacker-actor-critic algorithm in each predictive interval is presented in Algorithm \ref{algorithm 1}. The learning convergence and closed-loop stability under the distributed attacker–actor–critic learning framework are established but omitted here for brevity. Please refer to Appendix~\ref{app:closed-loop} for details.

\begin{remark}
In each prediction interval $[t_0,t_0+t_p]$, the min-max DMPC problem~\eqref{Eqn:optimiz} is solved using the computationally efficient update laws~\eqref{update-law-critic} and~\eqref{update-law-actor}, which are executed forward in time at discrete instants $t_0,\, t_0+T,\, \cdots, t_0+NT$. Unlike traditional numerical optimization methods that solve the optimization problem in each prediction interval independently, our approach continuously updates the weights of the actor, attacker, and critic over successive prediction intervals. This enables real-time policy generation with reduced computational load and improved learning efficiency. 
\end{remark}
	
%
\begin{table*}[h]
  \centering
  \caption{Criterion for the Accessment of the Distributed Control Performance under Different Attacks.}
\mc{   \begin{tabular}{ccccccccccc}
    \toprule
    \multirow{2}{*}{Attacks}&\multirow{2}{*}{Methods} & \multicolumn{3}{c}{Two Robots} & \multicolumn{3}{c}{Four Robots} & \multicolumn{3}{c}{Six Robots} \\
\cmidrule{3-11}    \multicolumn{2}{c}{} & \multicolumn{1}{c}{IAE} & \multicolumn{1}{c}{ITAE} & \multicolumn{1}{c}{$J_c$} & \multicolumn{1}{c}{IAE} & \multicolumn{1}{c}{ITAE} & \multicolumn{1}{c}{$J_c$} & \multicolumn{1}{c}{IAE} & \multicolumn{1}{c}{ITAE} & \multicolumn{1}{c}{$J_c$} \\
    \midrule
    \multirow{3}[2]{*}{$\beta=0.1$} & DSLC  &  \textbf{0.57}&	\textbf{0.47}&	\textbf{3.35}&	\textbf{0.82}&	\textbf{0.8}&	\textbf{4.13}&	\textbf{0.93}&	\textbf{1.22}&	\textbf{7.09}\\
          & DRHRL~\cite{zhang2025toward}     &  7.04&	46.14&	44.45&	11.49&	101.41&	39.52&	25.65&	255.22&	249.43\\
          & DRMPC~\cite{zhou2022event}    & 19.56&	206.85&	23.68&	38.51&	546.78&	127.14&	89.83&	1014.86&	470.44\\
    \midrule
\multirow{3}[2]{*}{$\beta=0.2$} & DSLC     &  \textbf{0.65}&	\textbf{0.5}&	\textbf{3.36}&	\textbf{0.73}&	\textbf{1.42}&	3.84&	\textbf{0.78}&	\textbf{1.15}&	\textbf{6.01}\\
          & DRHRL~\cite{zhang2025toward}     &   17.63&	172.64&	283.94&	28.45&	214.12&	187.36&	24.34&	241.9&	277.96 \\
          & DRMPC~\cite{zhou2022event}    &  63.53&	761.76&	267.73&	6.25&	64.06&	\textbf{2.9}&	106.04&	1286.68&	719.28 \\
    \midrule
    \multirow{3}[2]{*}{$\beta=0.3$} & DSLC     & \textbf{0.73}&	\textbf{0.37}&	\textbf{3.07}&\textbf{0.92}&\textbf{1.54}&	\textbf{5.33}&	\textbf{0.89}&	\textbf{2.1}&	\textbf{5.34}\\
          & DRHRL~\cite{zhang2025toward}     &   18.79&	136.1&	103.19&	14.08&	59 &	145.24&	20.57&	203.65&	68.4\\
          & DRMPC~\cite{zhou2022event}     &   90.75&	930.27&	507.33&	83.44&	1020.81&	472.49&	119.47&	1418.92&	911.39 \\
    \midrule
    \multirow{3}[2]{*}{$\beta=0.4$} & DSLC     & \textbf{0.64}&	\textbf{0.31}&	4.4&	\textbf{0.53}&	\textbf{0.71}&	\textbf{3.2}&	\textbf{1}&	\textbf{2.56}&	\textbf{3.92}\\
          & DRHRL~\cite{zhang2025toward}     &   8.94&	68.81&	75.32&	14.77&	135.16&	50	&21.91&	110.22&	192.65 \\
          & DRMPC~\cite{zhou2022event}     &   4.43&	43.92&	\textbf{1.43}&	22.13&	204.24&	38.51&	141.02&	1618.08&	1198.06\\
    \midrule
    \multirow{3}[2]{*}{$\beta=0.5$} & DSLC   &   \textbf{0.62}&	\textbf{0.35}&	5.11&	\textbf{0.81}&	\textbf{1.96}&	\textbf{2.82}&	\textbf{1.31}&	\textbf{4.46}&	\textbf{5.47}\\
          & DRHRL~\cite{zhang2025toward}     &      22.54&	219.36&	256.33&	7.12&	46.89&	19.6&	20.26&	198.34&	138.21 \\
          & DRMPC~\cite{zhou2022event}     &    5.64&	55.29&	\textbf{2.25}&	8.1&	106.69&	6.15&	142.06&	1866.72&	1460.03\\
    \midrule
    \multirow{3}[2]{*}{$\beta=0.6$} & DSLC    & \textbf{0.56}&	\textbf{0.59}&	1.57&	\textbf{0.69}&	\textbf{1.54}&	\textbf{3.18}&\textbf{	1.21}&	\textbf{3.13}&	\textbf{6.13}\\
          & DRHRL~\cite{zhang2025toward}     &    17.2&	160.01&	97.27&	26.3&	275.65&	326.65&	21.77&	207.97&	441.99\\
          & DRMPC~\cite{zhou2022event}     &   2.07&	8.96&	\textbf{0.74}&	170.55&	2156.23&	1914.15&	157.42&	2070.78&	1747.56\\
    \midrule
    \multirow{3}[2]{*}{$\beta=0.7$} & DSLC     & \textbf{0.6}&	\textbf{0.32}&	\textbf{4.18}&	\textbf{0.86} &	\textbf{1.35}&	\textbf{5.02}&	\textbf{0.97}&	\textbf{2.65}&	\,{\textbf{4.38}}\\
          & DRHRL~\cite{zhang2025toward}      &   44.13&	243.14&	442.89&	20.73&	198.72&	164.22&	20.42&	199.1&	309.31\\
          & DRMPC~\cite{zhou2022event}     &    130.73&	1590.58&	1099.49&	17.53&	201.57&	18.59&	258.31&	3071.1&	4011.04\\
    \midrule
    \multirow{3}[2]{*}{$\beta=0.8$} & DSLC    &\textbf{ 0.58}&	\textbf{0.3}&	\textbf{4.87}&	\textbf{0.98}&	\textbf{2.59}&	\textbf{3.14}&	\textbf{0.98}&	\textbf{3.11}&	\textbf{5.53}\\
          & DRHRL~\cite{zhang2025toward}     &   26.38&	158.87&	114.48&	41.83&	304.79&	343.37&	21.47&	186.9&	60.76\\
          & DRMPC~\cite{zhou2022event}     &   101.18&	1438.59&	884.88&	54.91&	771.02&	251.79&	284.34&	3327.38&	4833.78\\
    \midrule
    \multirow{3}[2]{*}{$\beta=0.9$} & DSLC     &  \textbf{0.59}&	\textbf{0.35}&	\textbf{7.1}&\textbf{	0.81}&	\textbf{1.58}&	\textbf{5.19}&\textbf{1.36}&	\textbf{4.92}&	\textbf{4.95}\\
          & DRHRL~\cite{zhang2025toward}      &   22.49&	224.52&	360.27&	30.49&	328.23&	449.88&	63.04&	665.82&	815.02\\
          & DRMPC~\cite{zhou2022event}     &    118.74&	1640.72&	1136.03&	96.59&	1174.04&	664.81&	79.17&	932.58&	466.19\\
    \midrule
    \multirow{3}[2]{*}{$\beta=1$} & DSLC    &   \textbf{0.61}&	\textbf{0.28}&	\textbf{3.69}&	\textbf{0.61}&	\textbf{1.06}&	\textbf{3.03}&	\textbf{1.34}&	\textbf{3.75}&	\textbf{7.5}\\
          & DRHRL~\cite{zhang2025toward}      &    94.1&	792.94&	992.42&	39.67&	447.36&	321.14&	25.1&	304.52&	321.6\\
          & DRMPC~\cite{zhou2022event}     &  39.8&	590.23&	157.94&	21.76&	323.9&	48.84&	17.6&	179.69&	18.76\\
    \bottomrule
    \end{tabular}}
  \label{access-control-performance}%
\end{table*}%
 
 \begin{figure}[http]
	\centering
	\includegraphics[width=0.8\columnwidth]{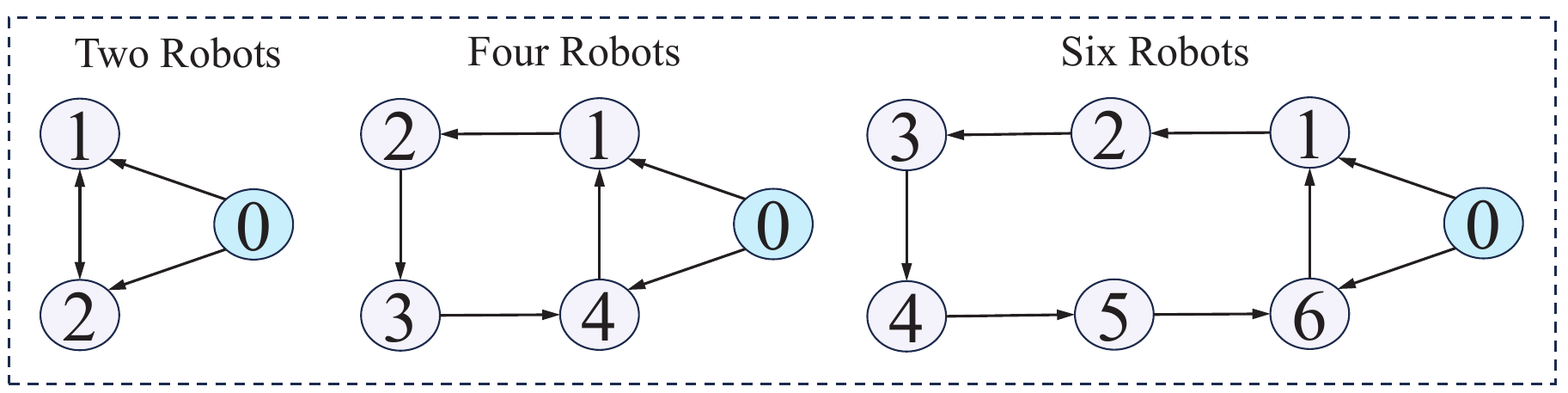}
	\caption{
		\mc{Communication graphs of multirobots, where Robot 0 is the virtual leader.}}
	\label{communication-graph}
\end{figure}

\begin{figure}[h]
					\centering
     \includegraphics[width=0.9\columnwidth]{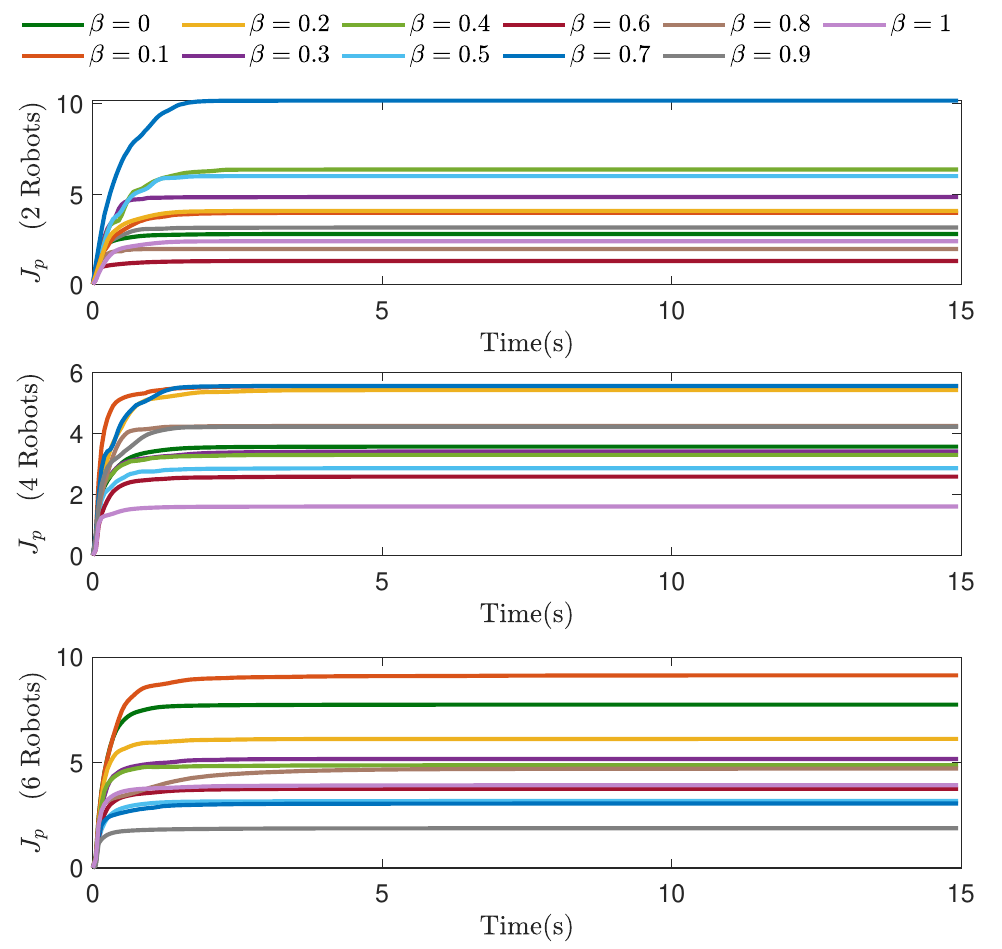}
				\caption{The curves of the cost to go $J_p = \int_{\tau=0}^{t_{\rm sim}} r(\tau)d\tau$ for formation control under different attack probabilities, where $t_{\rm sim}$ is the terminal time of the simulation period.}
				\label{index-Jc}
\end{figure}

\begin{figure*}[http]
					\centering
     \includegraphics[width=1.9\columnwidth]{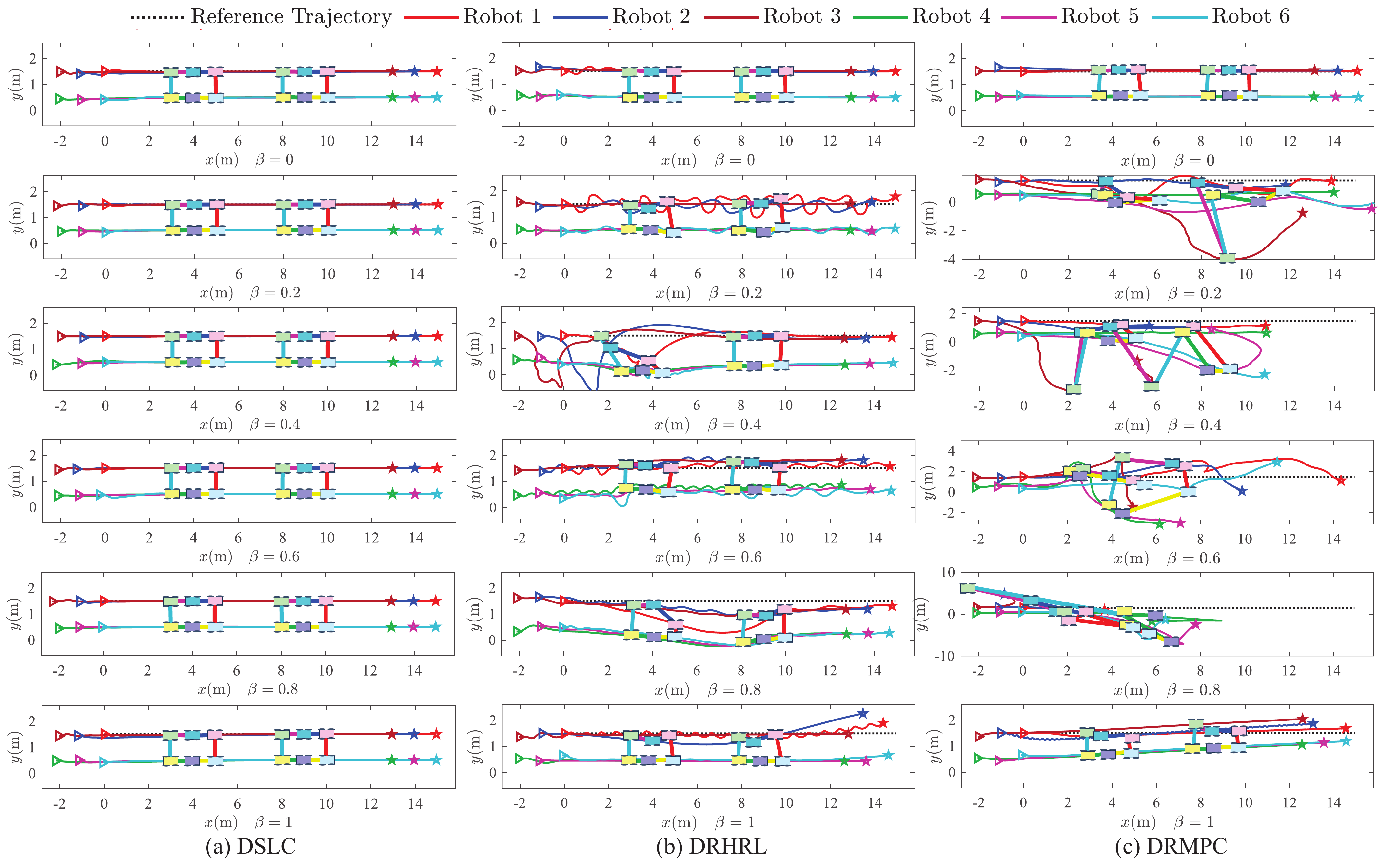}
    \caption{{The trajectories of six-robot formation with  different random attack probabilities under DRHRL~\cite{zhang2025toward}, DRMPC~\cite{zhou2022event}, and DSLC.}}
				\label{six-robots}
\end{figure*}
\begin{figure*}[t]
	\centering
	\includegraphics[width=1.8\columnwidth]{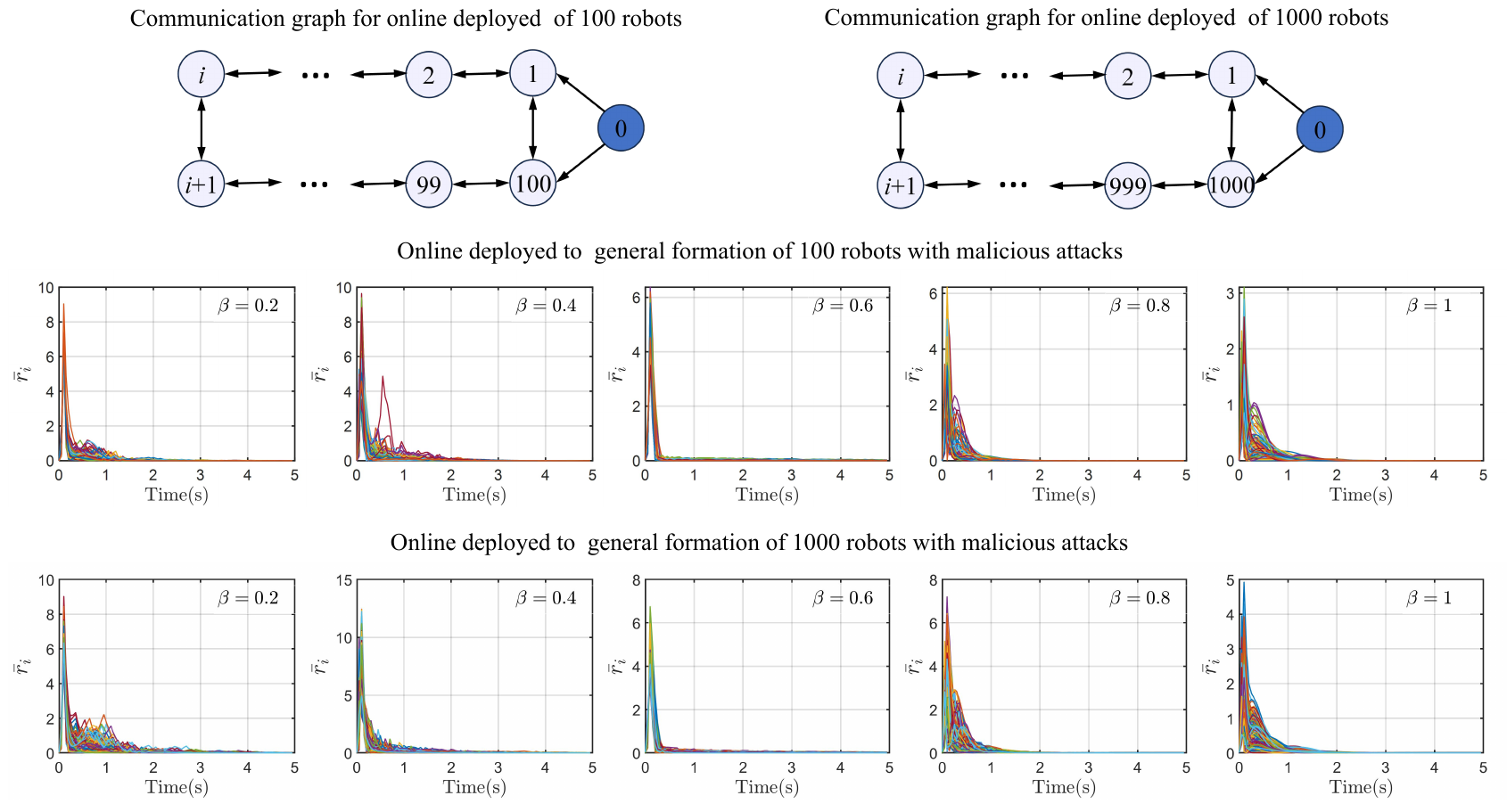}
	\caption{
		{\color{black}The curves of $\bar r_i=\| x_{\scriptscriptstyle \mathcal{N}_i}\|_{Q_i}^{2}+\| u_{s,i}\|_{R_{s,i}}^{2}$ for general formation of 100 and 1000 robots under different random probabilistic attacks using the policies learned from two robots.}}
	\label{deploy-100-robots}
\end{figure*}
\section{Simulation Studies}
	In this section, simulation studies are performed to verify the effectiveness of DSLC for MRS under actuator attacks.  Meanwhile, the scalability and practical deployment capability of DSLC are also validated through simulations on MRS with different scales up to 1000.

 \subsection{Parameter Setting and Evaluation Criteria}
 In the simulation studies, 
 the parameters of DSLC were set as 
 \begin{equation}\label{control_parameter_1}
     \begin{aligned}
&\eta_{c,i}=0.02,\eta_{s,i}=0.01,\eta_{a,i}=0.01,\;T=0.05\, {\rm s},\\
     &t_p=20T,Q_{i}=I_8,R_{s,i}=0.5I_2,R_{a,i}=0.1 I_2.
     \end{aligned}
 \end{equation}
 {\color{black}The above parameters were chosen empirically, subject to certain constraints. Specifically, $\eta_{c,i}$ must be chosen larger than $\eta_{s,i}$ and $\eta_{a,i}$ to satisfy the theoretical condition~\eqref{Eqn:convergence}, as detailed in Appendix~\ref{app:closed-loop}. In addition,  $Q_{i}$, $R_{s,i}$, and $R_{a,i}$ were chosen to satisfy condition~\eqref{Lyapunov-equation}.}
 To simplify the experimental setup, all attack probabilities are set the same, i.e., $\beta=\beta_i,\, \forall i\in\mathbb{N}_1^M$. The DSLC approach was comprehensively evaluated under different values of $\beta$, {\color{black}while the complementary validations with different attack probabilities among robots are deferred in Fig.~\ref{six-robots-different} of Appendix~\ref{Auiliary-results}.}
  {\color{black} Furthermore, the following indices, including the Integral of Absolute Error (IAE) and the Integral of Time-weighted Absolute Error (ITAE)}, were introduced to quantitatively evaluate the control performance under attacks, i.e.,
 \begin{equation}\label{assess_control_performance}
     \begin{aligned}\text{IAE}=&\int_{\tau=0}^{t_{\text{sim}}}\|x\|d\tau,\;
     \text{ITAE}=\int_{\tau=0}^{t_{\text{sim}}}\tau\|x\|d\tau,\\
     J_c=&\int_{\tau=0}^{t_{\text{sim}}} \lVert x \rVert^2_{\scriptscriptstyle Q} + \lVert u \rVert^2_{\scriptscriptstyle R} d\tau, 
     \end{aligned}
 \end{equation}
 where $ Q=I_{\scriptscriptstyle 4M},R=0.5I_{\scriptscriptstyle2M}$, and $t_\text{sim}$ is the terminal time of the simulation period.
 
\subsection{Comparative Results}
We initially performed comparative studies for MRS in three scenarios: (1) two robots, (2) four robots, and (3) six robots.  In the comparative studies, the attack probabilities were set as $\beta=0,0.1,0.2,\cdots,1$ for each local robot. The associated communication graphs are shown in Fig.~\ref{communication-graph}. {\color{black}As discussed in Section~\ref{sec:related work}, no existing studies have addressed distributed optimal control for nonlinear multi-robot systems (MRS) under malicious, stealthy actuator attacks.} Therefore, we compare our approach with two baseline distributed optimal control methods, namely distributed receding-horizon reinforcement learning (DRHRL)~\cite{zhang2025toward} and distributed robust MPC (DRMPC)~\cite{zhou2022event}, to validate the effectiveness of DSLC in secure control under attacks, as well as its generalizability across various tasks and scalability to large-scale MRS.  In the simulation studies, the initial states of MRS were set as uniformly random values, with the distance to its expected value in the interval $[-0.5,\, 0.5]$. In our approach and DRHRL, the control policies were continuously updated online, while the control actions of DRMPC were calculated with the \emph{fmincon} solver. 
Table~\ref{access-control-performance} presents the performance criteria of the three methods (DSLC, DRHRL~\cite{zhang2025toward}, DRMPC~\cite{zhou2022event}) under the same conditions. It can be seen that both the IAE and ITAE of DSLC are much less than those of DRHRL and DRMPC. In most cases, the cost $J_c$ achieved by DSLC is significantly lower than that obtained by DRHRL and DRMPC, {\color{black}owing to the game-theoretic secure learning control mechanism underlying the proposed approach.} 
 
 The curves of the cost to go $J_p= \int_{\tau=0}^{t_{\rm sim}} r(\tau)d\tau$ are shown in Fig.~\ref{index-Jc}. The results show that, although attacks with different probabilities are imposed on the MRS, the cost to go $J_p$ under attacks asymptotically converges to constant values. 
 The comparison results in the six-robot scenario with $\beta =0,0.2,0.4,\cdots,1$ are given in Fig.~\ref{six-robots}, while the results in the two-robot and four-robot scenarios are deferred to Figs.~\ref{two-robots} and~\ref{four-robots} in Appendix~\ref{Auiliary-results}. The results show that MRS with all methods could track the reference trajectory and maintain the desired formation in the absence of attacks, i.e., $\beta=0$.
 The multirobots exhibit formation disorder and failure in trajectory tracking under DRHRL and DRMPC, when there are attacks with different probabilities.
 In contrast, DSLC can effectively control the MRS to track the reference trajectory, even when attacks with different probabilities arise and the number of robots varies.

\subsection{Scalability Validation of DSLC under Attacks}
To verify the scalability and practical deployment capability of DSLC, 
the policies learned from just two robots under attacks were directly deployed to 100 robots and 1000 robots.
As shown in Fig.~\ref{deploy-100-robots}, the policies learned from just two robots can successfully stabilize 100 and even 1000 robots under various attack conditions, demonstrating the scalability and practical deployment capability of the proposed approach.
This result indicates that control policies for large-scale MRS can be efficiently generated by training on small-scale systems, substantially reducing computational and deployment complexity. {\color{black}Note also that, the closed-loop dynamic behaviors may vary across different values of $\beta$, due to the varying levels of stochasticity introduced by distinct attack probability distributions.}
In general, the reported simulation results validate the effectiveness of DSLC for large-scale MRS under malicious cyber attacks.

\section{Real-World Experiments}
In this section, real-world experiments are carried out to further validate the effectiveness and generalizability of DSLC under malicious actuator attacks in various coordination scenarios, including general formation, affine formation, and containment.

The experimental platform is illustrated in Fig.~\ref{d_System_Attack}. It consists of multiple wheeled mobile robots, a vision-based localization system, and a host computer (see Fig.~\ref{d_System_Attack}(a)). {\color{black}Each mobile robot adopts a differential-drive configuration with dimensions of 119.75 mm $\times$ 105.01 mm $\times$ 79.07 mm.} The robot states were measured in real time using a Hikvision MV-CA013-21UM camera {\color{black}at a sampling frequency of 50 Hz. With the same frequency, the defense policies~\eqref{eqn:actor-hat} learned in simulation were implemented on the host computer equipped with an Intel i7-8565U CPU. The calculated control commands were transmitted to the STM32 microcontroller of each mobile robot via a WiFi 6 network using the UDP protocol.} As illustrated in Fig.~\ref{d_System_Attack}(b), each attacker utilized local state information to inject stealthy deception attack signals into the actuator of an individual robot. {\color{black}In all experiments, the learned attack policies~\eqref{eqn:attack-hat} were utilized as adversaries to destabilize MRS.}


\begin{figure*}[http]
	\centering
	\includegraphics[width=1.8\columnwidth]{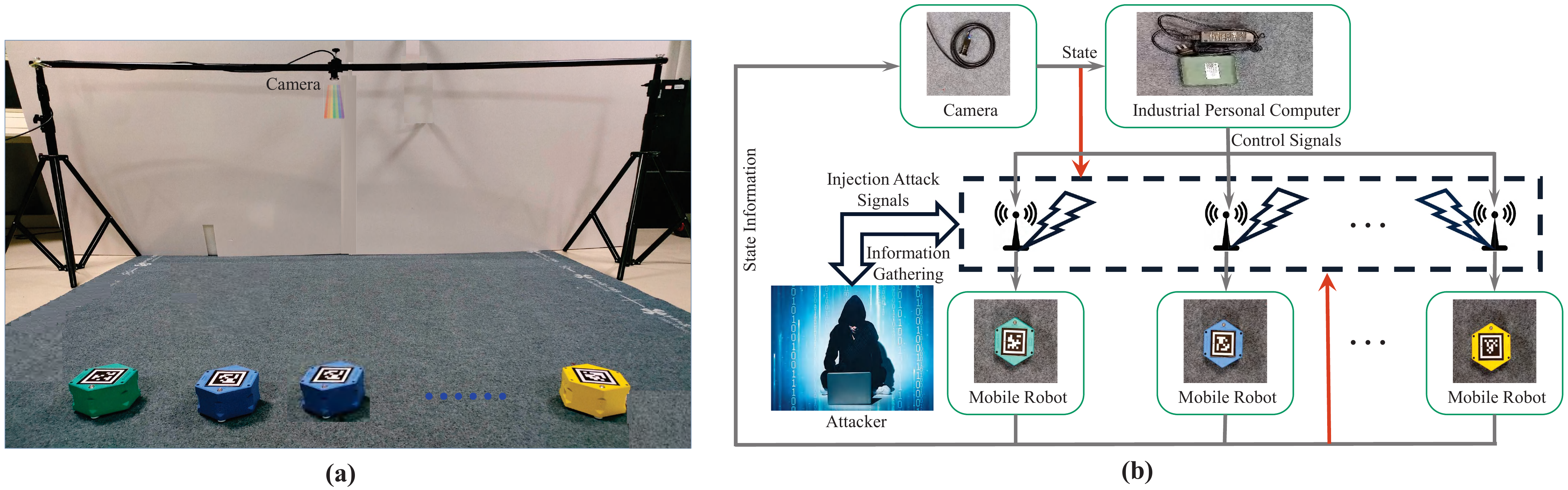}
	\caption{Schematic diagram of the MRS experimental platform under additive deception actuator attacks. (a) The scene of the experimental platform; (b) The information flow for MRS under attacks.}
	\label{d_System_Attack}
\end{figure*}
\subsection{General Formation Control under Attacks}
\begin{figure}[htb]
	\centering
	\includegraphics[width=1\columnwidth]{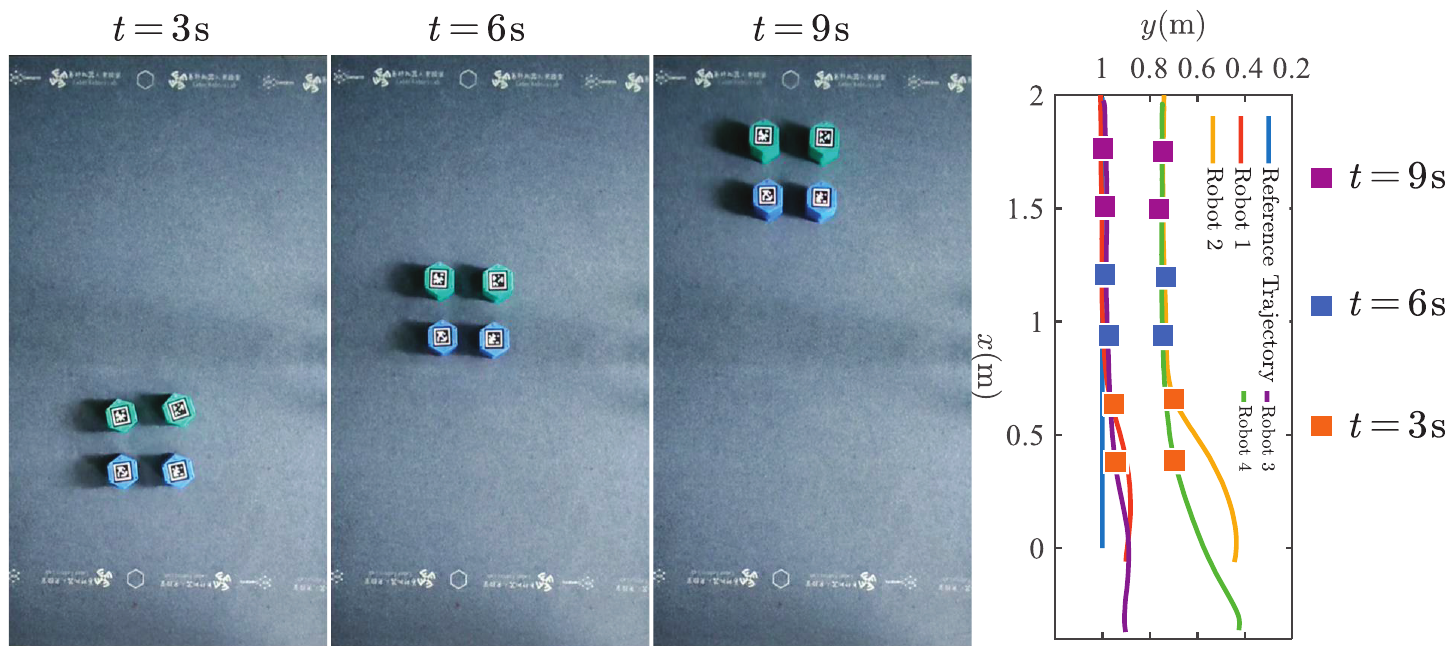}
	\caption{In Scenario I, trajectories and snapshots of 4 robots without attacks ($\beta=0$) under DSLC. }
	\label{R4LB00}
\end{figure}
\begin{figure}[htb]
	\centering
	\includegraphics[width=0.8\columnwidth]{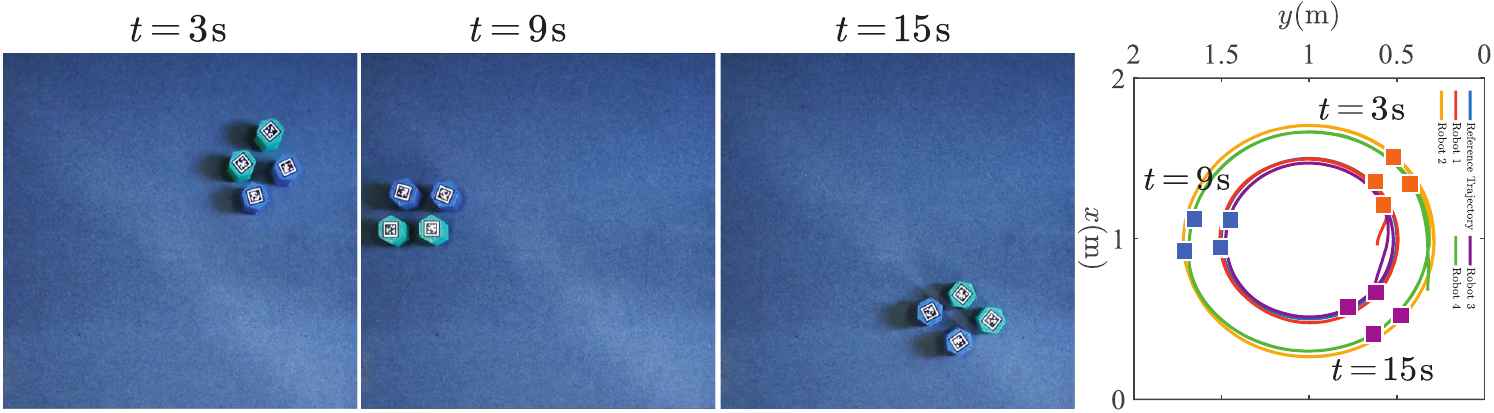}
	\caption{In Scenario II, trajectories and snapshots of 4 robots without attacks ($\beta=0$) under DSLC.}
	\label{R4CB00}
\end{figure}
\begin{figure}[htb]
	\centering
	\includegraphics[width=0.9\columnwidth]{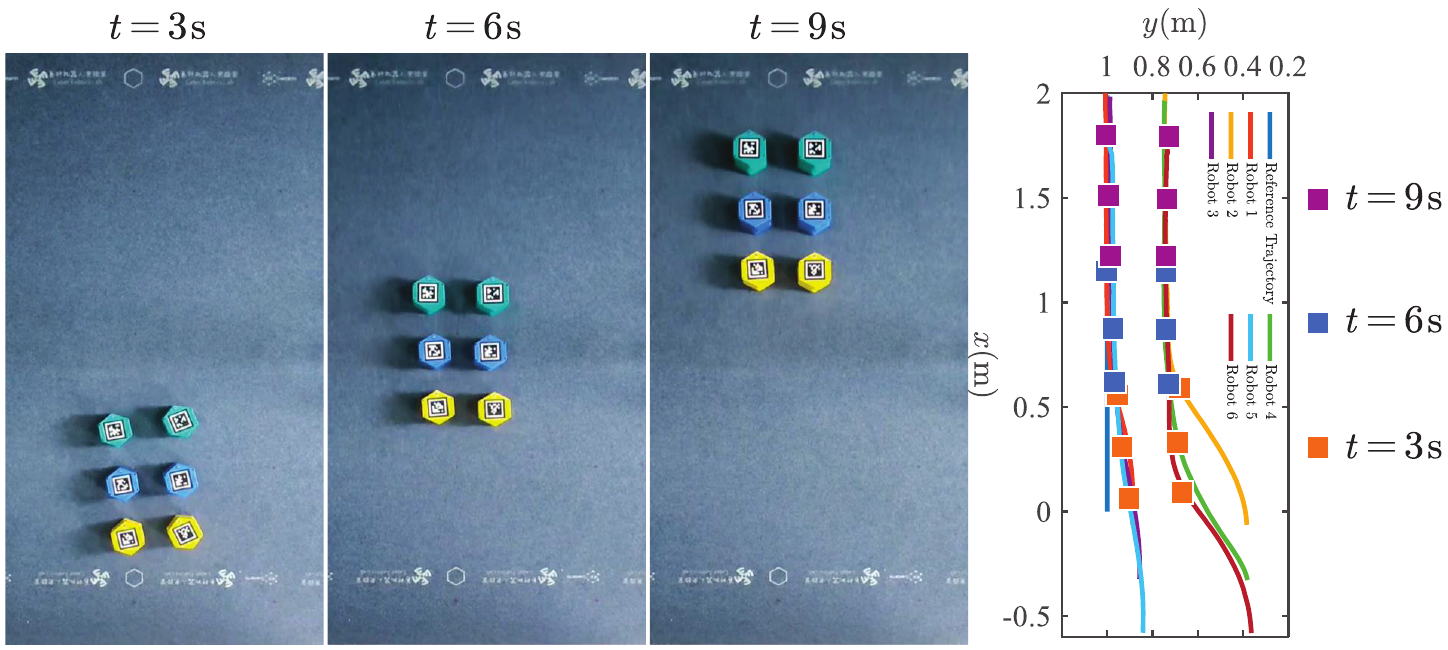}
	\caption{Scenario I: trajectories and snapshots of 6 robots without attacks ($\beta=0$) under DSLC. }
	\label{R6LB00}
\end{figure}
\begin{figure}[htb]
	\centering
	\includegraphics[width=0.9\columnwidth]{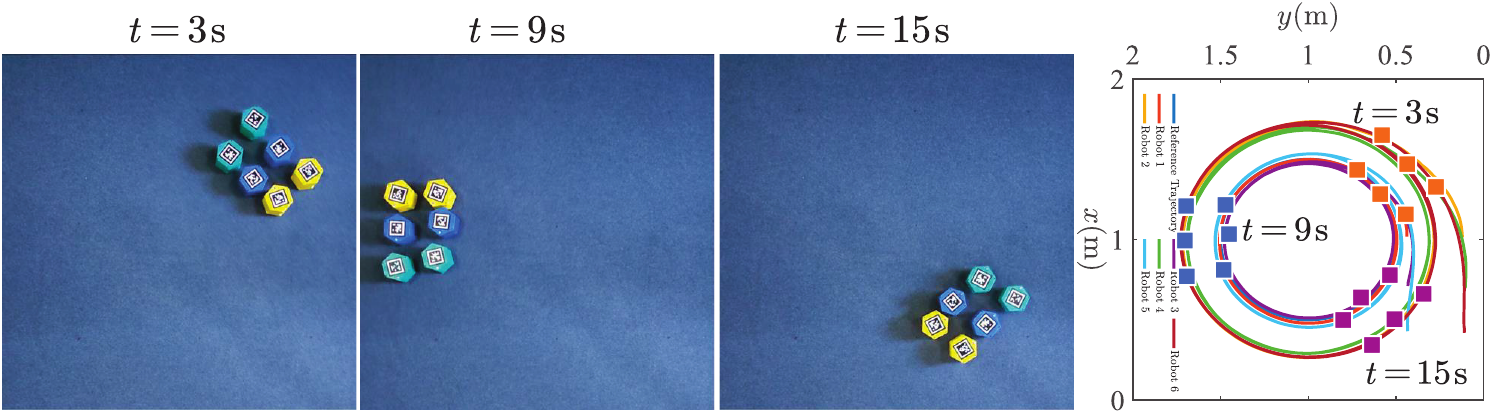}
	\caption{Scenario II: trajectories and snapshots of 6 robots without attacks ($\beta=0$) under DSLC. }
	\label{R6CB00}
\end{figure}
\begin{figure}[htb]
	\centering
	\includegraphics[width=0.9\columnwidth]{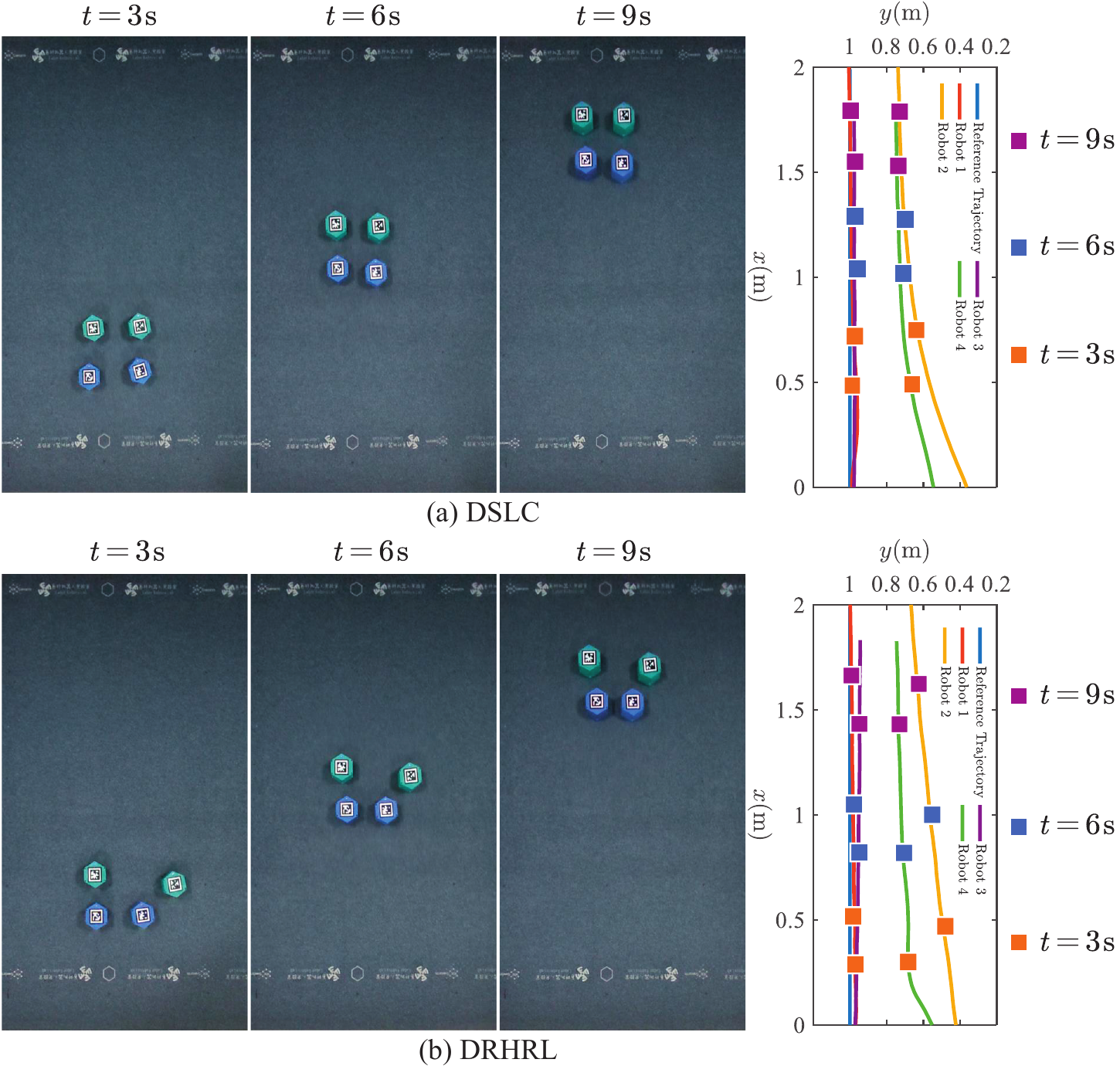}
	\caption{Scenario I: trajectories and snapshots of 4 robots with an attack probability of $\beta=0.5$. (a) DSLC; (b) DRHRL. The robots under DSLC formed and maintained a stable rectangular formation, whereas the formation under DRHRL was disrupted by the attack.}
	\label{R4LB05}
\end{figure}
\begin{figure}[htb]
	\centering
	\includegraphics[width=0.9\columnwidth]{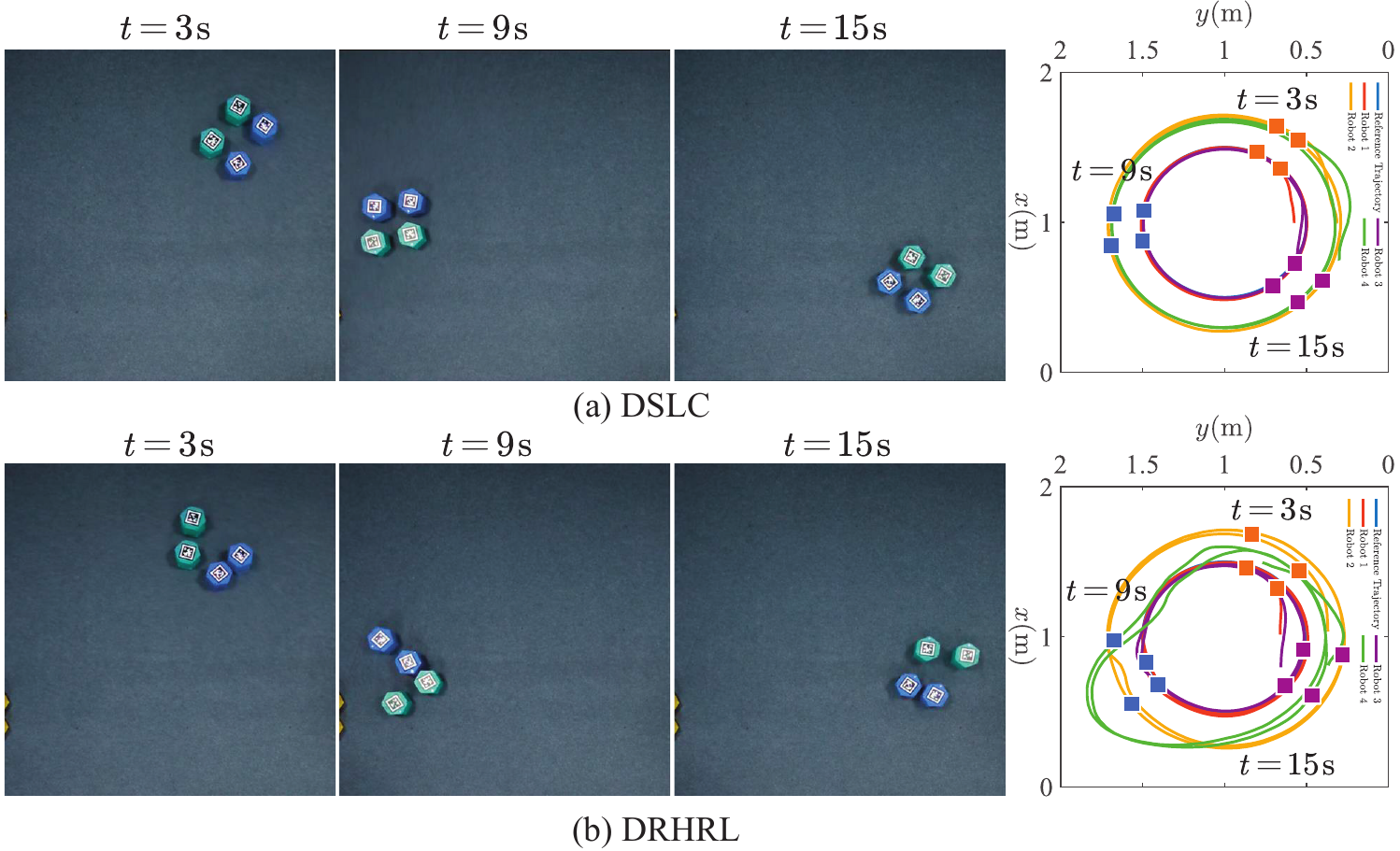}
	\caption{Scenario II: trajectories and snapshots of 4 robots with an attack probability of $\beta=0.5$. (a) DSLC; (b) DRHRL.  The robots under DSLC formed and maintained a stable rectangular formation, whereas the formation under DRHRL was disrupted by the attack.}
	\label{R4CB05}
\end{figure}
\begin{figure}[htb]
	\centering
	\includegraphics[width=0.9\columnwidth]{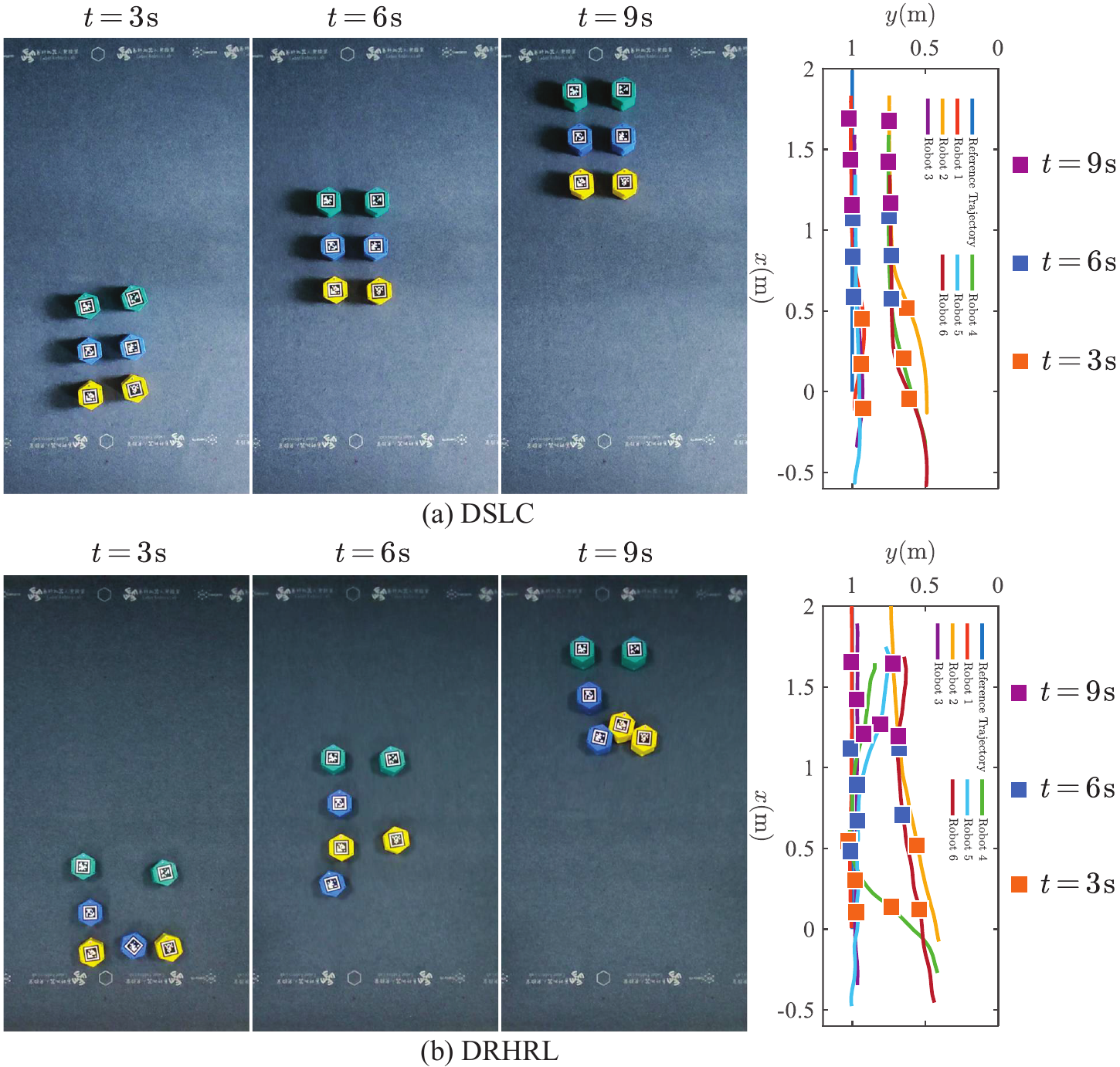}
	\caption{Scenario I: trajectories and snapshots of 6 robots with an attack probability of $\beta=0.5$. (a) DSLC; (b) DRHRL.  The robots under DSLC formed and maintained a stable rectangular formation, whereas the formation under DRHRL was disrupted by the attack.}
	\label{R6LB05}
\end{figure}
\begin{figure}[htb]
	\centering
	\includegraphics[width=0.9\columnwidth]{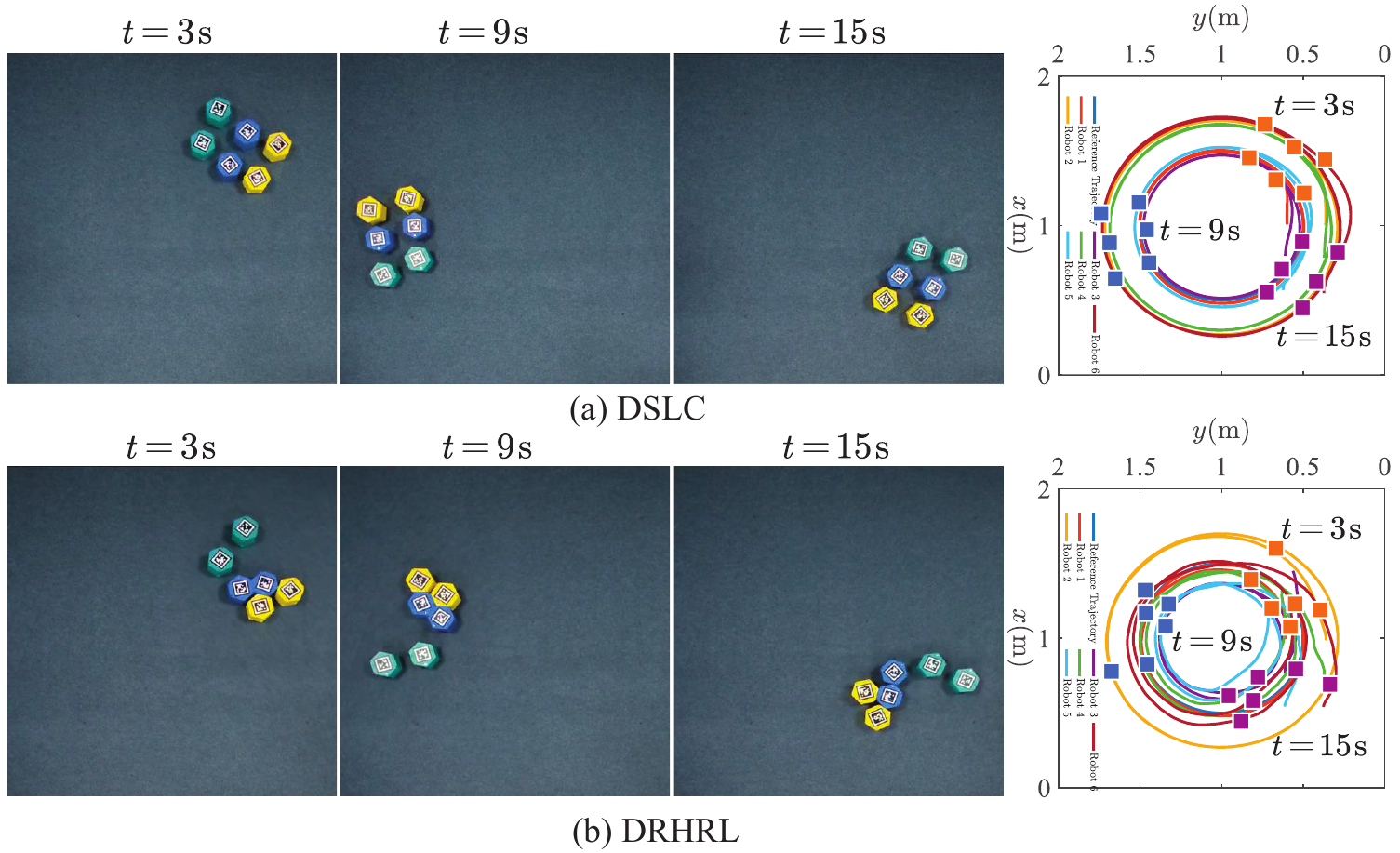}
	\caption{Scenario II: trajectories and snapshots of 6 robots with an attack probability of $\beta=0.5$. (a) DSLC; (b) DRHRL.  The robots under DSLC formed and maintained a stable rectangular formation, whereas the formation under DRHRL was disrupted by the attack.}
	\label{R6CB05}
\end{figure}
\begin{figure}[htb]
	\centering
	\includegraphics[width=0.9\columnwidth]{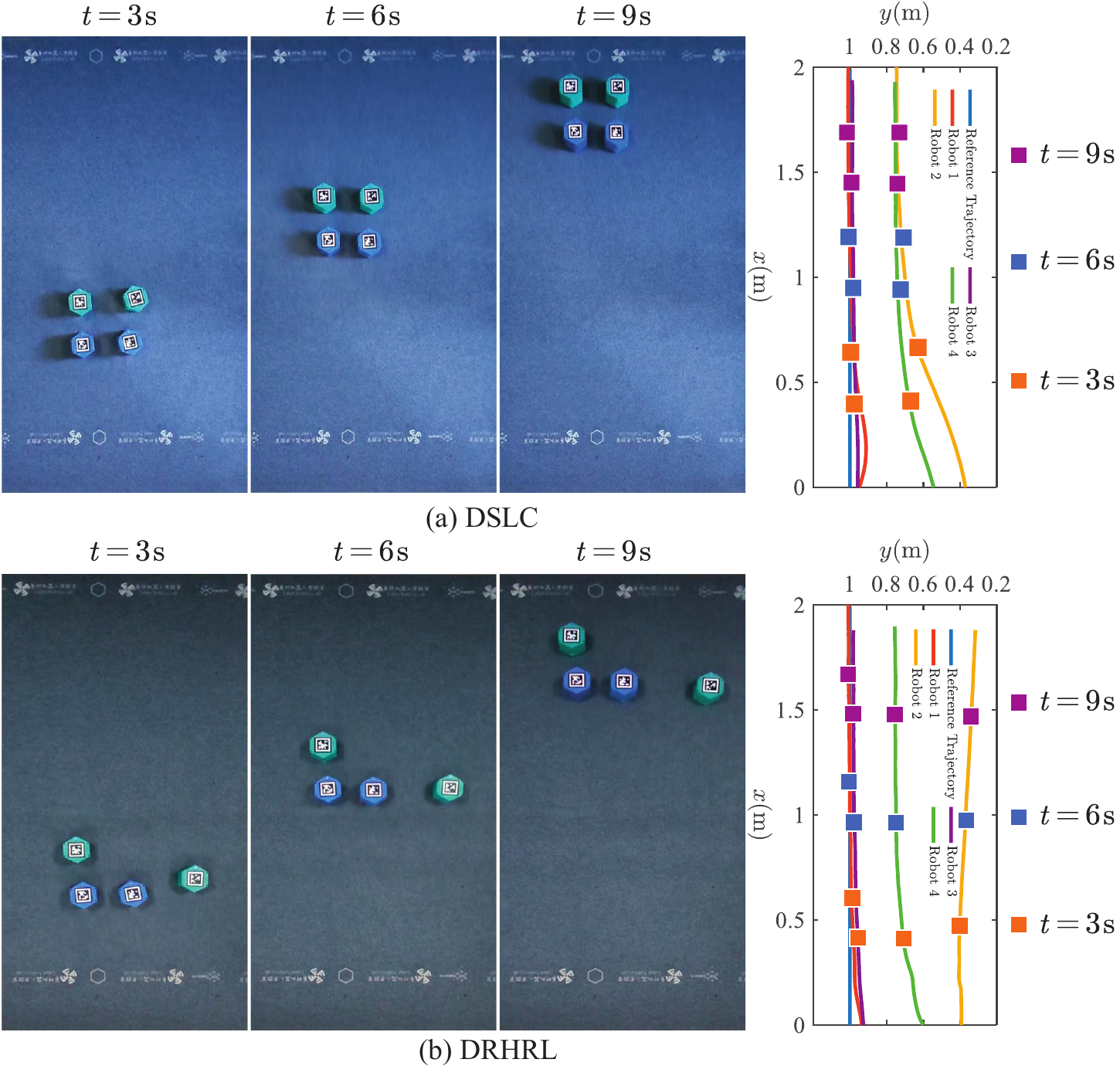}
	\caption{Scenario I: trajectories and snapshots of 4 robots with an attack probability of $\beta=1$. (a) DSLC; (b) DRHRL.  The robots under DSLC formed and maintained a stable rectangular formation, whereas the formation under DRHRL was disrupted by the attack.}
	\label{R4LB10}
\end{figure}
\begin{figure}[htb]
	\centering
	\includegraphics[width=0.9\columnwidth]{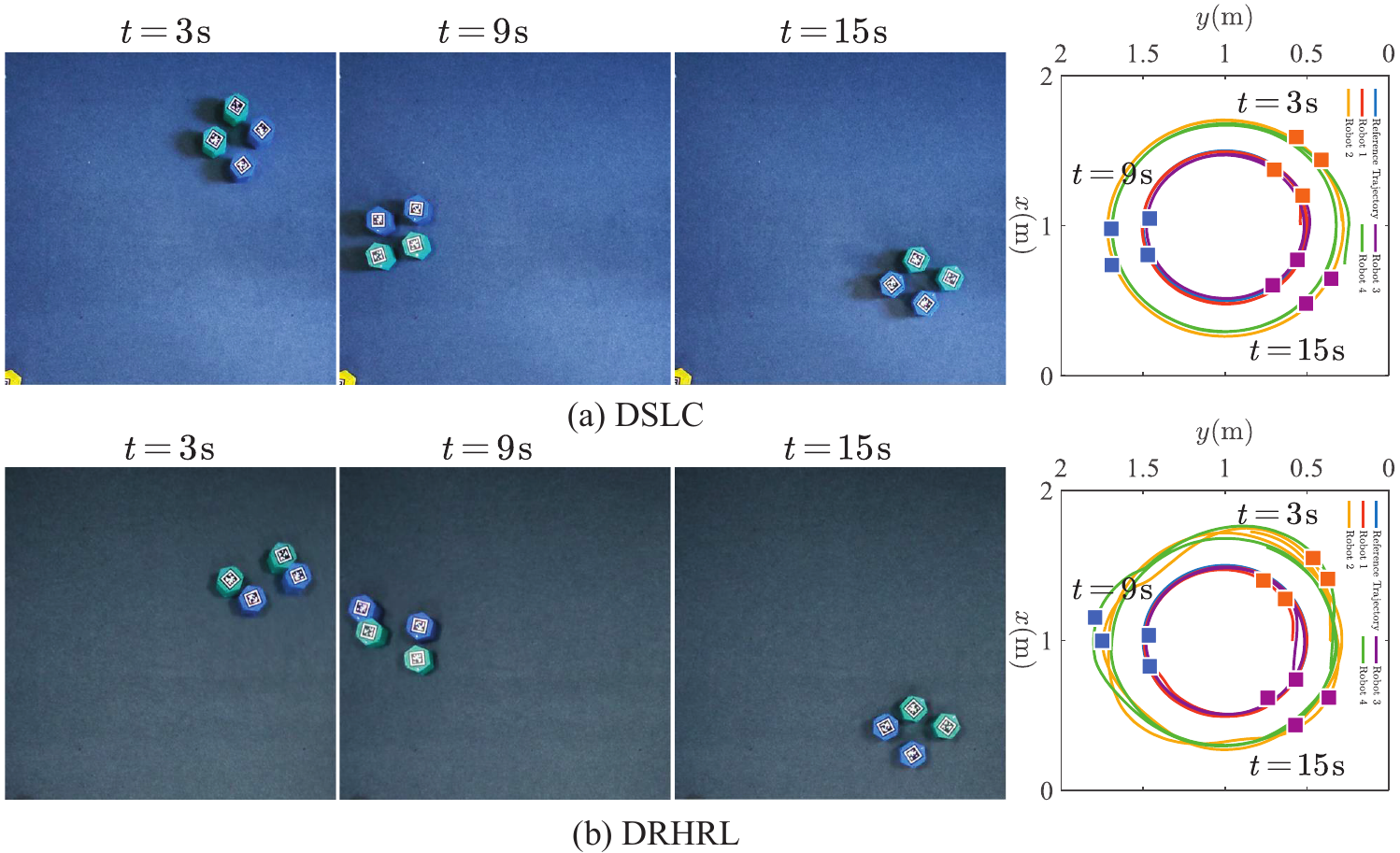}
	\caption{Scenario II: trajectories and snapshots of 4 robots with an attack probability of $\beta=1$. (a) DSLC; (b) DRHRL.  The robots under DSLC formed and maintained a stable rectangular formation, whereas the formation under DRHRL was disrupted by the attack.}
	\label{R4CB10}
\end{figure}
\begin{figure}[htb]
	\centering
	\includegraphics[width=0.9\columnwidth]{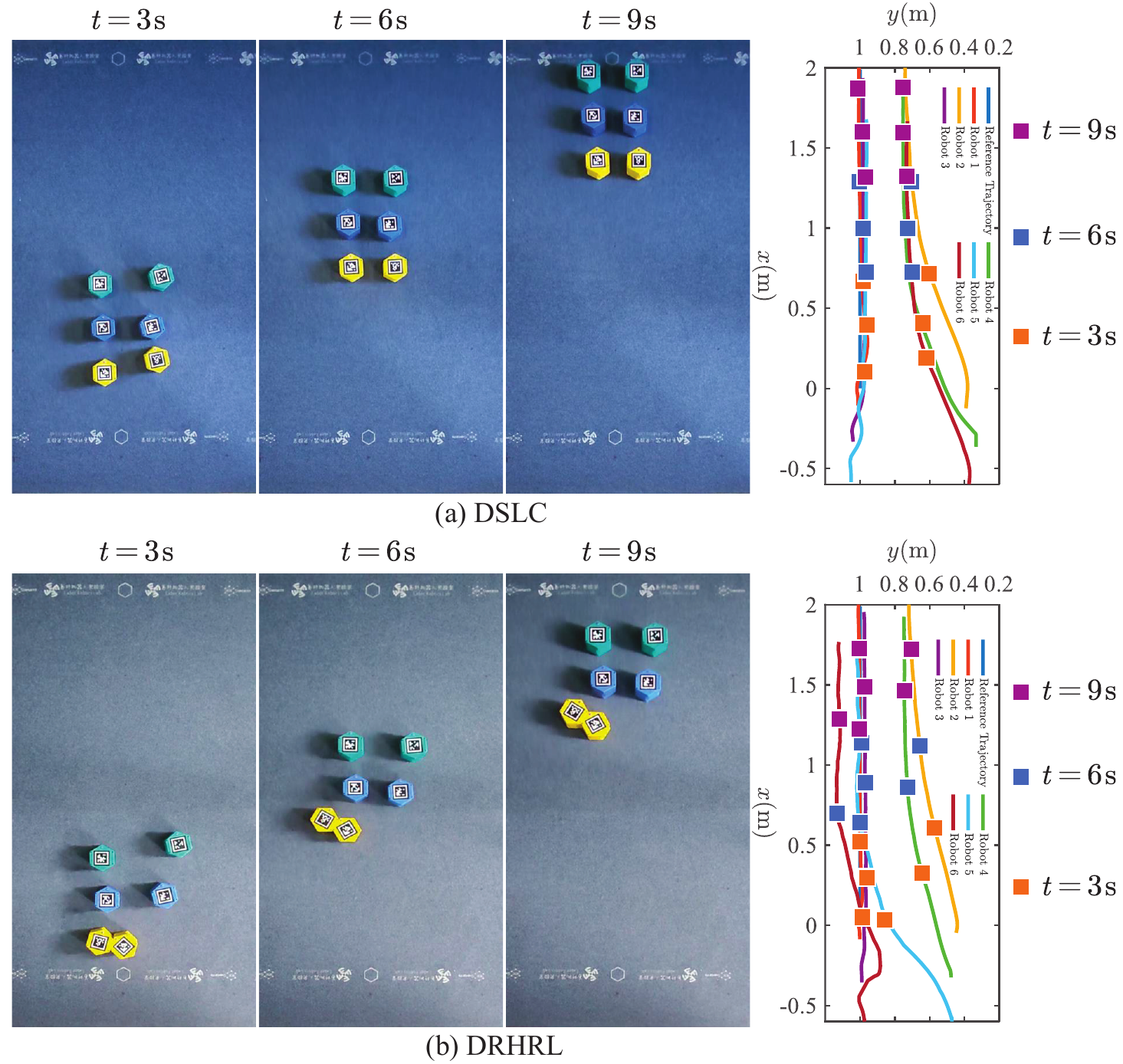}
	\caption{Scenario I: trajectories snapshots of 6 robots with an attack probability of $\beta=1$. (a) DSLC; (b) DRHRL.  The robots under DSLC formed and maintained a stable rectangular formation, whereas the formation under DRHRL was disrupted by the attack.}
	\label{R6LB10}
\end{figure}
\begin{figure}[htb]
	\centering
	\includegraphics[width=0.9\columnwidth]{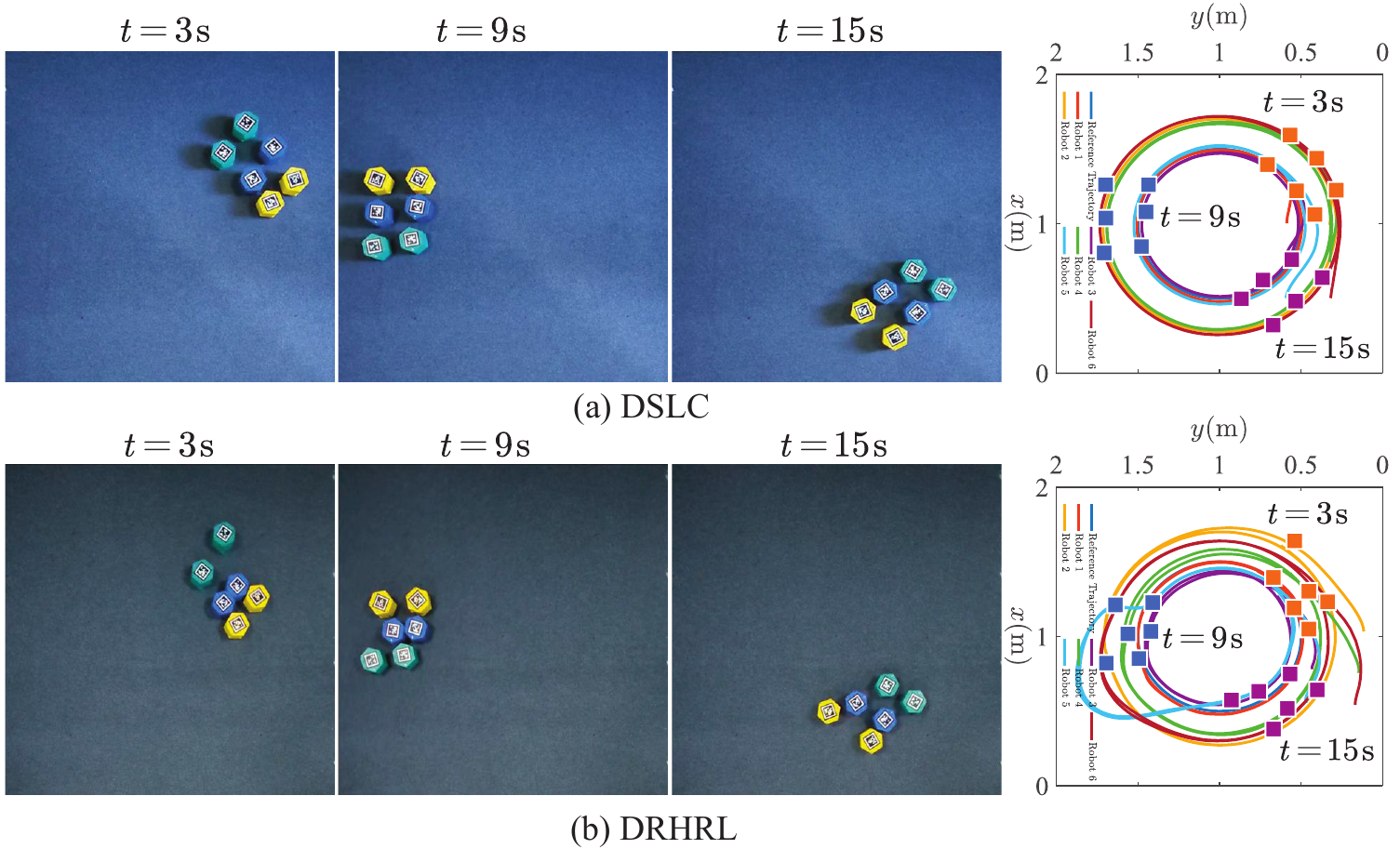}
	\caption{Scenario II: trajectories and snapshots of 6 robots with an attack probability of $\beta=1$. (a) DSLC; (b) DRHRL. The robots under DSLC formed and maintained a stable rectangular formation, whereas the formation under DRHRL was disrupted by the attack.}
	\label{R6CB10}
\end{figure}

 The experiments for general formation of MRS are conducted in two scenarios. In Scenario I, the MRS follows a straight-line reference trajectory, while in Scenario II, it tracks a circular reference trajectory. The communication graphs for general formation experiments are given in Fig.~\ref{communication-graph}. In the experiments, Robot 0 served as the virtual leader and its path matched the desired path of Robot 1. 
In Scenario I, the side length of the formation shape was set as 0.25 m, and in Scenario II, it was set as 0.2 m.
The control parameters in the general formation were set as in~\eqref{control_parameter_1}.

General formation experiments were conducted with different numbers of mobile robots ($M=$ 4 and 6) and various attack probabilities ($\beta = 0,\, 0.5,\, 1$). The experimental results are presented in Figs.~\ref{R4LB00}–\ref{R6CB10}. Therein,
Figs.~\ref{R4LB00}–\ref{R6CB00} show the experimental trajectories and snapshots of DSLC without attacks ($\beta = 0$) in Scenarios I and II. The results demonstrate that the MRS can efficiently track the reference trajectories while maintaining the desired formation shapes.
To further validate the robustness of DSLC under attacks, experiments were conducted with different robot scales and attack conditions in both scenarios. The corresponding MRS trajectories and snapshots are shown in Figs.~\ref{R4LB05}–\ref{R6CB10}. These results confirm that DSLC enables robots of varying scales to accurately track the reference trajectories while preserving formation integrity, even under different attack distributions.
In contrast, under the same experimental conditions, mobile robots controlled by DRHRL deviate from the reference trajectory, experience formation disruptions, and suffer collisions when attacked.

The above experimental results demonstrate that DSLC outperforms DRHRL~\cite{zhang2025toward} under various attack probabilities.
DSLC effectively enables robots of different scales to accomplish general formation control tasks, showcasing strong robustness against cyber attacks.

\begin{table*}[t]	
	\centering
	\caption{
 Control Performance under Attacks with Different Probabilities for General Formation in the Experiments
 }
	\label{d_index_experiment}
	\begin{tabular}{cccccc}
\toprule
\multirow{2}[2]{*}{Attack probabilities} &\multirow{2}[2]{*}{Approaches}& \multicolumn{2}{c}{4 Robots} & \multicolumn{2}{c}{6 Robots} \\
\cmidrule{3-6}
		&&  \multicolumn{1}{c}{Scenario (\romannumeral1)} & \multicolumn{1}{c}{Scenario (\romannumeral2)} & \multicolumn{1}{c}{Scenario (\romannumeral1)} & \multicolumn{1}{c}{Scenario (\romannumeral2)}  \\
		\midrule
		\multirow{2}[2]{*}{$\beta=0.5$} & DSLC     &\textbf{	1.6}&\textbf{	7.36}&	\textbf{1.83}&	\textbf{15.17}\\
		& DRHRL\cite{zhang2025toward}   &	1.97&	8.46&	5.09&	20.47\\
		\midrule
		\multirow{2}[2]{*}{$\beta=1$} & DSLC     &	\textbf{1.85}&	\textbf{7.22}&	\textbf{2.19}&	\textbf{14.92}\\
		& DRHRL \cite{zhang2025toward} &	3.22&	8.27&	7.3&	16.12\\
\bottomrule
 \end{tabular}
	\label{d_Index_Experiment}%
\end{table*}%
To quantitatively evaluate the control performance of the two approaches under attacks, the IAE performance index defined in~\eqref{assess_control_performance} is adopted for performance analysis.
Since the initial positions of the robots are randomly distributed, the IAE performance index is recorded from $t=1$ s to eliminate statistical errors caused by different initial conditions.
The results, summarized in Table~\ref{d_Index_Experiment}, show that the IAE values of DSLC are lower than those of DRHRL with $\beta=0.5$ and $\beta=1$, indicating that DSLC achieves superior control performance compared to DRHRL.

\subsection{Affine Formation Control under Attacks}
We have further validated our approach to achieve affine formation with attack probabilities $\beta=0.5,1$.
The MRS includes four real robots and one virtual robot which generates the reference trajectory described as
\begin{equation}
\begin{aligned}
x_r(t)=&\left\{\begin{array}{ll}
C_{x,1} +\frac{v_r}{\omega_r}\cos{(\omega_r t +\text{atan}2)}&t\in[0,T_1)\\
L_{x,1} -v_r(t-T_1)\cos{(\text{atan}2)}&t\in[T_1,\infty),\\
\end{array}\right.\\
y_r(t)=&\left\{\begin{array}{ll}
C_{y,1} +\frac{v_r}{\omega_r}\sin{(\omega_r t +\text{atan}2)}&t\in[0,T_1)\\
L_{y,1} -v_r(t-T_1)\sin{(\text{atan}2)}&t\in[T_1,\infty),\\
\end{array}\right.\\
\end{aligned}
\end{equation}
where $T_1=6.2$, $C_{x,1}=1.2$,  $C_{y,1}=1.15$, $L_{x,1}=0.7$, $L_{y,1}=1.4$, $v_r=0.2$, $\omega_r=0.36$ for $t\in[0,T_1)$
 and $\omega_r=0$ for $t\in[T_1,\infty)$.
The affine formation shape for the MRS was set as
\begin{equation}
\mathcal{F}_1=\left[\begin{matrix}0.1 & 0.2  & -0.2 & -0.2 & 0.2\\
0 & 0.2 & 0.2 & -0.2 & -0.2\\
\end{matrix}\right]^{\top}.
\end{equation} 
The Laplacian matrix for affine formation of the MRS was given as
\begin{equation}
\bm{L}^* = \left[  \begin{array}{ccc|cc}
0  & 0 & 0 & 0 & 0\\
-1 &  1 & 0 &  0&  0\\
0 &  -1 &  1 &  0&  0\\
\hline
 -2 &  1.5 &  -0.5 &  1&  0\\
-4 &  2 &  0 &  1&  1\\
\end{array} \right].
\end{equation}
The control parameters for the affine formation were set as in~\eqref{control_parameter_1}.

\begin{figure}[http]
 \centering
 \includegraphics[width=0.9\columnwidth]{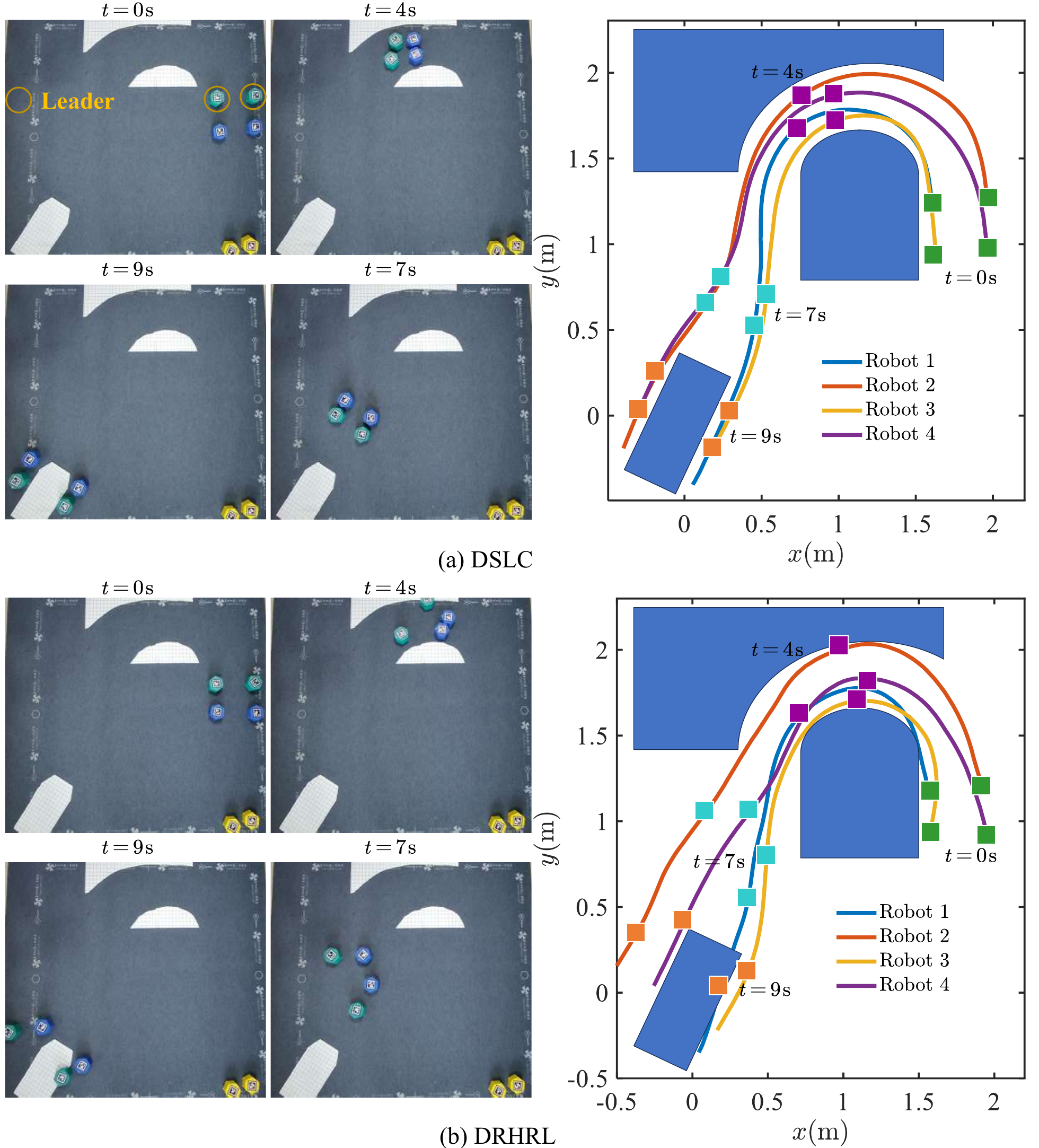}
\caption{Trajectories and snapshots for affine formation with an attack probability of $\beta=0.5$. (a) DSLC. (b) DRHRL.  The robots under DSLC formed a stable rectangular formation and dynamically expanded or contracted it to avoid collisions, whereas the robots under DRHRL failed to avoid the obstacles due to the attack.}
 \label{AF_B05_PC}
\end{figure}
\begin{figure}[http]
 \centering
 \includegraphics[width=0.9\columnwidth]{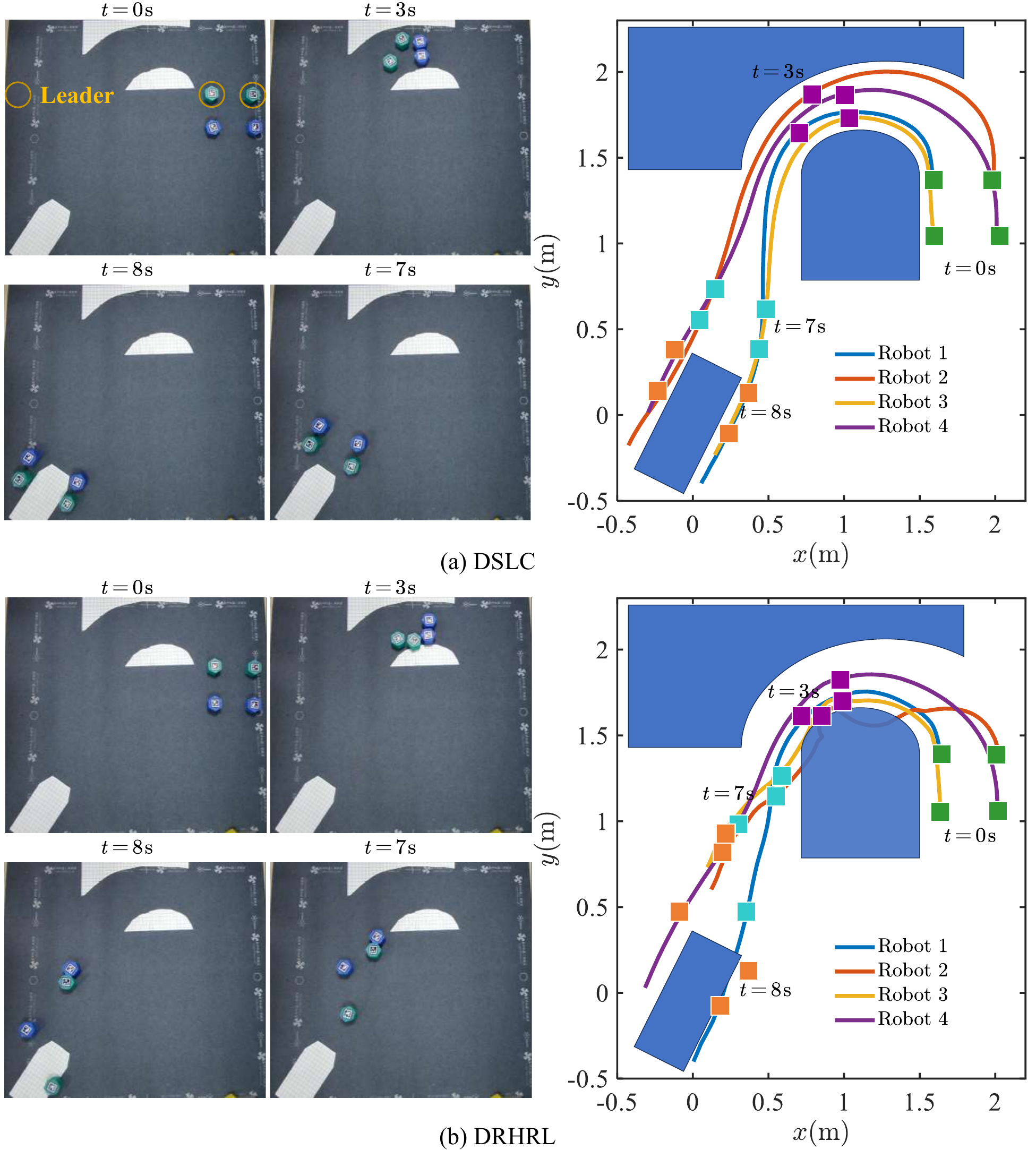}
\caption{Trajectories and snapshots for affine formation with an attack probability of $\beta=1$. (a) DSLC. (b) DRHRL.  The robots under DSLC formed a stable rectangular formation and dynamically expanded or contracted it to avoid collisions, whereas the robots under DRHRL failed to avoid the obstacles due to the attack.}
 \label{AF_B10_PC}
\end{figure}

The experimental results for the affine formation control of MRS are shown in Figs.~\ref{AF_B05_PC}–\ref{AF_B10_PC}.
The trajectories and snapshots of the MRS with an attack probability of $\beta=0.5$ are presented in Fig.~\ref{AF_B05_PC}(a). There are two obstacle regions: a narrow passage bordered by obstacles on both sides and a single obstacle located directly on the reference trajectory, both highlighted in blue (see Fig.~\ref{AF_B05_PC}(a)).
In the narrow passage, the multirobots, controlled by DSLC, perform an affine transformation to shrink the formation size and successfully navigate the confined space while maintaining the formation shape, as verified by the snapshot at $t=4$ s.
At $t=7$ s, the formation expands back to its original geometric size.
Subsequently, to avoid the obstacle on the reference path, the MRS again scales its formation to maneuver around the obstacle, as shown in the snapshot at $t=9$ s.

In contrast, as shown in Fig.~\ref{AF_B05_PC}(b), the mobile robots controlled by DRHRL experience formation disorder, collisions between robots, and collisions with obstacles.
Compared to DRHRL, DSLC can effectively control the MRS with an attack probability of $\beta=0.5$, allowing robots to follow the reference trajectory, maintain formation shapes, and safely navigate through obstacle areas.
Similar conclusions can be drawn as in Fig.~\ref{AF_B05_PC} for DSLC and DRHRL with $\beta=1$.
The IAE values, calculated using~\eqref{assess_control_performance}, for affine formation experiments under attack probabilities of $\beta=0.5$ and $\beta=1$ are summarized in Table~\ref{IAE_AffineFormation_Containment}.
The results show that the IAE values for DSLC are lower than those for DRHRL, and further demonstrate the superior performance and robustness of DSLC over DRHRL under cyber attacks.

\subsection{Containment Control under Attacks}
\begin{figure}[htb]
	\centering
	\includegraphics[width=0.9\columnwidth]{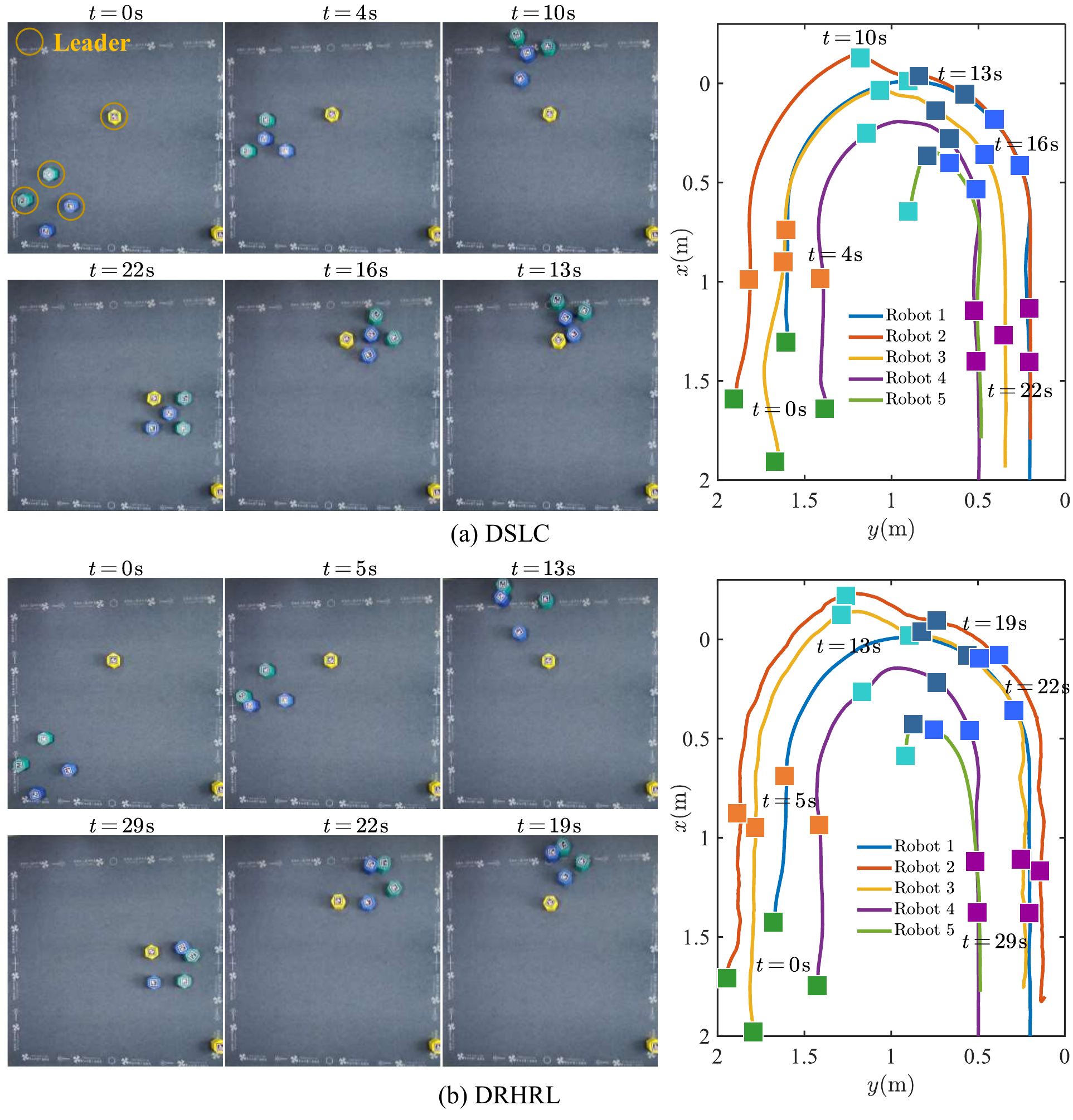}
	\caption{
    Trajectories and snapshots for containment with an attack probability of $\beta=0.5$. (a) DSLC. (b) DRHRL. The robots under DSLC formed a stable containment shape and successfully reconfigured to accommodate a newly joined robot. In contrast, the follower robot under DRHRL failed to effectively enter the convex hull and maintain safe distances, resulting in collisions.
    }
	\label{ContainmentB05_P_C}
\end{figure}
\begin{figure}[htb]
	\centering
	\includegraphics[width=0.9\columnwidth]{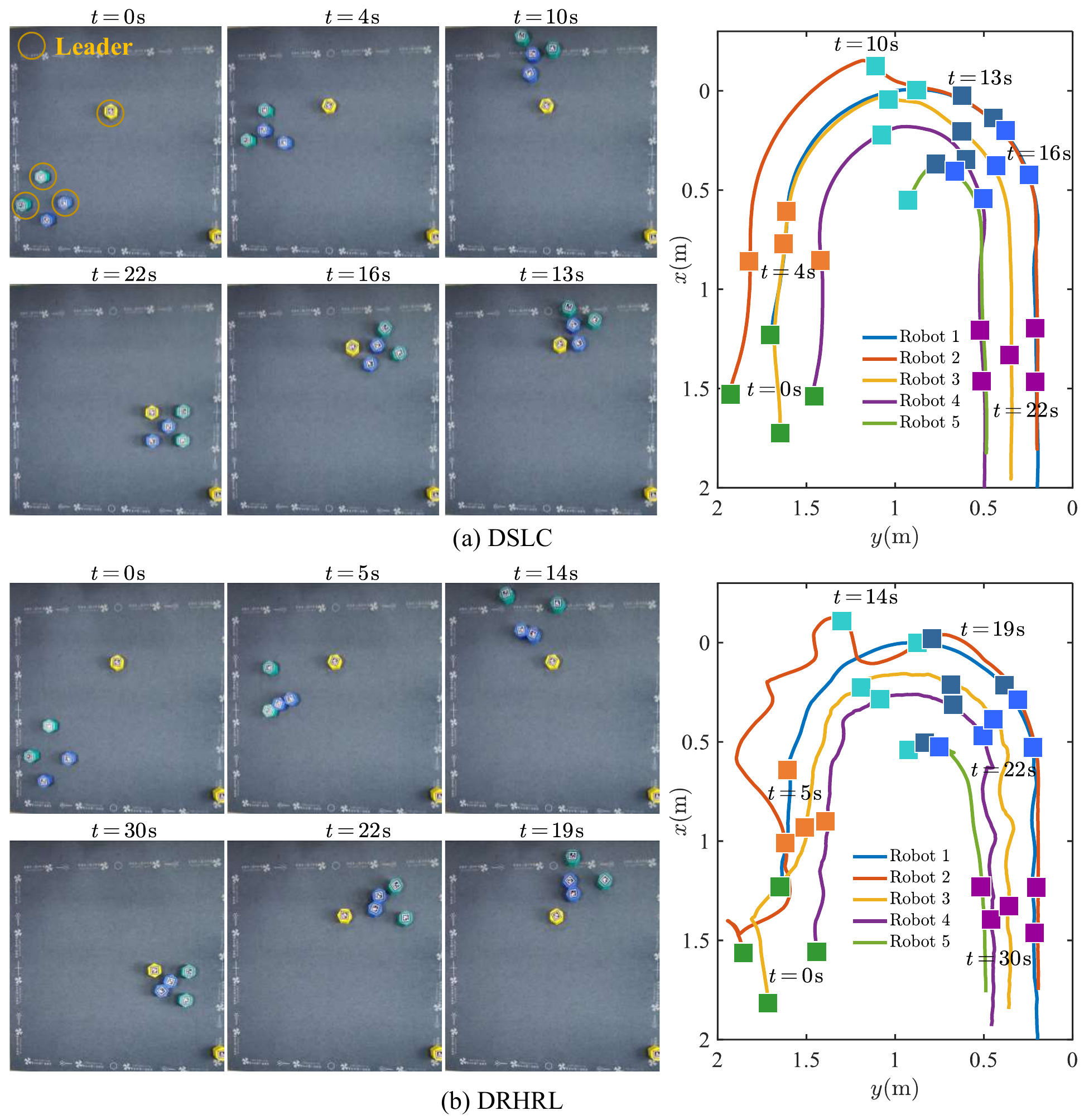}
	\caption{
    Trajectories and  snapshots for containment with an attack probability of $\beta=0.5$. (a) DSLC. (b) DRHRL. The robots under DSLC formed a stable containment shape and successfully reconfigured to accommodate a newly joined robot. In contrast, the follower robot under DRHRL failed to effectively enter the convex hull and maintain safe distances, resulting in collisions.
    }
	\label{ContainmentB10_P_C}
\end{figure}

We further validate the effectiveness of the proposed approach through experiments on containment control with $\beta=0.5$, $\beta=1$.
The MRS in the containment experiments includes five real robots and one virtual robot, and the reference trajectory is generated by the virtual robot, which is described as
\begin{equation}
\begin{aligned}
x_r(t)=&\left\{\begin{array}{ll}
L_{x,2} -v_rt&t\in[0,T_2)\\
C_{x,2} +\frac{v_r}{w_r}\cos{(\omega_r t+0.5\pi)}&t\in[T_2,T_3)\\
L_{x,3} -v_r(t-T_2)&t\in[T_3,\infty),\\
\end{array}\right.\\
y_r(t)=&\left\{\begin{array}{ll}
L_{y,2}&t\in[0,T_2)\\
C_{y,2} +\frac{v_r}{w_r}\sin{(\omega_r t+0.5\pi)}&t\in[T_2,T_3)\\
L_{y,3} &t\in[T_3,\infty),\\
\end{array}\right.\\
\end{aligned}
\end{equation}
where $L_{x,2}=1.8$, $L_{y,2}=1.6$,
$L_{x,3}=1.7$, $L_{y,3}=0.2$,
$C_{x,2}=0.7$, $C_{y,2}=0.9$,
$T_1=5.5$,  $T_2=16.5$, 
$v_r=0.2$, $\omega_r=0.29$ for $t\in [T_2,T_3)$ and $\omega_r=0$ for $t\in[0,T_2)\cup[T_3,\infty)$.
During $t\in [0,10)$, there existed leader robots and one follower robot in the containment, and after $t=10$s, another leader robot joined the MRS.
In the experiments, the relative desired position offset of the leader robots from the reference trajectory were given as 
\begin{equation}
\mathcal{F}_2=\left\{\begin{array}{cc}
 \left[\begin{matrix}
 0 & -0.195\sqrt{3}  & 0.195\sqrt{3} & \\
0 & -0.195 & 0.195 \\
\end{matrix}\right]^{\top}   &  t\in[0,10) \vspace{1mm}\\
\left[\begin{matrix}
0 & -0.3  & -0.15  & -0.3\\
0 & 0 & 0.15  & 0.3\\
\end{matrix}\right]^{\top}    & t\in[10,\infty).     \\
\end{array}\right.
\end{equation}
The control parameters in the experiments were set as in~\eqref{control_parameter_1}.  

The experimental results are shown in Figs.~\ref{ContainmentB05_P_C}–\ref{ContainmentB10_P_C}.
 The trajectories and snapshots of robots controlled by DSLC under attacks with $\beta=0.5$ is presented in Fig.~\ref{ContainmentB05_P_C}(a).
From $t=0$ s to $t=4$ s, the follower robot successfully entered the convex hull formed by the leader robots while maintaining safe distances.
During $t\in[4,10)$ s, the MRS remained stable in containment formation.
After $t=10$ s, an additional leader robot joined the MRS, and during $t\in[10,16)$ s, the follower robot reentered the updated convex hull formed by the leader robots, maintaining stability and safe interrobot distances thereafter.
The trajectories and snapshots of the multirobots controlled by DRHRL are shown in Fig.~\ref{ContainmentB05_P_C}(b).
It can be seen that the follower robot failed to effectively enter the convex hull and maintain safe distances, leading to collisions under $\beta=0.5$.

The experimental results under $\beta=1$ for both DSLC and DRHRL are provided in Fig.~\ref{ContainmentB10_P_C}.
The results show that DSLC efficiently guides the robots to follow the reference trajectory under $\beta=1$, ensuring that the follower robot enters the convex hull while maintaining safe distances to prevent collisions.
In contrast, under DRHRL, formation disorganization and collisions among the robots occurred, and the follower robot was unable to remain inside the convex hull under $\beta=1$.

The corresponding IAE results are summarized in Table~\ref{IAE_AffineFormation_Containment}, showing that DSLC achieves lower IAE values than DRHRL\cite{zhang2025toward} under attacks, thereby demonstrating its superior control performance.
\begin{table}[htbp]	
\centering
\caption{Performance Assessment of Affine Formation and Containment under Different Attacks in the Experiments }
\begin{tabular}{cccc}
\toprule
 Tasks&Approaches& $\beta=0.5$&$\beta=1$\\ 
 \midrule
 \multirow{2}{*}{Affine Formation}&DSLC&\textbf{4.41}&\textbf{4.06}\\
 &DRHRL\cite{zhang2025toward}&6.54&6.31\\
\midrule
\multirow{2}{*}{Containment}&DSLC&\textbf{21.44}&\textbf{15.84}\\
&DRHRL\cite{zhang2025toward}&31.66&40\\
\bottomrule
 \end{tabular}
 \label{IAE_AffineFormation_Containment}
\end{table}

\subsection{Discussion and Limitations}  
The extensive simulation and experimental results indicate that the proposed approach constitutes a generalizable framework for optimal cooperative control under malicious, stealthy actuator attacks across a broad range of coordination scenarios, including general formation, affine formation, and containment control. Moreover, our approach exhibits scalability for large-scale MRS under malicious, stealthy actuator attacks, owing to the optimization decomposition and heuristic policy iteration mechanism (see Table~\ref{tab:computation_time}  for the comparison of the computational time between DSLC and DRMPC). This enables real-time, distributed performance optimization to be implemented on small-scale onboard computing processors. Beyond its theoretical development and practical applicability, several limitations of the proposed approach should be acknowledged.

{\color{black}First, the current framework assumes a static communication network without communication delays. Extending the proposed method to support plug-and-play operations under switching communication topologies, as well as to ensure closed-loop robustness in the presence of communication delays, is left for future investigation.

Second, the attack probability
$\beta$ is assumed known in policy learning, which may be difficult to accurately estimate. Although Fig.~\ref{deploy-different} demonstrates that a defense policy designed for a specific attack probability can still be effective under mismatched attack probabilities, developing secure control strategies that are robust to unknown $\beta \in [0,1]$ remains a challenge problem.
 
Finally, our framework is formulated for general MRS governed by kinematic models and is, in principle, applicable to a wide range of robotic platforms, including ground and aerial robots. However, the experimental validation in this work is limited to wheeled ground robots. Extending the framework to aerial platforms will be explored in future studies.}

\section{Conclusions}
This paper proposes a distributed secure learning control framework for multirobots under malicious, stealthy actuator attacks. In our framework, adversaries that deteriorate control performance and defenders that secure the closed-loop system are modeled as adversarial players in a distributed zero-sum  differential game. Both defense and attack policies are optimized by a game-theoretic distributed learning-based predictive control approach. This approach integrates the receding horizon mechanism into the policy learning process, thus ensuring closed-loop stability and performance under suitable conditions. Notably, we designed a novel distributed attacker-actor-critic algorithm that can be performed online to efficiently implement the proposed DSLC approach. The resulting defense policies are of explicit closed-loop form, which could also be directly deployed (without retraining) to multirobots with varying scales and distinct attack distribution probabilities. This feature supports the scalability of our approach to secure control under attacks of large-scale MRS, which is, in general, nontrivial in a numerical optimization-based control framework. The simulation and experimental results, including comparisons with advanced approaches, have demonstrated the effectiveness and scalability of our approach. Future works will consider the extension to distributed secure control with communication delays unknown attack probabilities and even unknown dynamics with higher dimensions. 

\bibliographystyle{Bibliography/IEEEtranTIE}
\bibliography{Bibliography/,ref}

				\appendices
	\section{}\label{appendix}
\subsection{Functions in Section~\ref{problem_formulation}-A}\label{model-append}
The functions in \eqref{Eqn:LL} for different cooperative control tasks are described as follows.\\
(a) General formation
\begin{equation}
\begin{aligned}
&f_{i}(x_{\scriptscriptstyle \mathcal{N}_i}) =\\
&\left[
\begin{matrix}
c_{i} x_{r,4}\cos x_{ri,3} + h_i (x_{j,4} \cos x_{ji,3}) - (c_0 +h_i)x_{i,4}\\
c_i x_{r,4}\sin x_{ri,3} +  h_i (x_{j,4}\sin x_{ji,3})\\
\bm 0_{2\times 1}
\end{matrix} \right],
\end{aligned}
\end{equation}
\begin{equation}
g_i(x_i)=\left[\begin{matrix}
         \;\;\; 0, \;\;\;\;0, \;\;\;\;\;\;0, \;\;\;1\\
          x_{i,2},-x_{i,1},\;\;1,\;\;0\\
     \end{matrix}\right]^{\top};
\end{equation}
(b)  Affine formation
\begin{equation}
\begin{aligned}
&f_{i}(x_{\scriptscriptstyle \mathcal{N}_i}) =\\
&\left[
\begin{matrix}
\zeta_{ri} x_{r,4}\cos x_{ri,3} + h_i (x_{j,4} \cos x_{ji,3}) - (\zeta_{ri} +h_i)x_{i,4}\\
\zeta_{ri} x_{r,4}\sin x_{ri,3} +  h_i (x_{j,4}\sin x_{ji,3})\\
\bm 0_{2\times 1}
\end{matrix} \right],
\end{aligned}
\end{equation}
\begin{equation}
g_i(x_i)=\left[\begin{matrix}
         \;\;\; 0, \;\;\;\;0, \;\;\;\;\;\;0, \;\;\;1\\
          x_{i,1},-x_{i,2},\;\;1,\;\;0\\
     \end{matrix}\right]^{\top};
\end{equation}
(c) Containment
\begin{equation}
\begin{aligned}
&f_{i}(x_{\scriptscriptstyle \mathcal{N}_i}) =\\
&\left[
\begin{matrix}
h_i^l (x_{r_j,4}\cos x_{r_j i,3}) + h_i (x_{j,4} \cos x_{ji,3}) - (h_i^l  +h_i)x_{i,4}\\
h_i^l(x_{r_j,4}\sin x_{r_ji,3}) +  h_i (x_{j,4}\sin x_{ji,3})\\
\bm 0_{2\times 1}
\end{matrix} \right],
\end{aligned}
\end{equation}
\begin{equation}
g_i(x_i)=\left[\begin{matrix}
         \;\;\; 0, \;\;\;\;0, \;\;\;\;\;\;0, \;\;\;1\\
          x_{i,1},-x_{i,2},\;\;1,\;\;0\\
     \end{matrix}\right]^{\top},
\end{equation}
where $h_i(\cdot) = \sum_{j\in\mathcal{N}_i} (\cdot)$, $h_i^l(\cdot) = \sum_{j\in\mathcal{N}_i^l} (\cdot)$, $x_{ji,\ast}= x_{j,\ast}-x_{i,\ast}$, $x_{ri,\ast}= x_{r,\ast}-x_{i,\ast}$, $x_{r_j i,\ast}= x_{r_j,\ast}-x_{i,\ast}$, and $x_{i,\ast}$,  $x_{j,\ast}$, $x_{r,\ast}$, $x_{r_j,\ast}$ denote the $\ast$-th element of  $x_{i}$,  $x_{j}$, $x_{r}$, $x_{r_j}$, respectively.

\subsection{Closed-loop Stability Analysis}\label{appB}
Assume that at time $t_0=0$, the prediction horizon $t_p$ is selected so that the terminal state $x_i(t_0+t_p)$ is in the neighbor of the origin under the proposed policy learning approach.  Then, the following stability results are stated.
\begin{theorem}[Closed-loop stability]\label{stability_analysis}
    If there exists $P_i$ such that~\eqref{Lyapunov-equation} is met and the optimal problem \eqref{Eqn:optimiz} with \eqref{Eqn:HL_cost} is feasible at the initial time $t_0=0$, the closed-loop system~\eqref{Eqn:C full model} under optimal policies \eqref{Eqn:optimal} is asymptotically stable.
\end{theorem}
\begin{IEEEproof}
The terminal cost function for the $i$-th robot is considered as $V_i = \mathbb{E}\{ x_i^{\top} P_i x_i\} $.  Then, taking the derivative of $V_i$ along system \eqref{Eqn:LL} with the admissible control policies  $ u_{s,i} = K_{\scriptscriptstyle \mathcal{N}_{s,i}} x_{\scriptscriptstyle\mathcal{N}_i}$, $ u_{a,i} = K_{\scriptscriptstyle \mathcal{N}_{a,i}} x_{\scriptscriptstyle\mathcal{N}_i}$, and combining with~\eqref{Lyapunov-equation}, we can obtain
\begin{equation}\label{V-1}
\begin{aligned}
    \dot V_i = & \mathbb{E}\{\dot x_i^{\top} P_i x_i + x_i^{\top} P_i \dot x_i\}\\
             = & (A_{\scriptscriptstyle\mathcal{N}_i} x_{\scriptscriptstyle\mathcal{N}_i} + B_i K_{\scriptscriptstyle \mathcal{N}_{s,i}} x_{\scriptscriptstyle\mathcal{N}_i} + \beta_{i} B_i K_{\scriptscriptstyle \mathcal{N}_{a,i}} x_{\scriptscriptstyle\mathcal{N}_i} + \phi_i)^{\top} P_i x_i \\
             &+ x_i^{\top} P_i (A_{\scriptscriptstyle\mathcal{N}_i} x_{\scriptscriptstyle\mathcal{N}_i} + B_i K_{\scriptscriptstyle \mathcal{N}_{s,i}} x_{\scriptscriptstyle\mathcal{N}_i} + \beta_{i} B_i K_{\scriptscriptstyle \mathcal{N}_{a,i}} x_{\scriptscriptstyle\mathcal{N}_i} + \phi_i)\\
            \le & -x_{\scriptscriptstyle\mathcal{N}_i}^{\top}(\tau_{\scriptscriptstyle P} I - 2 T_{\scriptscriptstyle\mathcal{N}_i}^{\top} P_i L_{\phi_i} )x_{\scriptscriptstyle\mathcal{N}_i} - \|x_{\scriptscriptstyle\mathcal{N}_i}\|_{Q_i^*}^2+\|x_{\scriptscriptstyle\mathcal{N}_i}\|_{\Gamma_{\scriptscriptstyle \mathcal{N}_i}}^2 \\
            \le & - \chi_i (x_{\scriptscriptstyle \mathcal{N}_i},u_{s,i}, u_{a,i})+\|x_{\scriptscriptstyle\mathcal{N}_i}\|_{\Gamma_{\scriptscriptstyle \mathcal{N}_i}}^2:=\bar{\chi}_i (x_{\scriptscriptstyle \mathcal{N}_i},u_{s,i}, u_{a,i}),
\end{aligned}
\end{equation}
where $\chi_i(x_{\scriptscriptstyle \mathcal{N}_i},u_{s,i}, u_{a,i}) = \| x_{\scriptscriptstyle \mathcal{N}_i}\|_{Q_i}^{2}+\| u_{s,i}\|_{R_{s,i}}^{2}-\beta_{i} \|u_{a,i}\|^2_{R_{a,i}}$. Since $Q_i^*>0$, $\chi_i(x_{\scriptscriptstyle \mathcal{N}_i},u_{s,i}, u_{a,i}) >0 $ can be guaranteed.
For $\sigma_1 \le \sigma_2$, integrating both side of~\eqref{V-1} yields
\begin{equation}\label{V-3}
    \begin{aligned}
    V_i(\sigma_2) \le V_i(\sigma_1) - \int_{\sigma_1 }^{\sigma_2} \bar{\chi}_i(x_{\scriptscriptstyle \mathcal{N}_i},u_{s,i}, u_{a,i}) d\tau.
    \end{aligned}
\end{equation}
Since the optimal problem is feasible, there exists an optimal control policy $u_{s,i}^*$.
Based on the optimal control policy $u_{s,i}^*$, we select
\begin{equation}\label{V-5}
    \begin{aligned}
        & \bar u_{s,i}(t) = \left\{\begin{array}{ll}
            u_{s,i}^*(t)&  t \in [ t_0,\sigma_1)\\
             K_{\scriptscriptstyle \mathcal{N}_{s,i}} x_{\scriptscriptstyle \mathcal{N}_i}(t)&  t  \in [\sigma_1,\sigma_2),\\
        \end{array}\right.\\
    \end{aligned}
\end{equation}
with $u_{a,i}(t) = K_{\scriptscriptstyle \mathcal{N}_{a,i}} x_{\scriptscriptstyle \mathcal{N}_i}(t)$ for $t \in [\sigma_1,\sigma_2)$.
Then, define that $S_i\left(x_{\scriptscriptstyle \mathcal{N}_i}, \bar u_{s,i},u_{a,i}\right) = \underset{u_{a,i}}{\max}\,\mathbb{E}\{J_i(x_{\scriptscriptstyle \mathcal{N}_i},u_{s,i}^*,u_{a,i})\}$. 
The expectation of $J_i(x_{\scriptscriptstyle \mathcal{N}_i}(t))$ can be written as
\begin{equation}\label{V-4}
    \begin{aligned}
        V_i(\sigma_2) + \int_{t_0}^{\sigma_2}\chi_i(\tau) d\tau = &V_i(\sigma_2) - V_i(\sigma_1) + \int_{\sigma_1}^{\sigma_2}\chi_i(\tau) d\tau \\
        &+ V_i(\sigma_1) +     \int_{t_0}^{\sigma_1}\chi_i(\tau) d\tau.
    \end{aligned}
\end{equation}
According to~\eqref{V-3},~\eqref{V-4} becomes
\begin{equation}\label{V-6}
    \begin{aligned}
    V_i(\sigma_2) + \int_{t_0}^{\sigma_2}\chi_i(\tau) d\tau \le  V_i(\sigma_1) +&     \int_{t_0}^{\sigma_1}\chi_i(\tau) d\tau+\\
    &\int_{\sigma_1}^{\sigma_2}\|x_{\scriptscriptstyle \mathcal{N}_i}(\tau)\|_{\Gamma_{\scriptscriptstyle \mathcal{N}_i}}^2.
    \end{aligned}
\end{equation}
Thus, we can obtain
\begin{equation*}
    \begin{aligned}
    &\underset{u_{a,i}}{\max}\left\{V_i(\sigma_2) + \int_{t_0}^{\sigma_2}\chi_i(\tau) d\tau\right\} \\
    & \quad\quad\le  \underset{u_{a,i}}{\max}\left\{V_i(\sigma_1) +     \int_{t_0}^{\sigma_1}\chi_i(\tau) d\tau \right\}+\int_{\sigma_1}^{\sigma_2}\|x_{\scriptscriptstyle \mathcal{N}_i}(\tau)\|_{\Gamma_{\scriptscriptstyle \mathcal{N}_i}}^2.
    \end{aligned}
\end{equation*}
The maximum value of~\eqref{V-6} is $ S_i(x_{\scriptscriptstyle \mathcal{N}_i}(t_0),\sigma) $, and then, it leads to
\begin{equation}\label{V-7}
    \begin{aligned}
    S_i(x_{\scriptscriptstyle \mathcal{N}_i}(t_0),\sigma_1) +\int_{\sigma_1}^{\sigma_2}\|x_{\scriptscriptstyle \mathcal{N}_i}(\tau)\|_{\Gamma_{\scriptscriptstyle \mathcal{N}_i}}^2 d\tau\ge S_i(x_{\scriptscriptstyle \mathcal{N}_i}(t_0),\sigma_2). 
    \end{aligned}
\end{equation}
According to the definition of $S_i$, for any admissible attack input $u_{a,i}$ in $t \in [t_0,t_0+ t_p]$, we have 
\begin{equation}\label{V-8}
\begin{aligned}
    &S_i(x_{\scriptscriptstyle \mathcal{N}_i}(t_0),t_0+t_p) \ge V_i(t_0 + t_p) + \int_{t_0}^{t_0+t_p}\chi_i(\tau) d\tau\\
    &= V_i(t_0 + t_p)+\int_{t_0}^{t}\chi_i(\tau) d\tau + \int_{t}^{t_0+t_p}\chi_i(\tau) d\tau.
\end{aligned}
\end{equation}
Similarly, we obtain
\begin{equation}\label{V-9}
\begin{aligned}
    &S_i(x_{\scriptscriptstyle \mathcal{N}_i}(t_0),t_0+t_p) \ge \int_{t_0}^{t}\chi_i(\tau) d\tau \\
    &\quad\quad\quad\quad\quad+ \underset{u_{a,i}}{\max}\left\{V_i(t_0 + t_p)+ \int_{t}^{t_0+t_p}\chi_i(\tau) d\tau\right\}.
\end{aligned}
\end{equation}
For any attack input in $\tau \in [t,t_0+t_p]$, $S_i(x_{\scriptscriptstyle \mathcal{N}_i}(t),t_0+t_p) \le \underset{u_{a,i}}{\max}\left\{V_i(t_0 + t_p)+ \int_{t}^{t_0+t_p}\chi_i(\tau) d\tau\right\}$ holds, and then~\eqref{V-9} can be rewritten as
\begin{equation}\label{V-10}
\begin{aligned}
    &S_i(x_{\scriptscriptstyle \mathcal{N}_i}(t_0),t_0+t_p) \ge \int_{t_0}^{t}\chi_i(\tau) d\tau + S_i(x_{\scriptscriptstyle \mathcal{N}_i}(t),t_0+t_p).\\ 
\end{aligned}
\end{equation}
Due to $t_0 +t_p \le t+t_p$, combining with Eqs.~\eqref{V-7},~\eqref{V-10}, yields
\begin{equation*}
\begin{aligned}
    S_i(x_{\scriptscriptstyle \mathcal{N}_i}(t_0),t_0+t_p) \ge& \int_{t_0}^{t}\chi_i(\tau) d\tau + S_i(x_{\scriptscriptstyle \mathcal{N}_i}(t),t+t_p)  \\
    &-\int_{t_0+t_p}^{t+t_p}\|x_{\scriptscriptstyle \mathcal{N}_i}(\tau)\|_{\Gamma_{\scriptscriptstyle \mathcal{N}_i}}^2d\tau.
\end{aligned}
\end{equation*}
By utilizing the above results iteratively, for any $t \le t+\Delta t$, we have 
\begin{equation}\label{V-11}
\begin{aligned}
    S_i(x_{\scriptscriptstyle \mathcal{N}_i}(t),t+\Delta t+t_p) \ge -\int_{t+t_p}^{t+\Delta t+t_p}\|x_{\scriptscriptstyle \mathcal{N}_i}(\tau)\|_{\Gamma_{\scriptscriptstyle \mathcal{N}_i}}^2d\tau \\ 
       \int_{t}^{t+\Delta t}\chi_i(\tau) d\tau + S_i(x_{\scriptscriptstyle \mathcal{N}_i}(t+\Delta t),t+\Delta t+t_p).
\end{aligned}
\end{equation}
Considering $J_i(t) = S_i(x_{\scriptscriptstyle \mathcal{N}_i}(t), t+t_p) $,
$J(t_0+\Delta t) - J(t_0)  = \sum_{i=1}^M \left(J_i(t_0+\Delta t) -J_i(t_0)\right)$. Setting $\Delta t=t_c$, according to \eqref{V-11}, we can obtain $J(t_0+t_c) - J(t_0)= -\sum_{i=1}^M \int_{\tau=t_0}^{t_0+t_c} \chi_i(\tau) d\tau+\int_{t_0+t_p}^{t_0+ t_c+t_p}\|x_{\scriptscriptstyle \mathcal{N}_i}(\tau)\|_{\Gamma_{\scriptscriptstyle \mathcal{N}_i}}^2d\tau$. Note that, $\sum_{i=1}^M \Gamma_{\scriptscriptstyle \mathcal{N}_i}\leq 0$, one has $\sum_{i=1}^M\int_{t_0+t_p}^{t_0+ t_c+t_p}\|x_{\scriptscriptstyle \mathcal{N}_i}(\tau)\|_{\Gamma_{\scriptscriptstyle \mathcal{N}_i}}^2d\tau\leq 0$. In addition, with a suitable weight matrix $R_{a,i}$, $\chi_i(\tau) \ge 0 $ can be guaranteed. Hence, we can obtain $  \Delta J(t) \rightarrow 0$ as $t\rightarrow+\infty$, and the closed-loop system is asymptotically stable.
\end{IEEEproof}

\subsection{Learning Convergence and Practical Stability Verification}\label{app:closed-loop}
\emph{Learning convergence}: We first give the convergence analysis of the distributed attacker-actor-critic algorithm in each prediction interval.
%
To this end, according to the Weierstrass higher order approximation theorem \cite{vamvoudakis2010online}, one writes
the optimal cost $J_i^*(x_{\scriptscriptstyle \mathcal{N}_i}(t))$ at each time  $t\in [t_0,\,t_0+t_p]$
as
\begin{equation}\label{target-critic}
    J_i^*(x_{\scriptscriptstyle{\mathcal{N}}_i}(t)) = W_{c,i}^{\top} \varphi_{c,i}(x_{\scriptscriptstyle{\mathcal{N}}_i}(t)) + \varepsilon_{c,i} (x_{\scriptscriptstyle{\mathcal{N}}_i}(t)),
\end{equation}
where $W_{c,i}$ is the optimal weight matrix, $\varepsilon_{c,i}(x_{\scriptscriptstyle{\mathcal{N}}_i}(t))$ denote the residual error of the neural network for the $i$-th robot, respectively.
Taking the derivative of~\eqref{target-critic} on $x_{\scriptscriptstyle {\mathcal{N}}_i}(t)$ leads to
\begin{equation*}
\triangledown J_i^*= \triangledown\varphi_{c,i}^{\top} W_{c,i} + \triangledown \varepsilon_{c,i},\\
\end{equation*}
where $\triangledown J_i^*=\frac{\partial J_i^*}{\partial x_{\scriptscriptstyle {\mathcal{N}}_i(t)}}$, 
$\triangledown \varphi_{c,i} = \frac{\partial \varphi_{c,i} }{\partial x_{\scriptscriptstyle {\mathcal{N}}_i(t)}}$, 
and $\triangledown \varepsilon_{c,i} = \frac{\partial \varepsilon_{c,i} }{\partial x_{\scriptscriptstyle{\mathcal{N}}_i(t)}}$.

Likewise, the optimal control inputs for the $i$-th robot can be represented as follows:
\begin{equation*}
\begin{aligned}
u_{a,i}^*(x_{\scriptscriptstyle{\mathcal{N}}_i}(t))&=W_{a,i}^{\top} \varphi_{a,i}(x_{\scriptscriptstyle{\mathcal{N}}_i}(t)) + \varepsilon_{a,i} (x_{\scriptscriptstyle{\mathcal{N}}_i}(t))\\
u_{s,i}^*(x_{\scriptscriptstyle{\mathcal{N}}_i}(t))&=W_{s,i}^{\top} \varphi_{s,i}(x_{\scriptscriptstyle{\mathcal{N}}_i}(t)) + \varepsilon_{s,i} (x_{\scriptscriptstyle{\mathcal{N}}_i}(t)),\\
\end{aligned}
\end{equation*}
where $W_{a,i}$ and $W_{s,i}$ are the optimal weight matrices; $\varepsilon_{a,i}$ and $\varepsilon_{s,i} $ are the residual errors of the neural networks.

The following basic assumption holds.
\begin{assumption}[Neural networks~\cite{vamvoudakis2010online}]\label{assmption-NN}
The optimal weight matrices,  approximation errors,  activation functions, and associated gradients of the neural networks, are bounded on a compact set $ \Omega $, i.e., 
\begin{equation*}
    \begin{aligned}
    & \|W_{*,i}\| \le \bar W_{*},\;\|\varepsilon_{*,i}\| \le \bar \varepsilon_*,\; \|\varphi_{*,i}\| \le \bar \varphi_*,\;(*=c,a,s),\\
    & \|\triangledown\varphi_{c,i}\| \le \triangledown \bar \varphi_c,\; \|\triangledown\varepsilon_{c,i}\| \le \triangledown \bar \varepsilon_c,
    \end{aligned}
\end{equation*}
where $\bar W_{*},\; \bar \varepsilon_*,\;\bar \varphi_*\;(*=c,a,s)$, $\triangledown \bar \varphi_c$, $ \triangledown \bar \varepsilon_c$  are positive constants.
\end{assumption}

Let $\Tilde{W}_{c,i} = W_{c,i} - \hat W_{c,i}$, 
    $\tilde{W}_{a,i} = W_{a,i} - \hat W_{a,i}$, $\tilde{W}_{s,i} = W_{s,i} - \hat W_{s,i}$. The following convergence result can be stated.
\begin{theorem}[Convergence analysis]\label{convergence_analysis}
If the persistent excitation conditions of $\varphi_{*,i}\; (*=c,a,s)$ are satisfied and Assumption~\ref{assmption-NN} is fulfilled, then the approximation error matrices $\tilde{W}_{c,i} $, $ \tilde{W}_{a,i}$, and $\tilde{W}_{s,i}$ are uniformly ultimately bounded.
\end{theorem}
\begin{IEEEproof}
The convergence analysis of the neural networks is performed by the Lyapunov stability theory, and the Lyapunov function candidate is considered as 
$L = 
 \sum_{i=1}^{M} L_i $,
 where
\begin{equation}
    \begin{aligned}
        L_i=  L_{c,i} + L_{a,i} + L_{s,i},
    \end{aligned}
\end{equation}
with 
\begin{equation}
    \begin{aligned}
    &L_{c,i} = \frac{1}{2}\mathbb{E}\left\{ \text{tr} \left\{\tilde W_{c,i}^{\top}  \tilde W_{c,i}\right\}\right\},\\
    &L_{a,i} = \frac{1}{2} \text{tr} \left\{\tilde W_{a,i}^{\top}  \tilde W_{a,i}\right\},\;L_{s,i} = \frac{1}{2} \text{tr} \left\{\tilde W_{s,i}^{\top}  \tilde W_{s,i}\right\}.\\
    \end{aligned}
\end{equation}
Taking the derivative of $L$, we can obtain $ \dot L =  \mathbb{E}\left\{\sum_{i=1}^{M} \dot L_i\right\} $, where $  \dot L_i= \dot L_{c,i} + \dot L_{a,i} + \dot L_{s,i} $.
Considering $\dot L_{c,i}$, we can obtain
\begin{equation*}
    \begin{aligned}
    \dot L_{c,i}=&\mathbb{E}\left\{ \tilde W_{c,i}^{\top} \dot{\tilde W}_{c,i}\right\}\\
    =&\eta_{c,i}  \tilde W_{c,i}^{\top} \frac{\Delta\varphi_{c,i}}{(1+\Delta\varphi_{c,i}^{\top}\Delta\varphi_{c,i})^2}\mathbb{E}\left\{ \zeta_{i}^{\top}\right\}.
    \end{aligned}
\end{equation*}
Utilizing the HJI equation and~\eqref{error-HJB-IRL-1} yields
\begin{equation}\label{e1i}
    \begin{aligned}
     \mathbb{E}\{\zeta_{i}\} =&  \int_{\tau=t}^{t+T}\hat r_i(\tau)d\tau + \hat W_{c,i}^{\top}\Delta \varphi_{c,i} \\
     = & -\beta_{i}\int_{\tau=t}^{t+T} \|\hat u_{a,i}\|^2_{R_{a,i}} - \| u_{a,i}\|^2_{R_{a,i}} d\tau + \epsilon_i\\
     & + \int_{\tau=t}^{t+T} \| \hat u_{s,i}\|_{R_{s,i}}^{2} - \| u_{s,i}\|_{R_{s,i}}^{2} d\tau  - \tilde W_{c,i}^{\top}\Delta  \varphi_{c,i}\\
      \le& - \tilde W_{c,i}^{\top}\Delta \varphi_{c,i}  + b_c,
    \end{aligned}
\end{equation}
where $b_c =  b_\epsilon + T \bar \lambda_{\scriptscriptstyle R_s} \varrho_s^2 + T  \bar \lambda_{\scriptscriptstyle R_a} \varrho_a^2 $, $ \epsilon_i=\varepsilon_{c,i}(t)- \varepsilon_{c,i}(t+T)$, with $\|\epsilon_i\| \le b_\epsilon $, $b_\epsilon= 2 \bar \epsilon_c$, $\bar \lambda_{\scriptscriptstyle R_*}\;(*=a,s)$ denotes the maximum eigenvalue of $R_{*,i}\;(*=a,s)$. 
Since the actor networks are specical designed to deal with the constraints on the actuators of the real rotot systems, $\|\hat u_{*,i}\| \le \varrho_{*},\; \| u_{*,i}\| \le \varrho_{*}$ hold, where $ \varrho_{*} = \max\{\varrho_{*,i}\}$, $*=a,s;\;i=1,2,\cdots,M$.

Define $ m_{\varphi_{c,i}} =  \frac{\Delta\varphi_{c,i}}{1+\Delta\varphi_{c,i}^{\top}\Delta\varphi_{c,i}}$, and then, with~\eqref{e1i}, the derivative of $ L_{c,i}$ becomes
\begin{equation}\label{dotLc}
    \begin{aligned}
    \dot L_{c,i} \le & - \eta_{c,i} \underline \lambda_{\scriptscriptstyle \Lambda_i}  \|\tilde W_{c,i}\|^2 + \eta_{c,i} b_c \triangledown \bar \varphi_c \|\tilde W_{c,i}\|,
    \end{aligned}
\end{equation}
where $\underline \lambda_{\scriptscriptstyle \Lambda_i}$ is the minimum eigenvalue of $ \Lambda_i $, with $\Lambda_i = m_{\varphi_{c,i}}  m_{\varphi_{c,i}}^{\top} $.
By~\eqref{update-law-actor}, the derivative of $L_{a,i}$ can written as

\begin{equation}\label{dotLa}
    \begin{aligned}
    &\dot L_{a,i} =  \tilde W_{a,i}^{\top}  \dot{\tilde W}_{a,i}\\
    &=-\eta_{a,i} \tilde W_{a,i}^{\top}   \varphi_{a,i}\left(\frac{1}{2} R_{a,i}^{-1}\sum_{j\in\bar{\mathcal{N}}_i} g_i^{\top}\triangledown\varphi_{c,j}^{\top} \hat W_{c,j} - \hat W_{a,i}^{\top} \varphi_{a,i}\right)^{\top}\\
    &\le  -\frac{7}{8}\eta_{a,i} \|\tilde W_{a,i}^{\top} \varphi_{a,i} \|^2  + \frac{1}{2} \eta_{a,i} \bar \lambda_{\scriptscriptstyle R_{a}^{-1}}^2 b_g^2 \triangledown \bar \varphi_c^2 n_{\scriptscriptstyle\bar{\mathcal{N}}}\sum_{j\in\bar{\mathcal{N}}_i} \tilde W_{c,j}^2 \\
    & \hspace{3mm}
    + (\eta_{a,i} \bar \varphi_a \bar W_a + \frac{1}{2} \eta_{a,i} \bar \lambda_{R_a^{-1}} n_{\scriptscriptstyle\bar{\mathcal{N}}} b_g \triangledown \bar \varphi_c \bar W_c)\|\tilde W_{a,i}^{\top} \varphi_{a,i} \|,\\
    \end{aligned}
\end{equation}
where $\|g_i\| \le b_g$, $n_{\bar{\scriptscriptstyle \mathcal{N}}}$ is the maximum number of elements in $\bar{\mathcal{N}}_i$, and $ \bar \lambda_{\scriptscriptstyle R_{a}^{-1}}$ represents the maximum eigenvalue of $ R_{a,i}^{-1} $. 
Similarly, $\dot L_{s,i}$ can be written as
\begin{equation}\label{dotLs}
    \begin{aligned}
    \dot L_{s,i} 
  \le & -\frac{7}{8}\eta_{s,i} \|\tilde W_{s,i}^{\top} \varphi_{s,i} \|^2  + \frac{1}{2} \eta_{s,i} \bar \lambda_{\scriptscriptstyle R_{s}^{-1}}^2 b_g^2 \triangledown \bar \varphi_c^2 n_{\scriptscriptstyle\bar{\mathcal{N}}}\sum_{j\in\bar{\mathcal{N}}_i} \tilde W_{c,j}^2 \\
    & + (\eta_{s,i} \bar \varphi_s \bar W_s + \frac{1}{2} \eta_{s,i} \bar \lambda_{R_s^{-1}} n_{\scriptscriptstyle\bar{\mathcal{N}}} b_g \triangledown \bar \varphi_c \bar W_c)\|\tilde W_{s,i}^{\top} \varphi_{s,i} \|,\\ 
    \end{aligned}
\end{equation}
where $\bar \lambda_{\scriptscriptstyle R_s^{-1}}$ denotes the maximum eigenvalue of $R_{s,i}^{-1}$.
Together with Eqs. \eqref{dotLc}, \eqref{dotLa}, \eqref{dotLs}, the derivative of $L$ can be written as
\begin{equation}\label{dotL}
    \begin{aligned}
    \dot L \le &- \sum_{i=1}^M  \frac{7}{8}\left( \eta_{a,i} \|\tilde W_{a,i}^{\top} \varphi_{a,i} \|^2  + \eta_{s,i} \|\tilde W_{s,i}^{\top} \varphi_{s,i} \|^2 \right)\\
    & - \sum_{i=1}^M  \left( \check \eta_{c,i} \|\tilde W_{c,i}\|^2 + \eta_{c,i} b_c \triangledown \bar \varphi_c \|\tilde W_{c,i}\|\right) \\
    &-  \sum_{i=1}^M \left( \wp_{a,i} \|\tilde W_{a,i}^{\top} \varphi_{a,i} \| + \wp_{s,i} \|\tilde W_{s,i}^{\top} \varphi_{s,i} \|  \right), \\
    \end{aligned}
\end{equation}
where $\wp_{a,i} = \eta_{a,i} \bar \varphi_a \bar W_a + \frac{1}{2} \eta_{a,i} \bar \lambda_{R_a^{-1}} n_{\scriptscriptstyle\bar{\mathcal{N}}} b_g \triangledown \bar \varphi_c \bar W_c$, $ \wp_{s,i}=\eta_{s,i} \bar \varphi_s \bar W_s + \frac{1}{2} \eta_{s,i} \bar \lambda_{R_s^{-1}} n_{\scriptscriptstyle\bar{\mathcal{N}}} b_g \triangledown \bar \varphi_c \bar W_c$, $\varkappa_{a,i} = \frac{1}{2} \eta_{a} \bar \lambda_{\scriptscriptstyle R_{a}^{-1}}^2 b_g^2 \triangledown \bar \varphi_c^2 n_{\scriptscriptstyle\bar{\mathcal{N}}} \varsigma_i$, $\varkappa_{s,i} = \frac{1}{2} \eta_{s,i} \bar \lambda_{\scriptscriptstyle R_{s}^{-1}}^2 b_g^2 \triangledown \bar \varphi_c^2 n_{\scriptscriptstyle\bar{\mathcal{N}}} \varsigma_i$, $\eta_a = \max\{\eta_{a,i}\}$, $\eta_s = \max\{\eta_{s,i}\}$,  and $\varsigma_i$ denotes the number of the robots which can receive the $i$-th robot's information. To achieve uniformly ultimately boundedness property in policy learning, one requires that 
\begin{equation}\label{Eqn:convergence}\check \eta_{c,i} = \eta_{c,i} \underline \lambda_{\scriptscriptstyle \Lambda_i} - \varkappa_{a,i} - \varkappa_{s,i}>0
\end{equation} is satisfied, which can be guaranteed with proper parameter tuning.
According to~\eqref{dotL}, if $ \|\tilde W_{c,i}\| \ge \frac{\eta_{c,i} b_c \triangledown \bar \varphi_c}{\check\eta_{c,i}}$, $ \|\tilde W_{a,i}^{\top} \varphi_{a,i} \| \ge \frac{\wp_{a,i}}{\eta_{a,i}}$, $\|\tilde W_{s,i}^{\top} \varphi_{s,i} \| \ge \frac{\wp_{s,i}}{\eta_{s,i}} $ hold, we can obtain $\dot L \le 0$, $\|\tilde W_{c,i}\|$, $ \|\tilde W_{a,i}^{\top} \varphi_{a,i}\| $, $\|\tilde W_{s,i}^{\top} \varphi_{s,i}\|$ are uniformly ultimately bounded.  Since $ \varphi_{*,i} \; (*=c,a,s)$ is assumed to be persistently excited, $\|\tilde W_{a,i}\| $, $\|\tilde W_{s,i}\|$ are uniformly ultimately bounded.
\end{IEEEproof}
    \begin{figure}[h]
					\centering
     \includegraphics[width=1\columnwidth]{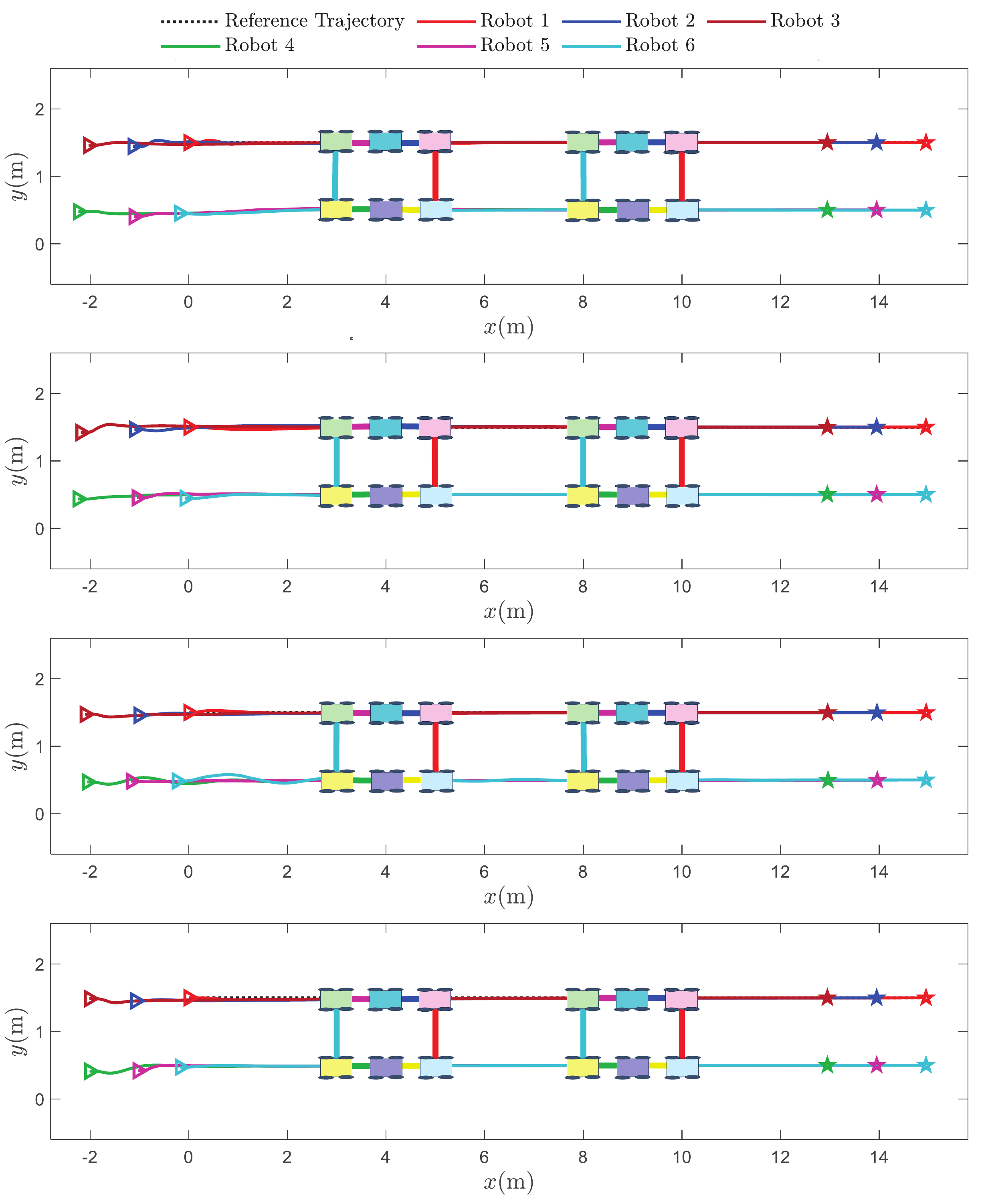}
				\caption{{\color{black}The trajectories of six-robot formation with  different random attack probabilities across robots: (a), $\vec\beta=(0.4,0.8,0.6,0.6,0.4,0.4)$; (b), $\vec \beta=(0.3,0.2,0.3,0.5,0.2,0.6)$; (c), $\vec \beta=(0.3,0.5,0.4,1,0.9,0.3)$; (d), $\vec \beta=(0.9,0.3,0.5,0.1,0.1,0.9)$, where $\vec \beta=(\beta_1,\,\cdots,\,\beta_6).$}}
		\label{six-robots-different}
\end{figure}
 \begin{figure*}[http]
					\centering
     \includegraphics[width=1.8\columnwidth]{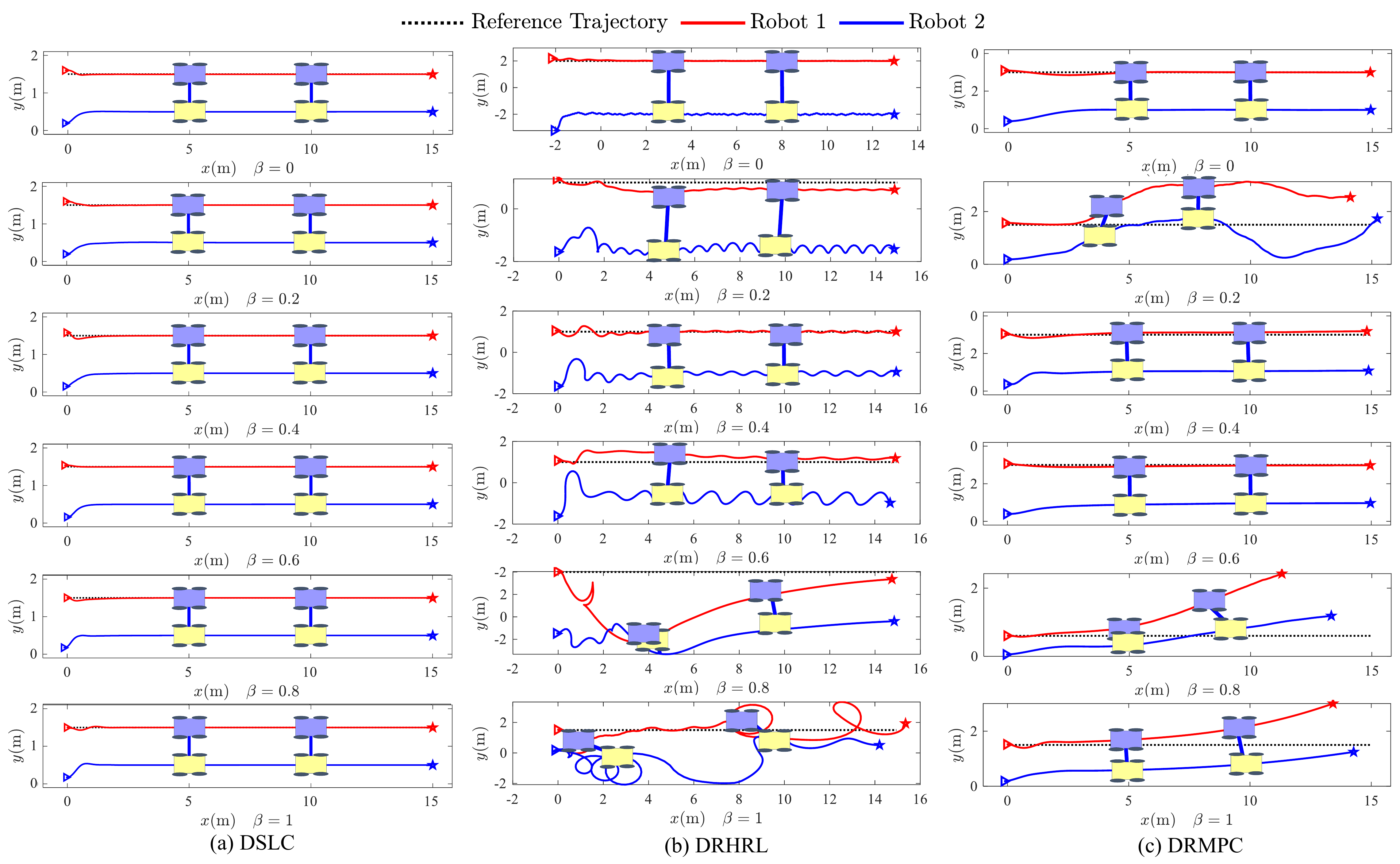}
				\caption{{The trajectories of two-robot formation with  different attack probabilities under DRHRL~\cite{zhang2025toward}, DRMPC~\cite{zhou2022event}, and DSLC.}}
				\label{two-robots}
\end{figure*}
\begin{figure*}[http]
					\centering
     \includegraphics[width=1.8\columnwidth]{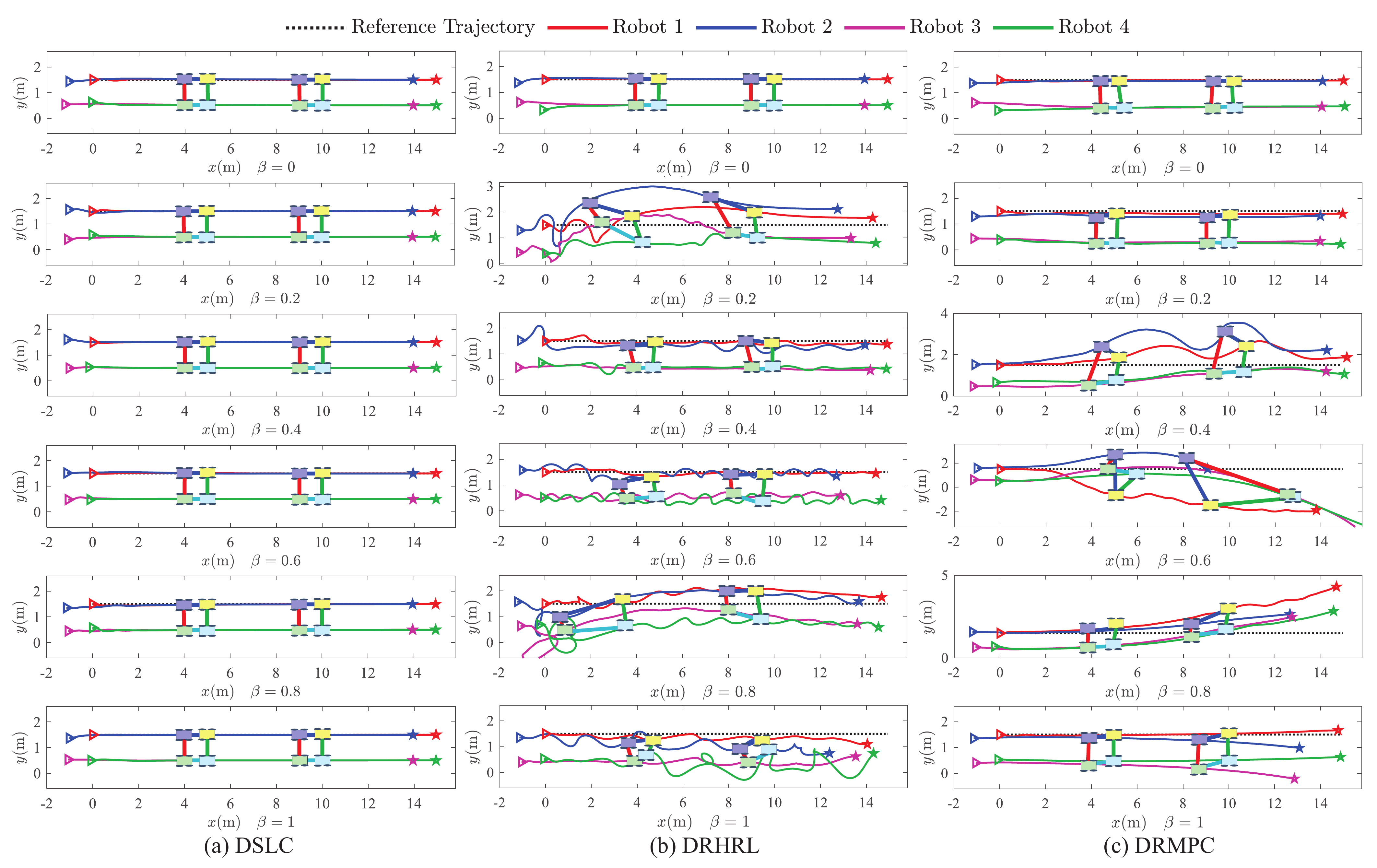}
				\caption{{The trajectories of four-robot formation with different attack probabilities under DRHRL~\cite{zhang2025toward}, DRMPC~\cite{zhou2022event}, and DSLC.}}
				\label{four-robots}
\end{figure*}
\emph{Practical stability verification}: Theorem~\ref{convergence_analysis} implies that the learned control policy might be suboptimal. The associated cost value $J(t)$ might not satisfy the monotonic decreasing property. Hence, as demonstrated in~\cite{zhang2025toward}, the theoretical results developed in DMPC may not be directly applicable. In line with~\cite{zhang2025toward}, we introduce a practical condition to verify the closed-loop stability. 
Let $\bm u_{a,i}^{b}=\break{\rm col}_{i\in\mathbb{N}_1^M}\bm u_{a,i}^{b}$ and $\bm u_{s,i}^{b}={\rm col}_{i\in\mathbb{N}_1^M}\bm u_{s,i}^{b}$,  such that $J^b(t)$  is a Lyapunov candidate function satisfying
		\begin{equation}\label{Eqn:baseline_cost-o} J^b(t+ t_c)- J^b(t)<-\sum_{i=1}^M \int_{\tau=t}^{t+t_c}\chi_i(x_{\scriptscriptstyle \mathcal{N}_i},u_{s,i}, u_{a,i})d\tau,
		\end{equation}
		where $\chi_i(x_{\scriptscriptstyle \mathcal{N}_i},u_{s,i}, u_{a,i})$ is a class $\mathcal{K}$ function, $ J^b=\sum_{i=1}^{M} J_i^b$, $J^b(x_{\scriptscriptstyle \mathcal{N}_i}^b)$ is the associated performance index in~\eqref{Eqn:HL_cost} under $u_i^{b}$, respectively.
		Hence, we introduce the following condition to verify closed-loop stability:
		\begin{equation}\label{Eqn:sta-con-o}
			J(t)\leq J^b(t), \ t=t_0,\,t_0+ t_c,\,t_0+2t_c\,\cdots.
		\end{equation}
		
		This condition is a practical and easily verifiable condition for verifying closed-loop stability. A simple but practical way is to draw $ J^b(t)$ as a monotonically decreasing function, with $J^b(\infty)=0$. With the designed monotonically decreasing function $ J^b(t)$, the only thing left is to verify whether $J(t)\leq J^b(t)$ during policy learning.  

\subsection{Auxiliary Simulation Results}\label{Auiliary-results}
{\color{black} This subsection first presents simulation results of DSLC under different attack probabilities across robots (see Fig.~\ref{six-robots-different}), demonstrating the effectiveness of our approach under diverse attack probabilities.}
The comparative results for the two-robot and four-robot formation scenarios under attacks are provided in Figs.~\ref{two-robots} and~\ref{four-robots}, respectively, which complement the six-robot formation scenario shown in Fig.~\ref{six-robots}. 
{\color{black} Due to space limitations, other simulation results including deploying of a defense policy trained under a specific attack probability to scenarios with different attack probabilities and secure control with alternative destabilizing attacker policies, are given in Appendix~\ref{Auiliary-MATERIAL}. } 
%
%
	{\color{black}\section{Auxiliary Materials}\label{Auiliary-MATERIAL}
\subsection{Learned Weights for Two-Robot Formation under $\beta=0.5,\, 1$}\label{weights}
In the following, we provide the learned weights of the actors and attackers for a two-robot formation scenario under the attack probabilities $\beta=0.5$ and 1.

In the case $\beta=0.5$, the learned weights of the actor and attack for each robot are as follows:\vspace{1mm} \\
%
%
\hspace{-1.5mm}\scalebox{0.9}{$
\begin{array}{ll}W_{s,1}=\\
\begin{bmatrix}
1.64&-0.89&0.52&3.13&-1.03&2.92&-0.54&2.16\\
-3.85&2.54&3.54&
-0.38&-0.07&1.49&0.70&-0.11    
\end{bmatrix}^{\top},
\end{array}
$}
\vspace{1mm}
\hspace{-1mm}\scalebox{0.9}{$\begin{array}{ll}W_{s,2}=\\
\begin{bmatrix}2.88&-2.55&0.40&3.70&2.14&-2.54&-3.47&-0.78\\
-2.57&2.87&
3.11&-1.47&-1.14&-0.20&2.04&-1.47\end{bmatrix}^{\top},\end{array}$} 
\vspace{1mm}
\hspace{-1mm}\scalebox{0.9}{$\begin{array}{ll}W_{a,1}=\\
\begin{bmatrix}1.26&3.44&0.73&2.25&1.89&2.21&-0.76&0.04\\
2.80&2.86& 4.39&
-2.16&-0.12&-2.00&-1.89&1.20\end{bmatrix}^{\top},\end{array}$}
\vspace{1mm}
\hspace{-1mm}\scalebox{0.9}{$\begin{array}{ll} W_{a,2}=\\
\begin{bmatrix}2.88&-0.72&-2.17&1.07&5.25&-1.20&1.98&-3.00\\
4.56&2.49&
-1.42&0.68&1.69&1.26&-3.00&0.43\end{bmatrix}^{\top}.\end{array}$}
\vspace{1mm}

In the case  $\beta=1$, the learned weights of the actor and attack for each robot are as follows: \\
%
%
%
\vspace{1mm}
\hspace{-1mm}\scalebox{0.9}{$\begin{array}{ll}W_{s,1}=\\
\begin{bmatrix}-0.07&-1.40&-2.58&2.95&3.42&-3.47&3.01&0.90\\        -0.93&2.66&
0.72&-3.08&-0.32&3.04&-1.80&1.95\end{bmatrix}^{\top},\end{array}$}
\vspace{1mm}
\hspace{-1mm}\scalebox{0.9}{$\begin{array}{ll}W_{s,2}=\\
\begin{bmatrix}1.67&2.04&-0.39&3.15&0.86&0.64&0.88&1.92\\
         -0.67&-0.48&2.90&         -2.47&2.15&2.72&-1.26&0.74\end{bmatrix}^{\top},\end{array}$}
\vspace{1mm}
\hspace{-1mm}\scalebox{0.9}{$\begin{array}{ll}
W_{a,1}=\\
\begin{bmatrix}4.02&0.33&1.99&2.53&-3.93&0.51&1.01&2.72\\
         1.18& 3.97&3.04&      -1.88&1.62&1.09&-2.10&3.24\end{bmatrix}^{\top},\end{array}$}
         \vspace{1mm}
\hspace{-1mm}\scalebox{0.9}{$\begin{array}{ll} W_{a,2}=\\
\begin{bmatrix}2.52&-1.30&-0.73&2.69&1.12&1.14&-1.99& 0\\        
	-0.12&3.48&0.08&-1.23&0.63&-0.46&-0.45&0.20\end{bmatrix}^{\top}.\end{array}$}
\subsection{Auxiliary Results and Analysis}\label{auxiliary-res}
\emph{Policy deployment to scenarios with different attack probabilities:} We have conducted simulation studies to demonstrate that a defense policy trained under a specific attack probability can be deployed to scenarios with different attack probabilities. As shown in Fig.~\ref{deploy-different}, the policy learned under $\beta = 0.4$ successfully stabilizes the MRS for $\beta = 0.1,0.2,0.3,0.5$, while the policy learned under $\beta = 0.7$ stabilizes the MRS for $\beta = 0.6,0.8,0.9,1$. These results alleviate the requirement for exact knowledge of $\beta$ and highlight the practical value of the proposed approach under unknown attack probabilities.

 \emph{Scalability of DSLC versus DRMPC:} We have recorded the average computational time of our approach and DRMPC under different robot scales (see Table~\ref{tab:computation_time}). The results show that the average computational time of the proposed approach grows approximately linearly with the number of robots and exhibits a substantial reduction in computational load compared to DRMPC.

\emph{Completion time comparison with DRHRL in experiments:} Note that, the experimental results shown in Figs.~\ref{R4LB00}–\ref{ContainmentB10_P_C} demonstrate that the proposed approach successfully stabilizes the MRS in general formation, affine formation, and containment tasks, whereas DRHRL suffers from performance degradation under attacks in most cases. We also collected the corresponding task completion times associated with Figs.~\ref{R4LB00}–\ref{ContainmentB10_P_C}, as summarized in Table~\ref{tab:completion_time}. The results indicate that DSLC completes the tasks faster than, or comparably to, DRHRL. Taken together, both the control performance and task completion time validate the effectiveness and robustness of the proposed approach.

\emph{Defense performance under other destabilizing attack policies:}
In addition to the worst-case attack policies learned through the game-theoretic learning framework, we also conducted simulation studies using a general destabilizing attack policy, i.e., $u_{a,i}=K x_i,\, \forall i\in \mathbb{N}_1^M$, where $K=[0.5,1,0.1,0.5;0.2,0.7,0.3,1]$, to attack MRS. The results shown in Fig.~\ref{six-robots-attack-model} demonstrate that the deployed defense policy can still stabilize the MRS under unseen attack strategies.
}

\begin{figure*}[htbp]
\centering
     \includegraphics[width=1.7\columnwidth]{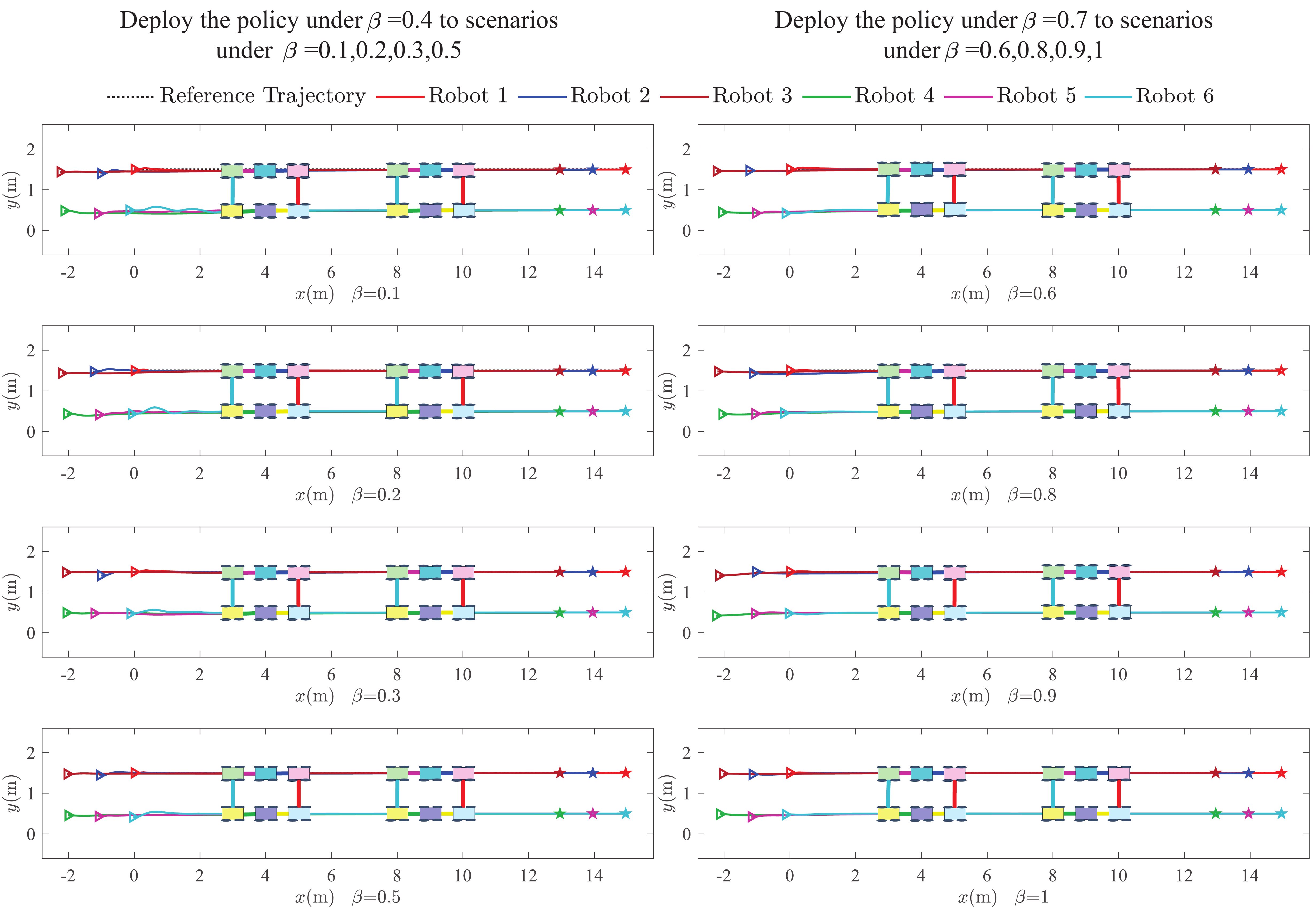}
				\caption{{\color{black}The trajectories of a six-robot formation illustrating the deployment of a defense policy trained under a specific attack probability to scenarios with different attack probabilities.}}
				\label{deploy-different}
\end{figure*}
\begin{table*}[htbp]
\centering
\caption{Average Computational Time Comparison between DSLC and DRMPC}
\label{tab:computation_time}
{\color{black}\begin{tabular}{cccccc}
\toprule
\multirow{2}{*}{Robot Number} & \multirow{2}{*}{$\beta$} & \multicolumn{2}{c}{Average Computation Time (ms)} \\
\cmidrule{3-4}
& & DSLC &DRMPC \\
\midrule
\multirow{3}{*}{2} & 0.1 & $1.38$ & $25.6$ \\
& 0.4 & $1.53$ & $16.3$ \\
& 0.7 & $1.28$ & $37.4$ \\
\midrule
\multirow{3}{*}{4} & 0.2 & $2.28$ & $56.2$ \\
& 0.5 & $2.00 $ & $54.2$ \\
& 0.8 & $3.05$ & $60.7$ \\
\midrule
\multirow{3}{*}{6} & 0.3 & $4.60$ & 183 \\
& 0.6 & $3.73$ & 195 \\
& 0.9 & $3.10$ & 156 \\
\midrule
\multirow{3}{*}{100} & 0.1 & $10.1$ & -- \\
& 0.4 & $10.2$ & -- \\
& 0.7 & $10.2$ & -- \\
\bottomrule
\end{tabular}}
\end{table*}
\begin{table*}[htbp]
    \centering
    \caption{Task Completion Time Comparison}
    \label{tab:completion_time}
   { \color{black}\begin{tabular}{l c c c c}
        \toprule
         \multirow{2}{*}{Task} & \multirow{2}{*}{Robot Number} & \multirow{2}{*}{$\beta$} & \multicolumn{2}{c}{Completion Time (s)} \\
        \cmidrule{4-5}
        & & & DSLC & DRHRL \\
        \midrule
        \multirow{2}{*}{Containment} & \multirow{2}{*}{5} & 0.5 & 25.02 & 36.11 \\
        & & 1 & 25.02& 34.28 \\
        \midrule
        \multirow{2}{*}{Affine Formation} & \multirow{2}{*}{4} & 0.5 & 10.01 & 10.01 \\
        & & 1 & 10.02 & 10.02 \\
        \midrule
        \multirow{8}{*}{General Formation} & \multirow{2}{*}{4 (Scenario I)} & 0.5 & 11.11 & 11.18 \\
        & & 1 & 11.04 & 11.07 \\
        \cmidrule{2-5}
        & \multirow{2}{*}{6 (Scenario I)} & 0.5 & 11.11 & 10.29 \\
        & & 1 & 11 & 10.21 \\
        \cmidrule{2-5}
        & \multirow{2}{*}{4 (Scenario II)} & 0.5 & 20.02 & 22.04 \\
        & & 1 & 20.05 & 21.06 \\
        \cmidrule{2-5}
        & \multirow{2}{*}{6 (Scenario II)} & 0.5 & 20.09 & 20.13 \\
        & & 1 & 20.17 & 20.09 \\
        \bottomrule
    \end{tabular}}
\end{table*}
\begin{figure*}[h]
\label{policy_under_other_attack}
     \centering
     \includegraphics[width=1.6\columnwidth]{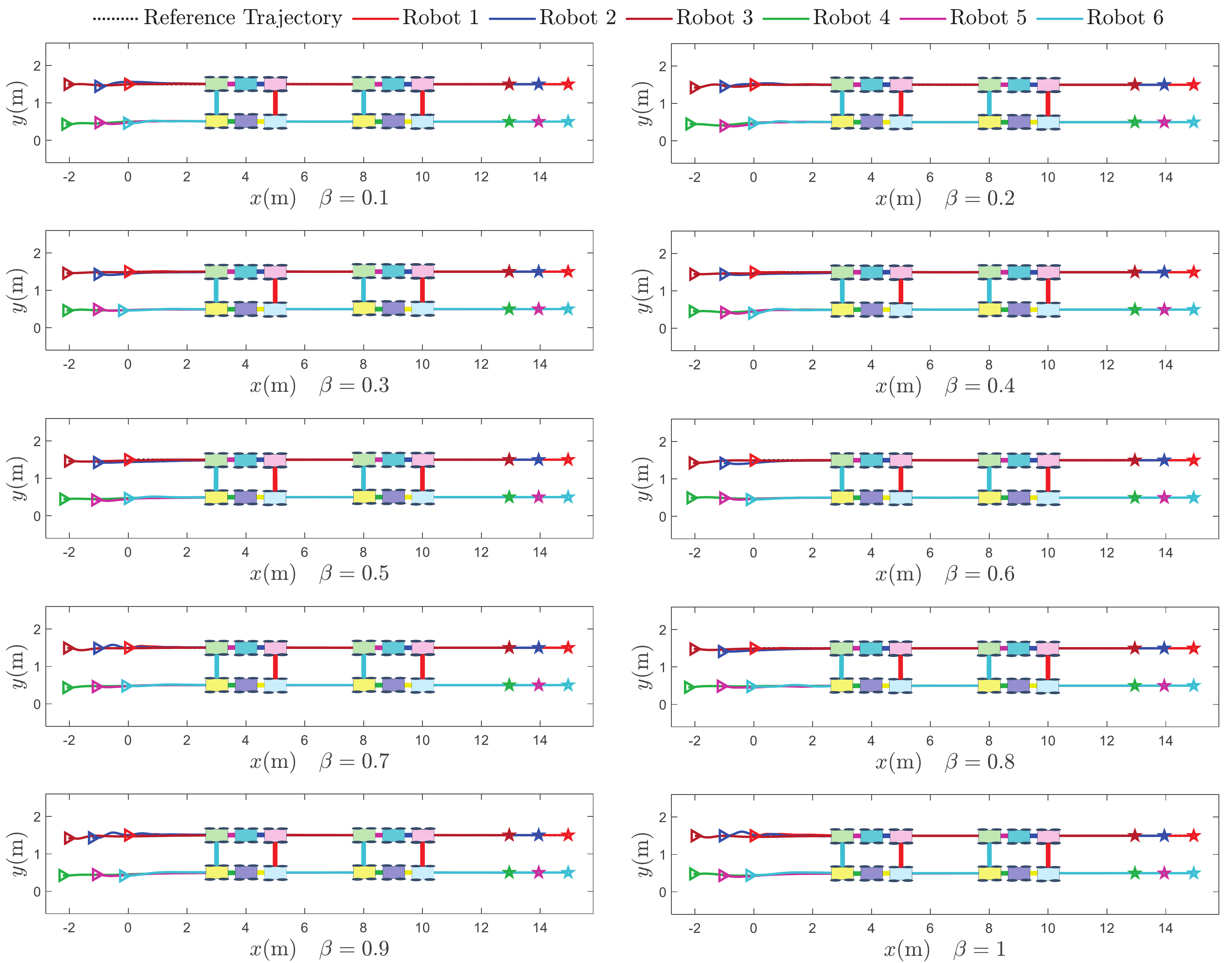}
				\caption{{\color{black}
                The trajectories of a six-robot formation under DSLC in the presence of a destabilizing attack policy, i.e., $u_{a,i}=K x_i,\, \forall i\in \mathbb{N}_1^M$, where $K=[0.5,1,0.1,0.5;0.2,0.7,0.3,1]$ is a randomly generated matrix.}}
				\label{six-robots-attack-model}
\end{figure*}
\end{document}